\documentclass{article}

\usepackage[square,numbers,sort&compress]{natbib}
\usepackage[hidelinks]{hyperref} 

\usepackage[final]{neurips_2023}
\usepackage{amsthm}

\usepackage{float}
\newtheorem{definition}{Definition}
\newtheorem{proposition}{Proposition}
\newtheorem{theorem}{Theorem}
\newtheorem{lemma}{Lemma}
\newtheorem{remark}{Remark}
\newtheorem{assumption}{Assumption}

\def\eps{\varepsilon}

\newcommand{\ean}{\end{eqnarray*}}
\newcommand{\ban}{\begin{eqnarray*}}
\newcommand{\ea}{\end{eqnarray}}
\newcommand{\ba}{\begin{eqnarray}}

\newcommand{\EE}{\mathbb{E}}

\usepackage[utf8]{inputenc} 
\usepackage[T1]{fontenc}    
\usepackage{hyperref}       
\hypersetup{breaklinks=true,
            colorlinks=true,
            citecolor=black,
            urlcolor=blue,
            linkcolor=black}
\usepackage{url}
\usepackage{amsmath}
\usepackage{booktabs}       
\usepackage{amsfonts}       
\usepackage{nicefrac}       
\usepackage{microtype}      
\usepackage{xcolor}         
\usepackage{graphicx} 
\usepackage{subcaption} 

\usepackage{enumitem}
\setlist[itemize]{leftmargin=.4cm} 
\setlist[enumerate]{leftmargin=.4cm} 

\newcommand{\PLH}{{\mkern-2mu\times\mkern-2mu}}

\title{Conditional score-based diffusion models for \\
Bayesian inference in infinite dimensions}

\makeatletter
\newcommand\notsosmall{\@setfontsize\notsosmall{9.5}{11}}
\makeatother

\author{%
  Lorenzo Baldassari \\
  University of Basel \\
  \texttt{\texttt{\notsosmall lorenzo.baldassari@unibas.ch}} \\
  \And
  Ali Siahkoohi \\
  Rice University \\
  \texttt{\notsosmall alisk@rice.edu} \\
  \And
  Josselin Garnier\\
  Ecole Polytechnique, IP Paris\\
  \texttt{\notsosmall josselin.garnier@polytechnique.edu}\\
  \And
  Knut S\o{}lna\\
  University of California Irvine\\
  \texttt{\notsosmall ksolna@math.uci.edu}\\
  \And
  Maarten V. de Hoop\\
  Rice University\\
  \texttt{\notsosmall mvd2@rice.edu}\\
}

\begin{document}

	\maketitle

		\begin{abstract}
Since their initial introduction, score-based diffusion models (SDMs) have been successfully applied to solve a variety of linear inverse problems in finite-dimensional vector spaces due to their ability to efficiently approximate the posterior distribution. However, using SDMs for inverse problems in infinite-dimensional function spaces has only been addressed recently, primarily through methods that learn the unconditional score. While this approach is advantageous for some inverse problems, it is mostly heuristic and involves numerous computationally costly forward operator evaluations during posterior sampling. To address these limitations, we propose a {theoretically grounded method for sampling from the posterior} of infinite-dimensional Bayesian linear inverse problems based on {amortized conditional SDMs}. In particular, we prove that one of the most successful approaches for estimating the conditional score in finite dimensions---the conditional denoising estimator---can also be applied in infinite dimensions. A significant part of our analysis is dedicated to demonstrating that extending infinite-dimensional SDMs to the conditional setting requires careful consideration, as the conditional score typically blows up for small times, contrarily to the unconditional score. We conclude by presenting stylized and large-scale numerical examples that validate our approach, offer additional insights, and demonstrate that our method enables {large-scale}, {discretization-invariant} Bayesian inference.
	\end{abstract}



\section{Introduction}

Inverse problems seek to estimate unknown parameters using noisy observations or measurements. One of the main challenges is that they are often ill-posed. A problem is ill-posed if there are no solutions, or there are many (two or more) solutions, or the solution is unstable  in relation to small errors in the observations \cite{hadamard1923lectures}.  A common approach to transform the original ill-posed problem into a well-posed one is to formulate it as a least-squares  optimization problem that minimizes the difference between observed and predicted data. However, minimization of the data misfit alone negatively impacts the quality of the
obtained solution due to the presence of noise in the data and the inherent nullspace of the forward operator \cite{aster2018parameter, ito2014inverse}. Casting the inverse problem into a {Bayesian probabilistic framework} allows, instead, for a full characterization of all the possible solutions \cite{lehtinen1989linear,  stuart2010inverse, tarantola2005inverse}. The Bayesian approach consists of putting a prior probability distribution describing uncertainty in the parameters
of interest, and finding the posterior distribution over these parameters \cite{knapik2011bayesian}. The prior must be chosen
appropriately in order to mitigate the ill-posedness of the problem and facilitate computation of the posterior. By adopting the Bayesian formulation, rather than finding one single
solution to the inverse problem (e.g., the maximum a posteriori estimator \cite{bennett2005inverse}), a distribution of solutions---the {posterior}---is finally
obtained, whose samples are consistent with the observed data. The posterior distribution can then be
sampled to extract statistical information that allows for uncertainty quantification \cite{stuart2014uncertainty}.

Over the past few years, deep learning-based methods have been successfully applied to analyze linear inverse problems in a Bayesian fashion. In particular, recently introduced score-based diffusion models (SDMs) \cite{song2021scorebased} have become increasingly popular, due to their ability of producing approximating samples from the posterior distribution \cite{ kawar2021snips,song2022solving}. An SDM consists of a diffusion process, which gradually perturbs the data distribution toward  a tractable distribution according to a prescribed stochastic differential equation (SDE) by progressively injecting Gaussian noise, and a generative model, which entails a denoising process defined by approximating the time-reversal of the diffusion. Crucially, the denoising stage is also a diffusion process \cite{anderson1982reverse} whose drift depends on the logarithmic gradients of the noised data densities---the {scores}---which are estimated by \citet{song2021scorebased} using a neural network. Among the advantages of SDMs over other deep generative models is that they produce high-quality samples, matching the performance of generative adversarial networks \cite{dhariwal2021diffusion}, without suffering from training instabilities and mode-collapse \cite{gnaneshwar2022score, song2021scorebased}. Additionally, SDMs are not restricted to invertible architectures like normalizing flows \cite{lee2023convergence}, which often limits the complexity of the distributions that can be learned.
Finally, and most importantly to the scope of this work, SDMs have demonstrated superior performance in a variety of inverse problems, such as image inpainting \cite{song2019generative, song2021scorebased}, image colorization \cite{song2021scorebased}, compressing sensing, and medical imaging \cite{jalal2021robust, song2022solving}.

In the aforementioned cases, SDMs have been applied by assuming that the data distribution of interest is supported on a finite-dimensional vector space. However, in many inverse problems, especially those governed by partial differential equations (PDEs), the unknown parameters to be estimated are {functions} (e.g., coefficient functions, boundary and initial conditions, or source functions) that exist in a suitable function space, typically an {infinite-dimensional Hilbert space}. The inverse heat equation or the elliptic inverse source problem presented in \cite{dashti2013bayesian} are typical examples of ill-posed inverse problems that are naturally formulated in infinite-dimensional Hilbert spaces. In addition to these PDE-based examples, other interesting cases that are not PDE-based include geometric inverse problems (e.g., determining the Riemann metric from geodesic information or the background velocity map from travel time information in geophysics \cite{uhlmann2016inverse}) and inverse problems involving singular integral operators \cite{dynin1978inversion}.
A potential solution for all of these problems could be to discretize the input and output functions into finite-dimensional vectors and apply SDMs to sample from the posterior. However,  theoretical studies of current diffusion models suggest that performance guarantees do not generalize well on increasing dimensions \cite{chen2022sampling, de2022convergence, pidstrigach2023infinitedimensional}. This is precisely why Andrew Stuart's guiding principle to study a Bayesian inverse problem for functions---``{avoid discretization until the last possible moment}'' \cite{stuart2010inverse}---is  critical to the use of SDMs.

Motivated by Stuart's principle, in this work we \emph{define a conditional score in the infinite-dimensional setting}, a critical step for studying Bayesian inverse problems directly in function spaces through SDMs. In particular, we show that using this newly defined score as a reverse drift of the diffusion process yields a generative stage that samples, under specified conditions, from the correct target conditional distribution.
We carry out the analysis by focusing on two cases: the case of a Gaussian prior measure and the case of a general class of priors given as a density with respect to a Gaussian measure. Studying the model for a Gaussian prior measure provides illuminating insight, not only because it yields an analytic formula of the score, but also because it gives a full characterization of SDMs in the infinite-dimensional setting, showing under which conditions we are sampling from the correct target conditional distribution and how fast the reverse SDE converges to it. It also serves as a guide for the analysis in the case of a general class of prior measures.
Finally, we conclude this work by presenting, in Section \ref{sec:numerics}, stylized and large-scale numerical examples that demonstrate the applicability of our SDM. Specifically, we show that our SDM model (i) is able to \emph{approximate non-Gaussian multi-modal distributions}, a challenging task that poses difficulties for many generative models \cite{arora2018gans}; (ii) is \emph{discretization-invariant}, a property that is a consequence of our theoretical and computational framework being built on the infinite-dimensional formulation proposed by \citet{stuart2010inverse}; and (iii) is \emph{applicable to solve large-scale Bayesian inverse problems}, which we demonstrate by applying it to a large-scale problem in geophysics, i.e., the linearized wave-equation-based imaging via the Born approximation that involves estimating a $256\PLH256$-dimensional unknown parameter.

\paragraph{Related works}
Our work is primarily motivated by Andrew Stuart's comprehensive mathematical theory for studying PDE-governed inverse problems in a Bayesian fashion \cite{stuart2010inverse}. In particular, we are interested in the infinite-dimensional analysis \cite{knapik2011bayesian, lasanen2007measurements}, which emphasizes the importance of analyzing PDE-governed inverse problems {directly in function space before discretization}.

Our paper builds upon a rich and ever expanding body of theoretical and applied works dedicated to SDMs.  \citet{song2021scorebased} defined SDMs integrating both score-based (\citet{hyvarinen2005estimation}; \citet{song2019generative}) and diffusion (\citet{sohl2015deep}; \citet{ho2020denoising}) models into a single continuous-time framework based on stochastic differential equations. The generative stage in SDMs is based on a result from \citet{anderson1982reverse} proving that the denoising process is also a diffusion process whose drift depends on the scores. This result holds only in \emph{vector spaces}, which explains the difficulties to extend SDMs to more general function spaces. Initially, there have been attempts to project the input functions into a finite-dimensional feature space and then apply SDMs (\citet{dupont2022data}; \citet{phillips2022spectral}). However,  these approaches \emph{are not discretization-invariant}. It is only very recently that SDMs have been directly studied in function spaces, specifically {infinite-dimensional Hilbert spaces}.  \citet{kerrigan2022diffusion} generalized diffusion models to operate directly in function spaces, but they did not consider the time-continuous limit based on SDEs (\citet{song2021scorebased}). \citet{dutordoir2023neural} proposed a denoising diffusion generative model for performing Bayesian inference of functions. \citet{lim2023score} generalized score matching for trace-class noise corruptions that live in the Hilbert space of the data. However, as \citet{kerrigan2022diffusion} and \citet{dutordoir2023neural}, they did not investigate the connection to the forward and backward SDEs as  \citet{song2021scorebased} did in finite dimensions. Two recent works, \citet{pidstrigach2023infinitedimensional} and \citet{franzese2023continuous}, finally established such connection for the \emph{unconditional setting}. In particular, \citet{franzese2023continuous} used results from infinite-dimensional SDEs theory (\citet{follmer1986time}; \citet{millet1989time}) close to \citet{anderson1982reverse}.

Among the mentioned works, \citet{pidstrigach2023infinitedimensional} is the closest to ours. We adopt their formalism to establish theoretical guarantees for sampling from the conditional distribution. Another crucial contribution comes from \citet{batzolis2021conditional}, as we build upon their proof to show that the score can be estimated by using {a denoising score matching objective conditioned on the observed data \cite{vincent2011connection, song2019generative}. A key element in \citet{pidstrigach2023infinitedimensional}, emphasized also in our analysis,
is obtaining an estimate on the expected square norm of the score that needs to be \emph{uniform in time}. We explicitly compute the expected square norm of the conditional score in the case of a Gaussian prior measure, which shows that a uniform in time estimate is \emph{not always possible in the conditional setting}.
This is not surprising, given that the singularity in the conditional score as noise vanishes is a well-known phenomenon in finite dimensions and has been investigated in many works, both from a theoretical and a practical standpoint \cite{kim2021soft, dockhorn2021score}. In our paper, we provide a set of concrete conditions to be satisfied to ensure a uniform estimate in time for a general class of prior measures in infinite dimensions.

\citet{pidstrigach2023infinitedimensional} have also proposed a method for performing conditional sampling, building upon the approach introduced by \citet{song2022solving} in a finite-dimensional setting.
Like our approach, their method can be viewed as a contribution to the literature on {likelihood-free, simulation-based inference} \cite{cranmer2020frontier, lavin2021simulation}. Specifically, the algorithm proposed by \citet{pidstrigach2023infinitedimensional}  relies on a projection-type approach that incorporates the observed data into the unconditional sampling process via a proximal optimization step to generate intermediate samples  consistent with the measuring acquisition process. This allows \citet{pidstrigach2023infinitedimensional} \emph{to avoid defining the conditional score}\footnote{We note that, in a newer version of their paper submitted to arXiv on October 3, 2023 (four months after our submission to arXiv and NeurIPS 2023), \citet{pidstrigach2023infinitedimensional} abandoned the projection-type approach. Instead, they invoke the conditional score function to perform posterior sampling and solve an inverse problem in a similar fashion to ours (we refer to Section 8.3 of their paper for details). However, they still do not address the well-posedness of the forward-reverse conditional SDE and the singularity of the conditional score, and their implementation is based on UNets and, thus, is not discretization-invariant.}. While their method has been shown to work well with specific inverse problems, such as medical imaging \cite{song2022solving}, it is primarily heuristic, and its computational efficiency varies depending on the specific inverse problem at hand. Notably, their algorithm may require numerous computationally costly forward operator evaluations during posterior sampling. Furthermore, their implementation does not fully exploit the discretization-invariance property achieved by studying the problem in infinite dimensions since they employ a UNet to parametrize their score, limiting the evaluation of their score function to the training interval.
The novelty of our work is then twofold. First, we provide theoretically grounded guarantees for an approach that is not heuristic and can be implemented such that \emph{it is not constrained to the grid on which we trained our network}. As a result, we show that we effectively take advantage of the discretization-invariance property achieved by adopting the infinite-dimensional formulation proposed by \citet{stuart2010inverse}. Second, we perform discretization-invariant Bayesian inference by learning an \emph{amortized version of the conditional score}. This is done by making the score function {depending on the observations}. As a result, provided that we have access to high-quality training data, during sampling we can
input any new observation that we wish to condition on directly during simulation of the reverse SDE. In this sense, our method is \emph{data-driven}, as the information about the forward model is implicitly encoded in the data pairs used to learn the conditional score. This addresses a critical gap in the existing literature, as the other approach using infinite-dimensional SDM resorts to projections onto the measurement subspace for sampling from the posterior---a method that not only lacks theoretical interpretation but may also yield unsatisfactory performance due to costly forward operator computations. There are well-documented instances in the literature where amortized methods can be a preferred option in Bayesian inverse problems \cite{baptista2020adaptive, kim2018semi, kruse2021hint, radev2020bayesflow, siahkoohi2022wave, siahkoohi2023reliable}, as they reduce inference computational costs by incurring an offline initial training cost for a deep neural network that is capable of approximating the posterior for unseen observed data, provided that one has access to a set of data pairs that adequately represent the underlying joint distribution.

\paragraph{Main contributions}
The main contribution of this work is the \emph{analysis of conditional SDMs in infinite-dimensional Hilbert spaces}. More specifically,

\vspace*{-0.6em}
\begin{itemize}
    \setlength\itemsep{0.0em}
    \item We introduce the conditional score in an infinite-dimensional setting (Section \ref{sec:conditional-score}).
    \item We provide a comprehensive analysis of the forward-reverse conditional SDE framework in the case of a \emph{Gaussian prior measure}. We explicitly compute the expected square norm of the conditional score, which shows that a uniform in time estimate \emph{is not always possible for the conditional score}. We prove that as long as we start from the invariant distribution of the diffusion process, the reverse SDE converges to the target distribution exponentially fast (Section \ref{sec:gaussian}).
    \item We provide a set of conditions to be satisfied to ensure a uniform in time estimate for a general class of prior measures that \emph{are given as a density with respect to a Gaussian measure}. Under these conditions, the conditional score---used as a reverse drift of the diffusion process in SDMs---yields a generative stage that samples from the target conditional distribution (Section \ref{sec:general-case}).
    \item We prove that the conditional score can be estimated via a {conditional denoising score matching objective} in infinite dimensions (Section \ref{sec:general-case}).
    \item We present examples that validate our approach, offer additional insights, and demonstrate that our method enables \emph{large-scale}, \emph{discretization-invariant} Bayesian inference (Section \ref{sec:numerics}).
\end{itemize}

\section{Background}

Here, we review the definition of unconditional score-based diffusion models (SDMs) in infinite-dimensional Hilbert spaces proposed by \citet{pidstrigach2023infinitedimensional}, as we will adopt the same formalism to define SDMs for conditional settings. We refer to {Appendix A} for a brief introduction to key tools of probability theory in function spaces.

Let $\mu_\text{data}$ be the target measure, supported on a separable Hilbert space $(H,\langle \cdot, \cdot \rangle)$. Consider a forward infinite-dimensional diffusion process $(X_t)_{t\in [0,T]}$ for continuous time variable $t \in [0,T]$, where $X_0$ is the starting variable and $X_t$ its perturbation at time $t$. The diffusion process is defined by the following SDE: \begin{equation}
d X_t = -
\frac{1}{2} X_t dt + \sqrt{C} dW_t,
\label{eq:forward-sde}
\end{equation}
where $C:H\to H$ is a fixed trace class, positive-definite, symmetric covariance operator and $W_t$ is a Wiener process on $H$. Here and throughout the paper, the initial conditions 
and the driving Wiener processes in \eqref{eq:forward-sde} are assumed independent.

The forward SDE evolves $X_0 \sim \mu_0$ towards the Gaussian measure $ \mathcal{N}(0,C)$ as $t\to \infty$. The goal of score-based diffusion models is to convert the SDE in \eqref{eq:forward-sde} to a generative model by first sampling $X_T \sim \mathcal{N}(0,C)$, and then running the correspondent reverse-time SDE.
In the finite-dimensional case, \citet{song2021scorebased} show that the reverse-time SDE requires the knowledge of the score function $\nabla \log p_t(X_t)$, where $p_t(X_t)$ is the density of the marginal distribution of $X_t$ (from now on denoted $\mathbb{P}_t$) with respect to the Lebesgue measure. In infinite-dimensional Hilbert spaces, there is no natural analogue of the Lebesgue measure (for additional details, see \cite{da2006introduction}) and the density is thus no longer well defined. However, \citet[Lemma 1]{pidstrigach2023infinitedimensional} notice that, in the finite-dimensional setting where $H=\mathbb{R}^D$, the score can be expressed as follows:
\begin{equation}
C \nabla_x \log p_t(x)  = - (1-e^{-t})^{-1} \left(x-e^{-t/2} \mathbb{E}[X_0|X_t = x]\right),
\end{equation}
for $t>0$. Since the right-hand side of the expression above is also well-defined in infinite dimensions, \citet{pidstrigach2023infinitedimensional} formally define the score as follows:

\begin{definition}
	\label{def:score}
	In the infinite-dimensional setting, the score or reverse drift is defined by
\begin{equation}
S(t,x) := - (1-e^{-t})^{-1} \left(x-e^{-t/2} \mathbb{E}[X_0|X_t = x]\right).
\end{equation}
\end{definition}
%
Assuming that the expected square norm of the score is uniformly bounded in time, \citet[Theorem 1]{pidstrigach2023infinitedimensional} shows that the following SDE
\begin{equation}
	dZ_t = \frac{1}{2} Z_t dt + S(T-t,Z_t) \text{d}t + \sqrt{C}dW_t, \qquad Z_0 \sim \mathbb{P}_T,
	\label{eq:reverse-SDE}
\end{equation}
is the time-reversal of \eqref{eq:forward-sde} and the distribution of $Z_T$ is thus equal to $\mu_0$, proving that the forward-reverse SDE framework of \citet{song2021scorebased} generalizes to the infinite-dimensional setting. The reverse SDE requires the knowledge of this newly defined score, and one approach for estimating it is, similarly to \cite{song2021scorebased}, by using the denoising score matching loss \cite{vincent2011connection}
\begin{equation}
\mathbb{E} \left[\big\| \widetilde{S}(t,X_t) + (1-e^{-t})^{-1} (X_t - e^{-t/2}X_0)\big\|^2\right],
\end{equation}
where $\widetilde{S}(t,X_t)$ is typically approximated by training a neural network.


%

\section{The conditional score in infinite dimensions} \label{sec:conditional-score}

Analogous to the score function relative to the unconditional SDM in infinite dimensions, we now define the score corresponding to the reverse drift of an SDE when conditioned on observations. We consider a setting where $X_0$ is an $H$-valued random variable and $H$ is an infinite-dimensional Hilbert space. Denote by
\begin{equation}
Y = A X_0 + B,
\label{eq:inverse-problem}
\end{equation}
a noisy observation given by $n$ linear measurements, where the measurement acquisition process is represented by a linear operator $A:H \to \mathbb{R}^n$, and $B \sim \mathcal{N}(0,C_B)$ represents the noise, with $C_B$ a $n\times n$ nonnegative matrix. Within a Bayesian probabilistic framework, solving \eqref{eq:inverse-problem} amounts to putting an appropriately chosen prior probability distribution $\mu_0$ on $X_0$, and sampling from the conditional distribution of $X_0$ given $Y=y$.

To the best of our knowledge, the only existing algorithm which performs conditional sampling using infinite-dimensional diffusion models on Hilbert spaces is based on the work of \citet{song2022solving}. The idea, adapted to infinite dimensions by \citet{pidstrigach2023infinitedimensional}, is to incorporate the observations into the unconditional sampling process of the SDM via a proximal optimization step to generate intermediate samples that are consistent with the measuring acquisition process. Our method relies instead on \emph{utilizing the score of infinite-dimensional SDMs conditioned on observed data}, which we introduce in this work. We begin by defining the conditional score, by first noticing that, in finite dimensions, we have the following lemma:
\begin{lemma}
\label{lem:conditional-score}
In the finite-dimensional setting where $H=\mathbb{R}^D$, we can express the conditional score function for $t>0$ as
\begin{equation}
C \nabla_x \log p_t(x|y) =
-(1-e^{-t})^{-1} \left(  x  - e^{-t/2}\EE\left[ X_0 | Y=y,X_t=x\right]\right).
\label{eq:conditional-score-finite}
\end{equation}
\end{lemma}
Since the right-hand side of \eqref{eq:conditional-score-finite} is well-defined in infinite dimensions, by following the same line of thought of \citet{pidstrigach2023infinitedimensional} we formally define the score as follows:
\begin{definition}\label{def:conditional-score}
In the infinite-dimensional setting, the conditional score is defined by
\begin{equation}
	S(t,x,y) :=
	-(1-e^{-t})^{-1} \left(  x  - e^{-t/2}\EE\left[ X_0 | Y=y,X_t=x\right]\right).
	\label{eq:conditional-score-infinite}
\end{equation}
\end{definition}
\begin{remark}
It is possible to define the conditional score in infinite-dimensional Hilbert spaces by resorting to the results of \cite{follmer1986time, millet1989time},
see Appendix \ref{app:mill}.
\end{remark}
For Definition \ref{def:conditional-score} to make sense, we need to show that if we use \eqref{eq:conditional-score-infinite} as the drift of the time-reversal of the SDE in \eqref{eq:forward-sde} conditioned on $y$, then it will sample the correct conditional distribution of $X_0$ given $Y=y$ in infinite dimensions. In the next sections, we will carry out the analysis by focusing on two cases: the case of a Gaussian prior measure $\mathcal{N}(0,C_\mu)$, and the case where the prior of $X_0$ is given as a density with respect to a Gaussian measure, i.e.,
\begin{equation}
	X_0 \sim \mu_0, \qquad \frac{d\mu_0}{d\mu}(x_0)=
	\frac{e^{-\Phi(x_0)}}{\EE_\mu[e^{-\Phi(X_0)}]} ,\qquad \mu={\cal N}(0,C_\mu),
	\label{eq:mumu0}
\end{equation}
where $C_\mu$ is positive and trace class and $\Phi$ is bounded with $\EE_\mu [\|C_\mu \nabla_H \Phi(X_0)\|^2]<+\infty$.

\section{Forward-reverse conditional SDE framework for a Gaussian prior measure} \label{sec:gaussian}

We begin our analysis of the forward-reverse conditional SDE framework by examining the case where the prior of $X_0$ is a Gaussian measure. This case provides illuminating insight, not only because it is possible to get an analytic formula of the score, but also because it offers a full characterization of SDMs in the infinite-dimensional setting,  showing under which conditions we are sampling from the correct target conditional distribution and how fast the reverse SDE converges to it.
We also show that the conditional score can have a singular behavior at small times when the observations are noiseless, in contrast with the unconditional score under similar hypotheses.

We assume that $\Phi=0$ in \eqref{eq:mumu0}. All distributions in play are Gaussian:
\begin{align}
	\label{eq:mu0}
    X_ 0 &\sim \mathcal{N} \big(0 , C_\mu\big) ,\\
	\label{eq:mu0X0}
	X_t |X_0 &\sim {\cal N}\big(e^{-t/2} X_0 , (1-e^{-t}) C\big) ,\\
	\label{eq:mu00Y}
	X_0 |Y & \sim {\cal N}\big(M_o Y, C_o \big) ,\\
	X_t |Y  & \sim {\cal N}\big(e^{-t/2} M_o Y ,  e^{-t} C_o +(1-e^{-t}) C \big) ,
	\label{eq:mu0tY}
\end{align}
%
%
where $M_o = C_\mu A^* (AC_\mu A^*+C_B)^{-1}$ and $C_o =  C_\mu -C_\mu A^* (A C_\mu A^*+C_B)^{-1} A C_\mu$. By Mercer theorem \cite{mercer1909xvi}, there exist $(\mu_j)$ in $[0,+\infty)$ and an orthonormal basis $(v_j)$ in $H$ such that
$C_\mu v_j=\mu_j v_j$ $\forall j$.
We consider the infinite-dimensional case with $\mu_j>0$ $\forall j$.
We assume that $C_\mu$ is trace class so that $\sum_j \mu_j<+\infty$.
We assume that the functions $v_j$ are eigenfunctions of $C$ and we denote by $\lambda_j$ the corresponding eigenvalues.

We assume an observational  model  corresponding to observing a finite-dimensional
subspace of $H$ spanned by $v_{\eta(1)},\ldots,v_{\eta(n)}$  corresponding to $g_k=v_{\eta(k)}, ~k=1,\ldots,n$, where $g_j\in H$ is such that $(Af)_j = \left<g_j,f\right>$.
We denote ${\cal I}^{(n)}=\{\eta(1),\ldots,\eta(n)\}$.
We assume moreover
$C_B=\sigma_B^2 I_n$.
Let $Z_t$ be the solution of reverse-time SDE:
\begin{equation}
	dZ_t = \frac{1}{2} Z_t dt + S(T-t, Z_t,y) \text{d}t + \sqrt{C}dW_t, \quad Z_0 \sim X_T|Y=y.
	\label{eq:conditional-reverse-SDE}
\end{equation}
We want to show that the reverse SDE we have just formulated in \eqref{eq:conditional-reverse-SDE} indeed constitutes a reversal of the stochastic dynamics from the forward SDE in \eqref{eq:forward-sde} conditioned on $y$.
To this aim, we will  need the following lemma:
\begin{lemma}\label{lem:gauss1}
We define $Z^{(j)}= \langle v_j, Z\rangle$, $p^{(j)} ={\lambda_j}/{\mu_j}$ for all $j$.
We also define $y^{(j)} = y_{\eta(j)}$ for $j \in {\cal I}^{(n)}$ and $y^{(j)}=0$ otherwise,
and $q^{(j)} = {\mu_j}/{\sigma_B^2} $ for  $j \in {\cal I}^{(n)}$ and $q^{(j)}=0$ otherwise.
Then we can write for all $j$
\begin{equation}
dZ_t^{(j)} = \mu^{(x,j)}(T-t) Z_t^{(j)} dt + \mu^{(y,j)}(T-t) y^{(j)} dt + \sqrt{\lambda_j} dW^{(j)} ,
\label{eq:reverse-sde-projected}
\end{equation}
with $W^{(j)}$ independent and identically distributed standard Brownian motions,
\begin{align}
\mu^{(x,j)}(t) =  \frac{1}{2} - \frac{e^{t} p^{(j)} (1+q^{(j)})}{1+ (e^{t}-1)p^{(j)}(1+q^{(j)})}  , \qquad
\mu^{(y,j)}(t) =    \frac{e^{t/2} p^{(j)}q^{(j)}}{1+ (e^{t}-1)p^{(j)}(1+q^{(j)})}  .
\end{align}
\end{lemma}
\begin{proof}
The proof is a Gaussian calculation. It relies on computing $\langle S, v_j \rangle$, which yields an analytic formula.
See Appendix B.
\end{proof}
\vspace*{-1em}
Lemma \ref{lem:gauss1} enables us to discuss when we are sampling from the correct target conditional distribution $X_0|Y \sim \mathcal{N} \big( M_o Y,C_o\big)$. We can make a few remarks:
\vspace*{-0.6em}
\begin{itemize}
    \item In the limit ${T\to \infty}$, we get $\mu^{(x,j)}(T-t) \to -1/2$ and $\mu^{(y,j)} (T-t)\to 0 $.
    \item If $j \notin {\cal I}^{(n)}$ then we have the same mode dynamics as in the unconditional case. Thus we sample from the correct  target distribution if $T$ is large or if we start from
$Z_0^{(j)} \sim {\cal N}(0,\Sigma^{(j)}_0)$ for
$\Sigma^{(j)}_0 =\mu_j e^{-T} + \lambda_j (1-e^{-T})$, which is the distribution of $X_T^{(j)}= \left< X_T,v_j\right>$ given $Y=y$.
    \item If $j \in {\cal I}^{(n)}$ and we start from $Z_0^{(j)} \sim {\cal N}(\bar{z}_0^{(j)},\Sigma^{(j)}_0)$, then
we   find $Z_{T}^{(j)}   \sim   {\cal N}\big(\bar{z}_T^{(j)} , \Sigma_T^{(j)})$ with
\begin{align}
\Sigma_T^{(j)} &=  \Sigma_0^{(j)} \left(\frac{e^T}{(1+(e^T-1) p^{(j)}(1+q^{(j)}))^2}\right) + \frac{\mu_j}{1+q^{(j)}} \left( 1- \frac{1}{1+(e^T-1) p^{(j)}(1+q^{(j)})} \right) , \\
\bar{z}_T^{(j)} &=    \frac{\bar{z}_0^{(j)} e^{T/2}}{1+(e^T-1)p^{(j)}(1+q^{(j)})} +
\frac{y^{(j)} q^{(j)} }{1+q^{(j)}}  \left( 1- \frac{1}{1+(e^T-1)p^{(j)}(1+q^{(j)})} \right).
\end{align}
The distribution of $X_0^{(j)} = \left< X_0,v_j\right>$ given $Y=y$ is ${\cal N}(y^{(j)} q^{(j)}/(1+q^{(j)}),  \mu_j/(1+q^{(j)}) $. As $\bar{z}_T^{(j)}  \to y^{(j)} q^{(j)}/(1+q^{(j)})$  and $\Sigma_T^{(j)} \to \mu_j/(1+q^{(j)})$ as $T\to +\infty$, this shows that we sample from the exact target distribution (the one of $X_0$ given $Y=y$) for $T$ large.
    \item If we start the reverse-time SDE from the correct model
\begin{equation}
\bar{z}_0^{(j)} = \frac{e^{-T/2} y^{(j)} q^{(j)} }{1+q^{(j)}}, \qquad  \Sigma^{(j)}_0 =\frac{\mu_j e^{-T}}{1+q^{(j)}} + \lambda_j (1-e^{-T}),
\end{equation}
then indeed   $Z_T^{(j)} \sim {\cal N}(y^{(j)} q^{(j)} /(1+q^{(j)}),\mu_j/(1+q^{(j)}))$. This shows that,  for any $T$, $Z_T$ has the same distribution as $X_0$ given $Y=y$, which is the exact target distribution. We can show similarly that  $Z_{T-t}$ has the same distribution as $X_t$ given $Y=y$ for any $t \in [0,T]$.
    \item In the case that
$\sigma_B=0$ so that we observe the mode values perfectly for $j \in {\cal I}^{(n)}$, 
then
\ba
\mu^{(x,j)}(t) = \frac{1}{2} - \frac{e^{t}}{ e^{t}-1}  , \label{eq:gbound1} \qquad
\mu^{(y,j)}(t)  =    \frac{e^{t/2}}{e^{t}-1}  ,
\ea
and indeed $\lim_{t \uparrow T} Z_t^{(j)}=y^{(j)}$ a.s. Indeed the $t^{-1}$
singularity at the origin drives the process to the origin like in the Brownian bridge.
\end{itemize}

Our analysis shows that, as long as we start from the invariant distribution of the diffusion process, we are able to sample from the correct target conditional distribution and that happens exponentially fast. This proves that the score of Definition~\ref{def:conditional-score} is the reverse drift of the SDE in \eqref{eq:conditional-reverse-SDE}. Additionally, the analysis shows that the score is uniformly bounded, except when there is no noise in the observations, blowing up near $t=0$. 

\begin{remark}
Note that, for $q^{(j)}=0$, we obtain the unconditional model:
\ba
dZ_t^{(j)} = \mu^{(j)}(T-t) Z_t^{(j)} dt + \sqrt{\lambda_j} dW^{(j)} , \mbox{ with }
\mu^{(j)}(t) =  \frac{1}{2} - \frac{e^{t} p^{(j)}}{1+ (e^{t}-1)p^{(j)}} .  \label{eq:gbound2}
\ea
If $C=C_\mu$, the square expectation of the norm and the Lipschitz constant of the score are uniformly bounded in time:
$\sup_{j, t \in [0,T]} | \mu^{(j)}(t) | = {1}/{2}$.
\end{remark}

\begin{proposition}\label{prop:Sbound}
The score is $S(t,x,y)= \sum_j S_G^{(j)}(t,\left< x, v_j\right> ,y^{(j)}) v_j$, $S_G^{(j)}(t,x^{(j)} ,y^{(j)}) =
\left(\mu^{(x,j)}(T-t) - 1/2\right)  x^{(j)} + \mu^{(y,j)}(T-t) y^{(j)}  $ and it satisfies
\begin{equation}
\label{eq:bound1g}
 \EE[\|S(t,X_t,y)\|_H^2 |Y=y ]  =
 \sum_j  \frac{e^t (1+q^{(j)})}{1+(e^t-1)p^{(j)}(1+q^{(j)} )}
 \frac{\lambda_j^2}{\mu_j}   .
\end{equation}
\end{proposition}
\begin{proof}
    The proof is a Gaussian calculation given in Appendix B.
\end{proof}

In the unconditional setting, we have $ \EE[\|S(t,X_t)\|_H^2 ]  =   \sum_j  \frac{e^t}{1+(e^t-1)p^{(j)}}
 \frac{\lambda_j^2}{\mu_j}  $ which is equal to $\sum_j \lambda_j$ when $C=C_\mu$. It is indeed uniformly bounded in time.\\
In the conditional and noiseless setting ($\sigma_B=0$), we have $\EE[\|S(t,X_t,y)\|_H^2 |Y=y ]  =  \sum_{ j \not\in {\cal I}^{(n)}}  \frac{e^t}{1+(e^t-1)p^{(j)}}
 \frac{\lambda_j^2}{\mu_j} + \sum_{j \in {\cal I}^{(n)}}  \frac{\lambda_j}{1-e^{-t}}$, which blows up as $1/t$ as $t\to 0$. This result shows that the extension of the score-based diffusion models to the conditional setting is not trivial.

\section{Well-posedness for the reverse SDE for a general class of prior measures} \label{sec:general-case}

We are now ready to consider the case of a general class of prior measures given as a density with respect to a Gaussian measure. The analysis of this case resembles the one of \citet{pidstrigach2023infinitedimensional} for the unconditional setting. The main challenge is the singularity of the score for small times, an event that in the Gaussian case was observed in the noiseless setting.
In this section we will provide a set of conditions to be satisfied by $\Phi$ in \eqref{eq:mumu0}, so that the conditional score is bounded uniformly in time. The existence of this bound is needed to make sense of the forward-reverse conditional SDE, and to prove the accuracy and stability of the conditional sampling.

We start the analysis by recalling that, in the infinite-dimensional case, the conditional score is (\ref{eq:conditional-score-infinite}).
It is easy to get a first estimate:
\begin{align}
\EE[ \| S(t,X_t,y) \|_H^2 |Y=y ]
&\leq (1-e^{-t})^{-1} {\rm Tr}(C)  .
\label{eq:bound2}
\end{align}
The proof follows from  Jensen inequality and the law of total expectation,
see Appendix C.
Note that (\ref{eq:bound2}) is indeed an upper bound of (\ref{eq:bound1g}) since ${\rm Tr}(C)=\sum_j \lambda_j$.

Note that the bound (\ref{eq:bound2}) is also valid for the unconditional score $S(t,x) = -(1-e^{-t})^{-1} \big( x - e^{-t/2}\EE[  X_0 |X_t=x ]\big)$.
We can observe that the upper bound (\ref{eq:bound2}) blows up in the limit of small times. We can make a few comments:
\vspace*{-0.6em}
\begin{itemize}
    \item The bound (\ref{eq:bound2}) is convenient for positive times, but the use of Jensen's inequality results in a very crude bound for small times. As shown in the previous section, we know that there exists a bound (\ref{eq:gbound2}) for the unconditional score in the Gaussian case that is uniform in time.
    \item The singular behavior as $1/t$ at small time $t$ is, however, not artificial. Such a behavior is needed in order to drive the state to the
deterministic initial condition when there are exact observations. This behavior has been exhibited by (\ref{eq:gbound1}) and (\ref{eq:bound1g}) in the Gaussian case when $\sigma_B=0$. This indicates that the following assumption (\ref{eq:assumption1}) is not trivial in the conditional setting.
\end{itemize}

For Definition \ref{def:conditional-score} to make sense in the more general case where the prior of $X_0$ is given as a density with respect to a Gaussian measure, we will need to make the following assumption.

\begin{assumption}
\label{assumption1}
For any $y \in \mathbb{R}^n$, we have
\begin{equation}
\label{eq:assumption1}
    \sup_{t \in [0,T]} \EE \big[ \| S(t,X_t,y) \|_H^2 | Y=y \big] <\infty.
\end{equation}
\end{assumption}

We are now ready to state the analogous result to \citet[Theorem 1]{pidstrigach2023infinitedimensional}.

\begin{proposition}
Under Assumption \ref{assumption1},
the solution of the reverse-time SDE
\begin{equation}
dZ_t = \frac{1}{2}Z_t dt +S(T-t,Z_t,y) dt +\sqrt{C} dW_t , \qquad Z_0 \sim X_T |Y=y,
\end{equation}
satisfies $Z_T \sim X_0 | Y=y$.
\end{proposition}
\begin{proof}
Given Assumption \ref{assumption1},
the proof follows the same steps as the one given in \cite{pidstrigach2023infinitedimensional} for the unconditional score. See Appendix C for the full proof.
\end{proof}

Assumption \ref{assumption1} is satisfied under some appropriate conditions. In the following proposition, we provide a set of conditions that ensure the satisfaction of this assumption. It shows that it is possible to get an upper bound in (\ref{eq:bound2}) that is uniform in time provided some additional conditions are fulfilled. 

\begin{proposition}\label{lemma3}
We assume that
$C_\mu$ in (\ref{eq:mumu0})  and $C$ in
(\ref{eq:forward-sde})
have the same basis of eigenfunctions $(v_j)$
 and we define $X^{(j)}_t =\langle X_t,v_j \rangle$ and
$S^{(j)}(t,x,y) = \langle S(t,x,y),v_j \rangle$ so that
in (\ref{prop:Sbound}) $S(t,x,y) = \sum_j S^{(j)}(t,x,y) v_j$.
We assume an observational model as described in Section \ref{sec:gaussian}
and that the $p^{(j)}(1+q^{(j)})$ are uniformly bounded with respect
to $j$ and that $C$ is of trace class.
We make a modified version of assumption in (\ref{eq:mumu0}) as follows. We assume that
1) the conditional distribution of $X_0$ given $Y=y$ is absolutely continuous with respect to the
Gaussian measure $\mu$ with a  Radon-Nikodym derivative proportional to $\exp(-\Phi(x_0,y))$;
2) we have $\Phi(x_0,y)=\sum_j \Phi^{(j)}(x_0^{(j)},y)$, $x_0^{(j)}=\langle x_0,v_j \rangle$;
3) for $\psi^{(j)}(x^{(j)},y)=\exp(-\Phi^{(j)}(x^{(j)},y))$ we have
\begin{equation}
         \frac{1}{K}  \leq  |\psi^{(j)}(x^{(j)},y)| \leq K , \qquad
      |\psi^{(j)}(x^{(j)},y) -\psi^{(j)}({x^{(j)}}',y)|  \leq  L |{x^{(j)}}'-x^{(j)}| ,
\end{equation}
where $K$ and $L$ do not depend on $j$.
Then Assumption \ref{assumption1} holds true.
\end{proposition}


\begin{proof}
    The proof is given in Appendix C.
\end{proof}
\vspace*{-0.75em}
To use the new score function of Definition \ref{def:conditional-score} for sampling from the posterior, we need to define a way to estimate it. In other words, we need to define a loss function over which the difference between the true score and a neural network $s_\theta(t,x_t,y)$ is minimized in $\theta$. A natural choice for the loss function is
\begin{equation}
	\mathbb{E}_{ {t \sim U(0,T),}x_t,y \sim {\cal L}(X_t,Y)} \big[ \|  
 S(t,x_t,y) -s_\theta(t,x_t,y)\|_H^2\big],
\end{equation}
however it cannot be minimized directly since we do not have access to the ground truth conditional score 
$S(t,x_t,y)$.
Therefore, in practice, a different objective has to be used. \citet{batzolis2021conditional} proved that, in finite dimensions, a denoising score matching loss can be used:
\begin{equation}
	\mathbb{E}_{ {t \sim U(0,T),}x_0,y \sim {\cal L}(X_0,Y), x_t \sim {\cal L}(X_t|X_0=x_0)} \big[ \|  C \nabla_{x_t} \ln p(x_t|x_0) -s_\theta(t,x_t,y)\|^2\big] .
\label{eq:objective}
\end{equation}
This expression involves only $\nabla_{x_t} \log p(x_t|x_0)$ which can be computed analytically from the transition kernel of the forward diffusion process, also in infinite dimensions. 
In the following proposition, we build on the arguments of \citet{batzolis2021conditional} and provide a proof that the conditional denoising estimator is a consistent estimator of the conditional score in infinite dimensions.

\begin{proposition} Under Assumption \ref{assumption1},
the minimizer in $\theta$ of
\begin{equation}
	\label{eq:carola1}
	\mathbb{E}_{x_0,y \sim {\cal L}(X_0,Y), x_t \sim {\cal L}(X_t|X_0=x_0)} \big[ \|
- (1-e^{-t})^{-1} ( x_t -e^{-t/2} x_0)
 - s_\theta(t,x_t,y)\|_H^2\big]
\end{equation}
is the same as the minimizer of
\begin{equation}
	\label{eq:carola2}
	\mathbb{E}_{x_t,y \sim {\cal L}(X_t,Y)} \big[ \|
 S(t,x_t,y) -s_\theta(t,x_t,y)\|_H^2\big] .
\end{equation}
The same result holds if we add $t \sim {\cal U}(0,T)$ in the expectations.
\end{proposition}
\begin{proof}
The proof combines some of the arguments of \citet{batzolis2021conditional} and steps of the proof of Lemma 2 in \cite{pidstrigach2023infinitedimensional}, see Appendix C.
\end{proof}

\begin{remark}
A statement of robustness can be written as in \cite[Theorem 2]{pidstrigach2023infinitedimensional}. 
\end{remark}

\section{Numerical experiments} \label{sec:numerics}
\vspace*{-0.4em}

To put the presented theoretical results into practice, we provide two examples. The first stylized example aims at showcasing (i) the ability of our method in capturing nontrivial conditional distributions; and (ii) the discretization-invariance property of the learned conditional SDM. In the second example, we sample from the posterior distribution of a linearized seismic imaging problem in order to demonstrate the applicability of our method to large-scale problems. In both examples, in order to enable learning in function spaces, we parameterize the conditional score using Fourier neural operators~\citep{li2021fourier}. Details regarding our experiment and implementation\footnote{Code to reproduce results can be found at \href{https://github.com/alisiahkoohi/csgm}{https://github.com/alisiahkoohi/csgm}.} are presented at Appendix D.

\paragraph{Stylized example}
Inspired by~\citet{phillips2022spectral}, we define the target density via the relation $x_0 = a y^2 + \eps$ with $\eps \sim \Gamma (1, 2)$, $a \sim \mathcal{U}\{-1, 1\}$, and $y \in [-3, 3]$. Figure~\ref{fig:true_samples} illustrates the samples $x_0$ evaluated on a fine $y$ grid. After training (details in Appendix D), we sample the conditional distribution on uniformly sampled grids between $[-3, 3]$, each having 20 to 40 grid points. Figures~\ref{fig:grid_25} and \ref{fig:grid_35} show the predicted samples for grid sizes of $25$ and $35$, respectively. The marginal conditionals associated with $y=-1.0, 0.0, 0.5$ are shown in Figures~\ref{fig:marginal_left}--\ref{fig:marginal_right}, respectively. The gray shaded density in the bottom row of Figure~\ref{fig:toy_example} indicates the ground truth density, and colored estimated densities correspond to different discretizations of the horizontal axis. The visual inspection of samples and estimated densities indicates that our approach is indeed discretization-invariant.

\paragraph{Linearized seismic imaging example}
In this experiment, we address the problem of estimating the short-wavelength component of the Earth's subsurface squared-slowness model
(i.e., seismic image; cf. Figure~\ref{fig:true_model}) given surface measurements and a long-wavelength, smooth squared-slowness model (cf. Figure~\ref{fig:background}). Following~\citet{OrozcoSiahkoohiEtAl_2023}, in order to reduce the high dimensionality of surface measurements, we apply the adjoint of the forward operator, the Born scattering operator, to the measurements and use the outcome (cf. Figure~\ref{fig:observed_data}) instead of measured data to condition the SDM. After training, given previously unseen observed data, we use the SDM to sample $10^3$ posterior samples to estimate the conditional mean (cf. Figure~\ref{fig:conditional_mean}), which corresponds to the minimum-variance estimate~\citep{AndersonMoore1979}, and the pointwise standard deviation (cf. Figure~\ref{fig:pointwise_std}), which we use to quantify the uncertainty. As expected, the pointwise standard deviation highlights areas of high uncertainty, particularly in regions with complex geological structures---such as near intricate reflectors and areas with limited illumination (deep and close to boundaries). We also observe a strong correlation between the pointwise standard deviation and the error in the conditional mean estimate (Figure~\ref{fig:error}), confirming the accuracy of our Bayesian inference method.

\begin{figure}[t]
    \centering
    \captionsetup[subfigure]{skip=-11pt}
    \begin{subfigure}[b]{0.329\textwidth}
        \includegraphics[width=\textwidth]{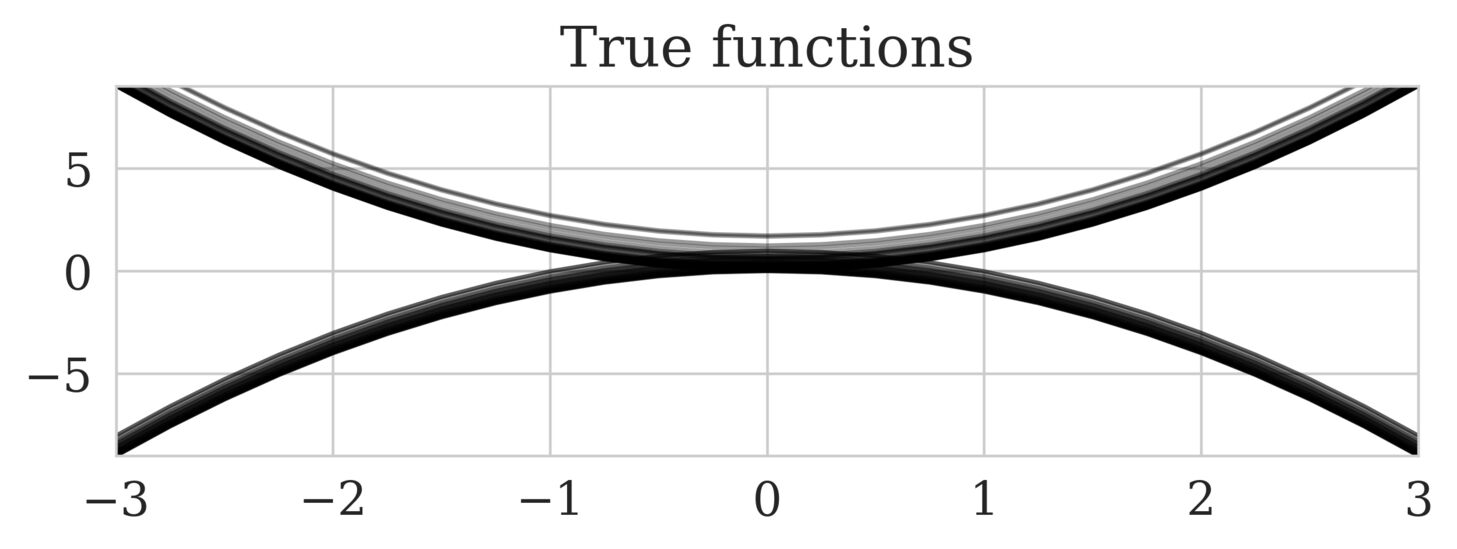}
    \vspace{0ex}\caption{}
    \label{fig:true_samples}
    \end{subfigure}\hspace{0em}
    \begin{subfigure}[b]{0.329\textwidth}
        \includegraphics[width=\textwidth]{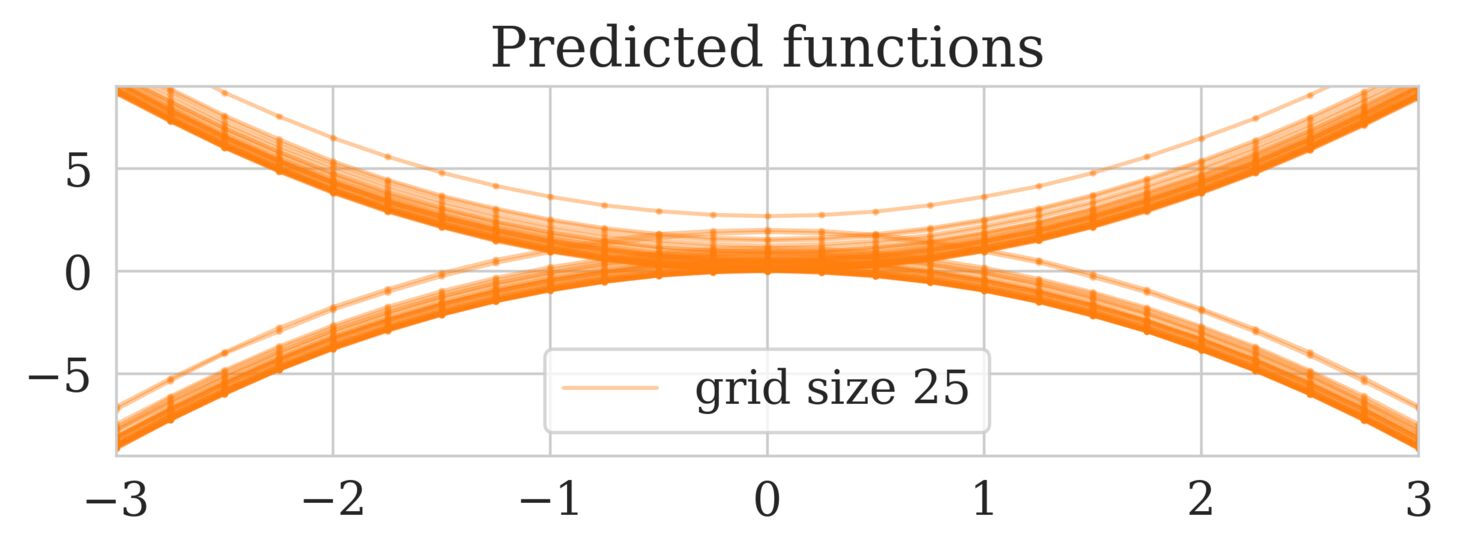}
    \vspace{0ex}\caption{}
    \label{fig:grid_25}
    \end{subfigure}\hspace{0em}
    \begin{subfigure}[b]{0.329\textwidth}
        \includegraphics[width=\textwidth]{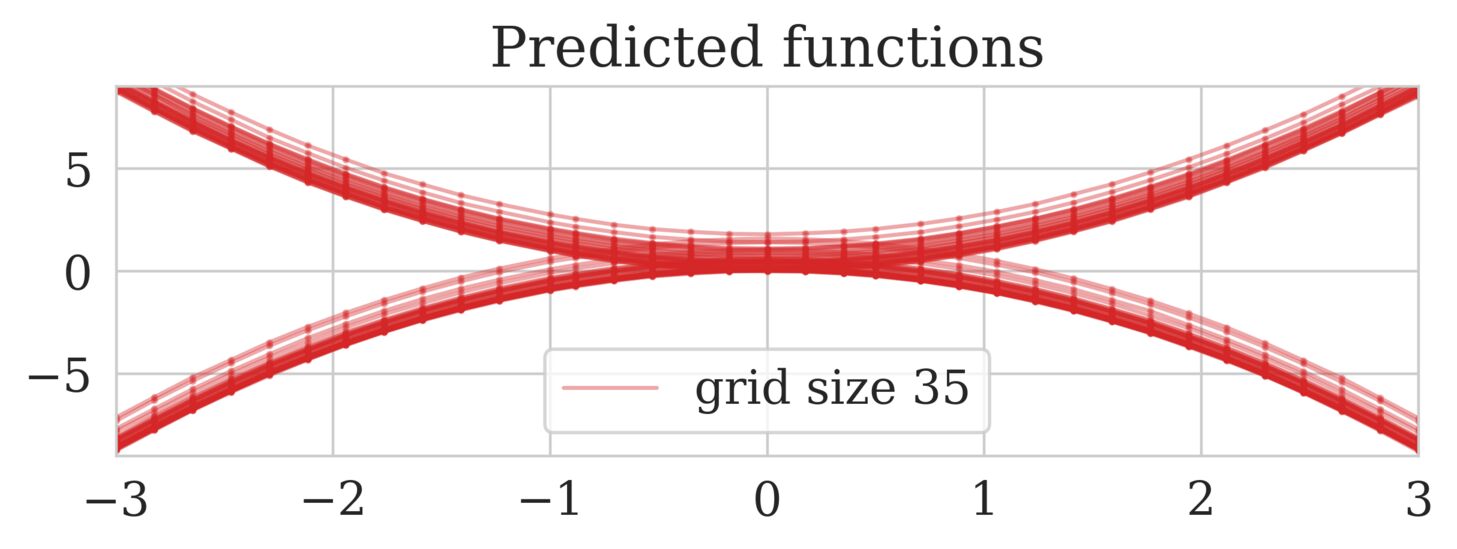}
    \vspace{0ex}\caption{}
    \label{fig:grid_35}
    \end{subfigure}\hspace{0em}

    \begin{subfigure}[b]{0.329\textwidth}
        \includegraphics[width=\textwidth]{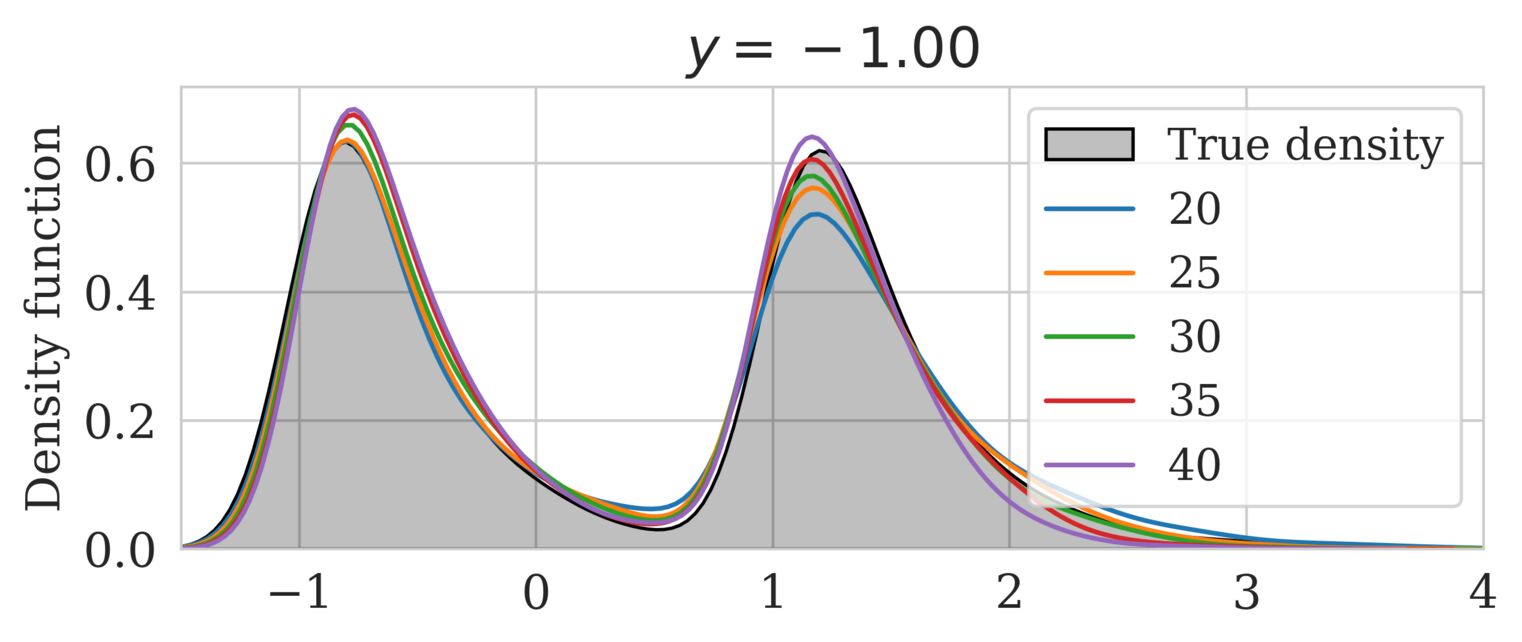}
    \vspace{0ex}\caption{}
    \label{fig:marginal_left}
    \end{subfigure}\hspace{0em}
    \begin{subfigure}[b]{0.329\textwidth}
        \includegraphics[width=\textwidth]{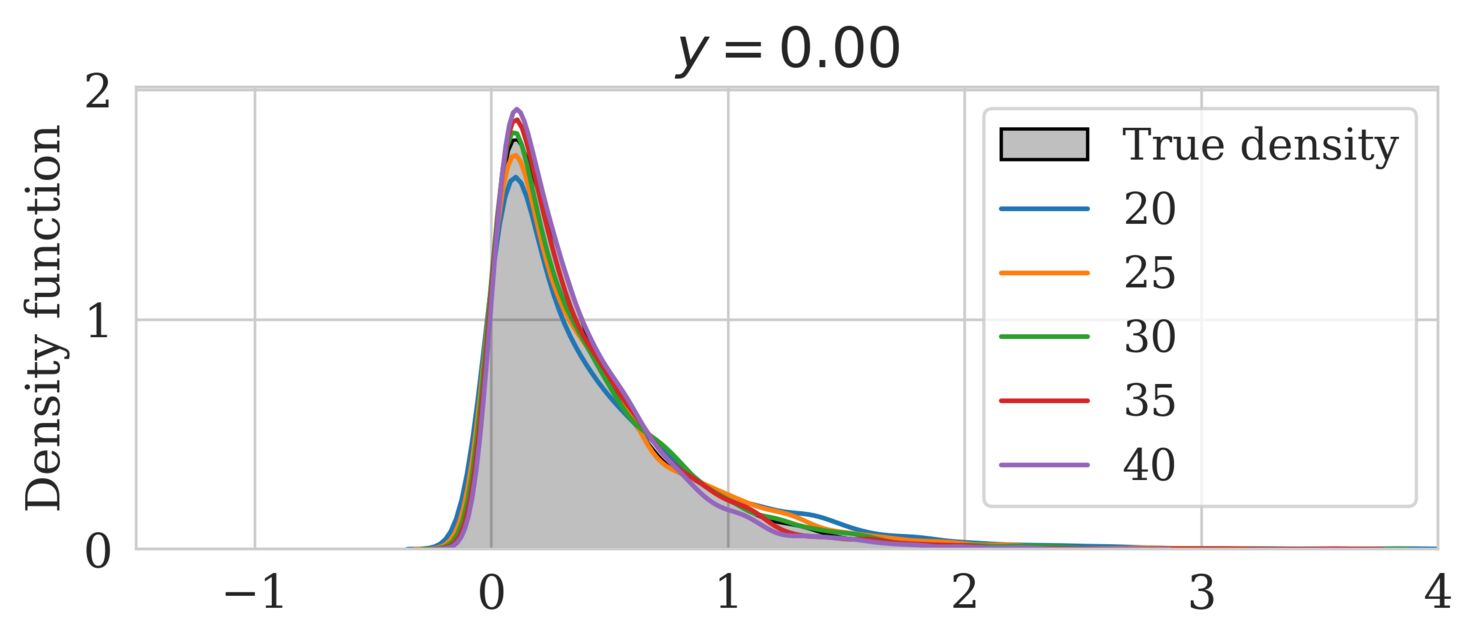}
    \vspace{0ex}\caption{}
    \label{fig:marginal_center}
    \end{subfigure}\hspace{0em}
    \begin{subfigure}[b]{0.329\textwidth}
        \includegraphics[width=\textwidth]{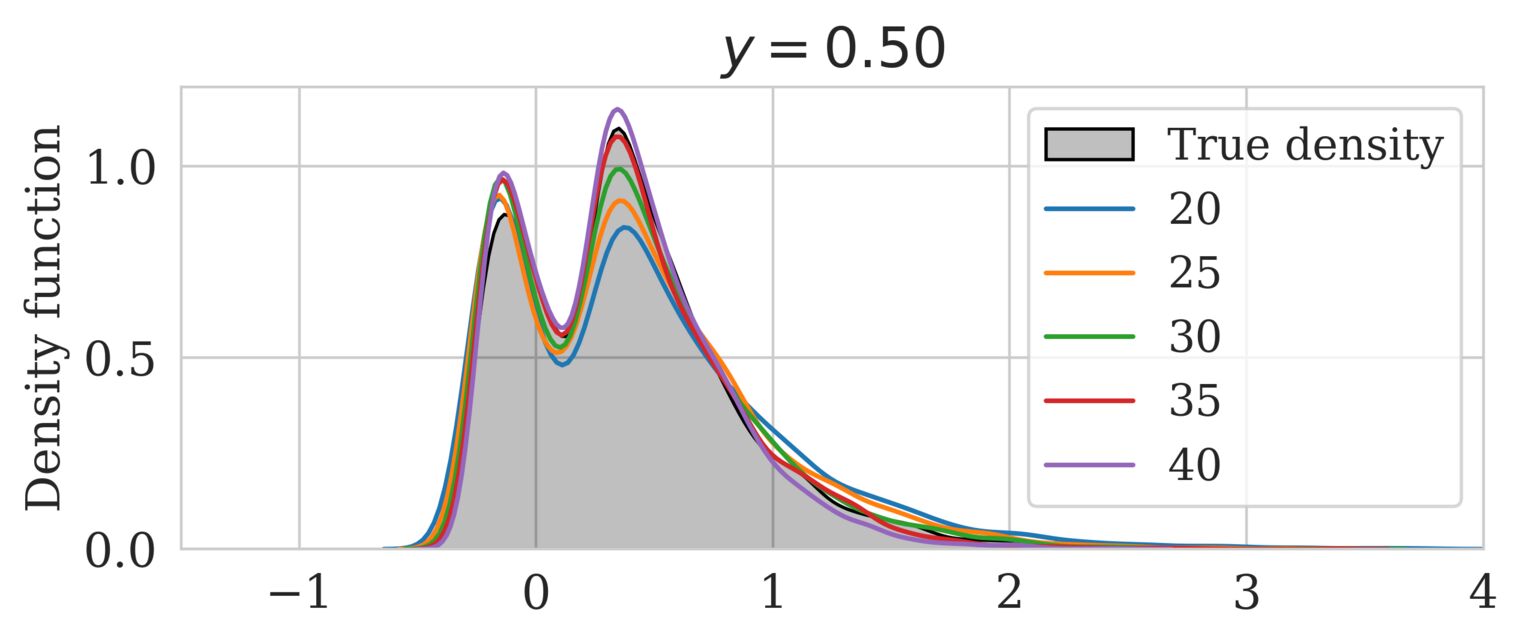}
    \vspace{0ex}\caption{}
    \label{fig:marginal_right}
    \end{subfigure}\hspace{0em}

    \caption{A depiction of the method's discretization invariance. Top row displays ground truth (a) and predicted samples (functions) on a uniformly sampled grid with (b) 25 and (c) 35 grid points. Bottom row shows conditional distribution marginals for (d) $y=-1.0$, (e) $y=0.0$, and (f) $y=0.5$.}
    \label{fig:toy_example}
\end{figure}

\begin{figure}[t]
    \centering
    \captionsetup[subfigure]{skip=-11pt}

        \begin{subfigure}[b]{0.325\textwidth}
            \includegraphics[width=\textwidth]{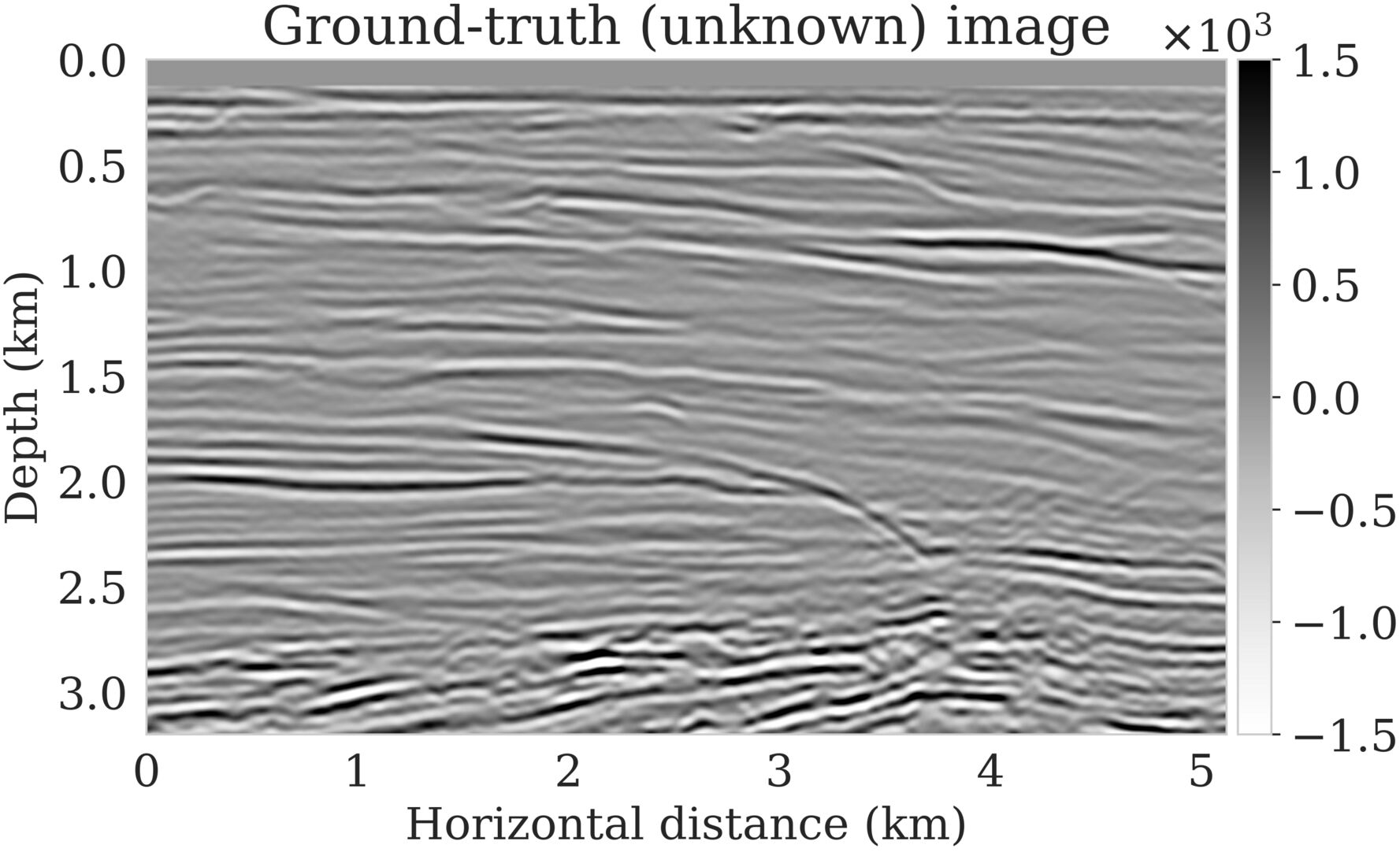}
        \vspace{0ex}\caption{}
        \label{fig:true_model}
        \end{subfigure}\hspace{0em}
        \begin{subfigure}[b]{0.32\textwidth}
            \includegraphics[width=\textwidth]{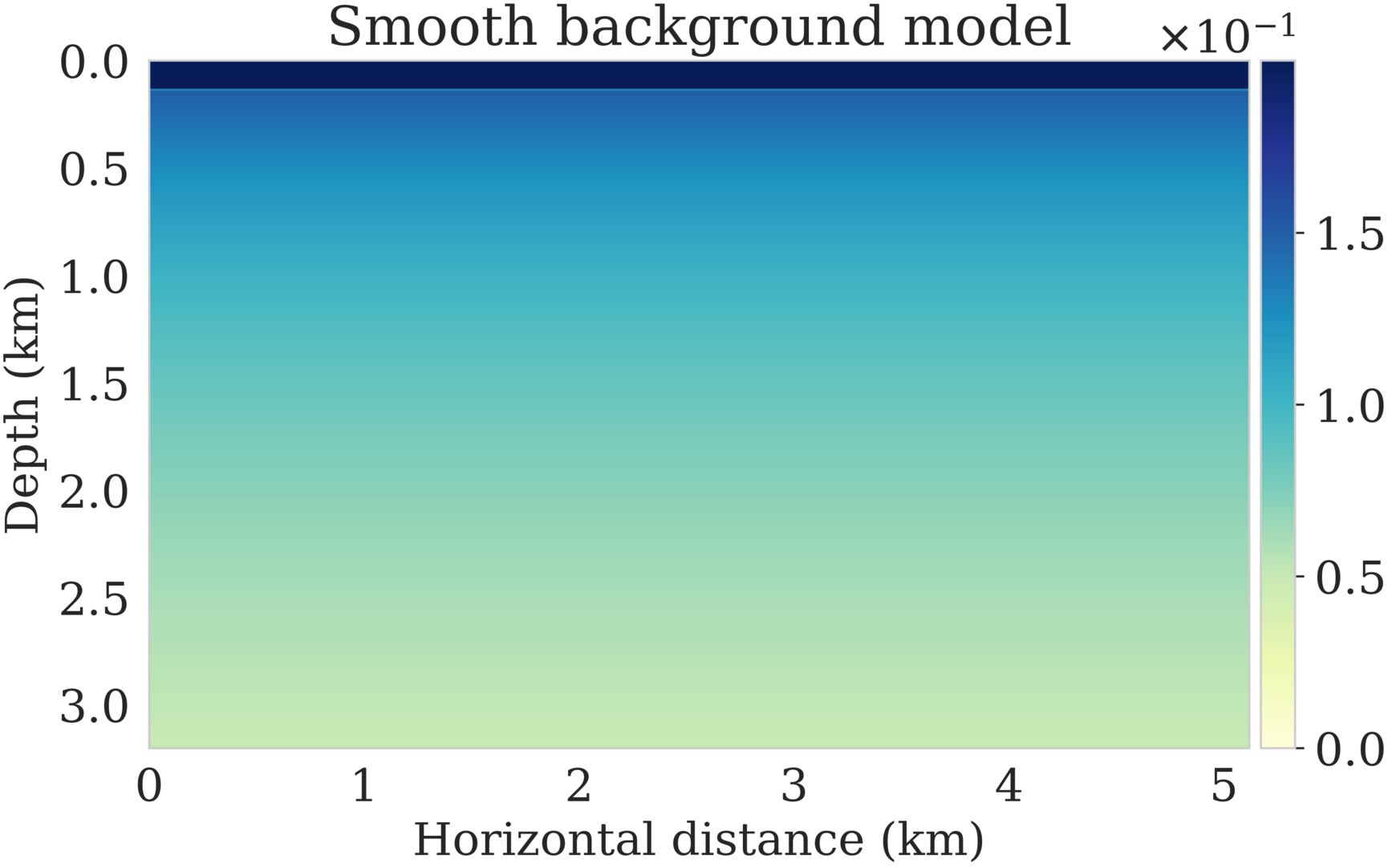}
        \vspace{0ex}\caption{}
        \label{fig:background}
        \end{subfigure}
        \begin{subfigure}[b]{0.325\textwidth}
            \includegraphics[width=\textwidth]{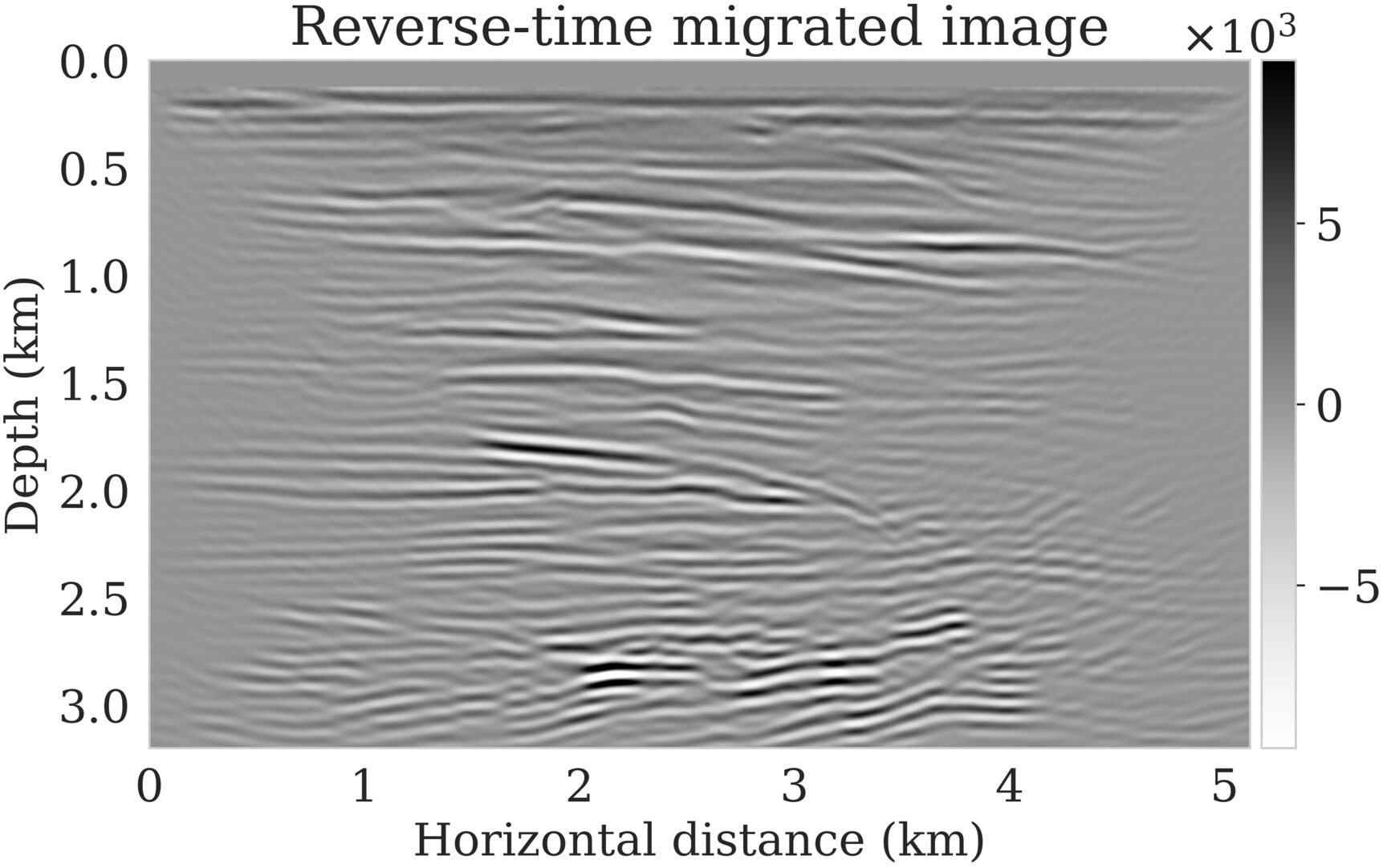}
        \vspace{0ex}\caption{}
        \label{fig:observed_data}
        \end{subfigure}\hspace{0em}

        \begin{subfigure}[b]{0.325\textwidth}
            \includegraphics[width=\textwidth]{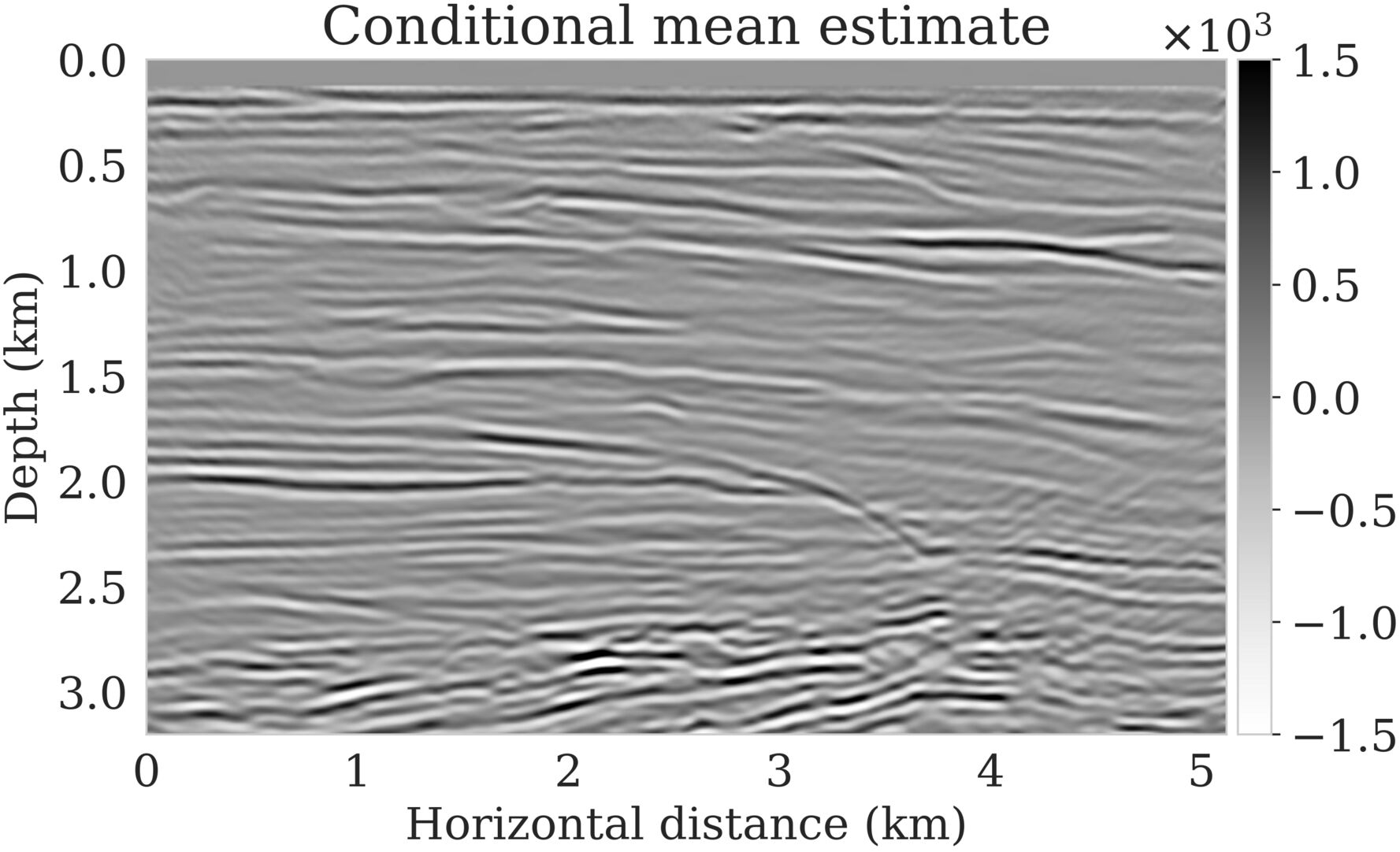}
        \vspace{0ex}\caption{}
        \label{fig:conditional_mean}
        \end{subfigure}\hspace{0em}
        \begin{subfigure}[b]{0.325\textwidth}
            \includegraphics[width=\textwidth]{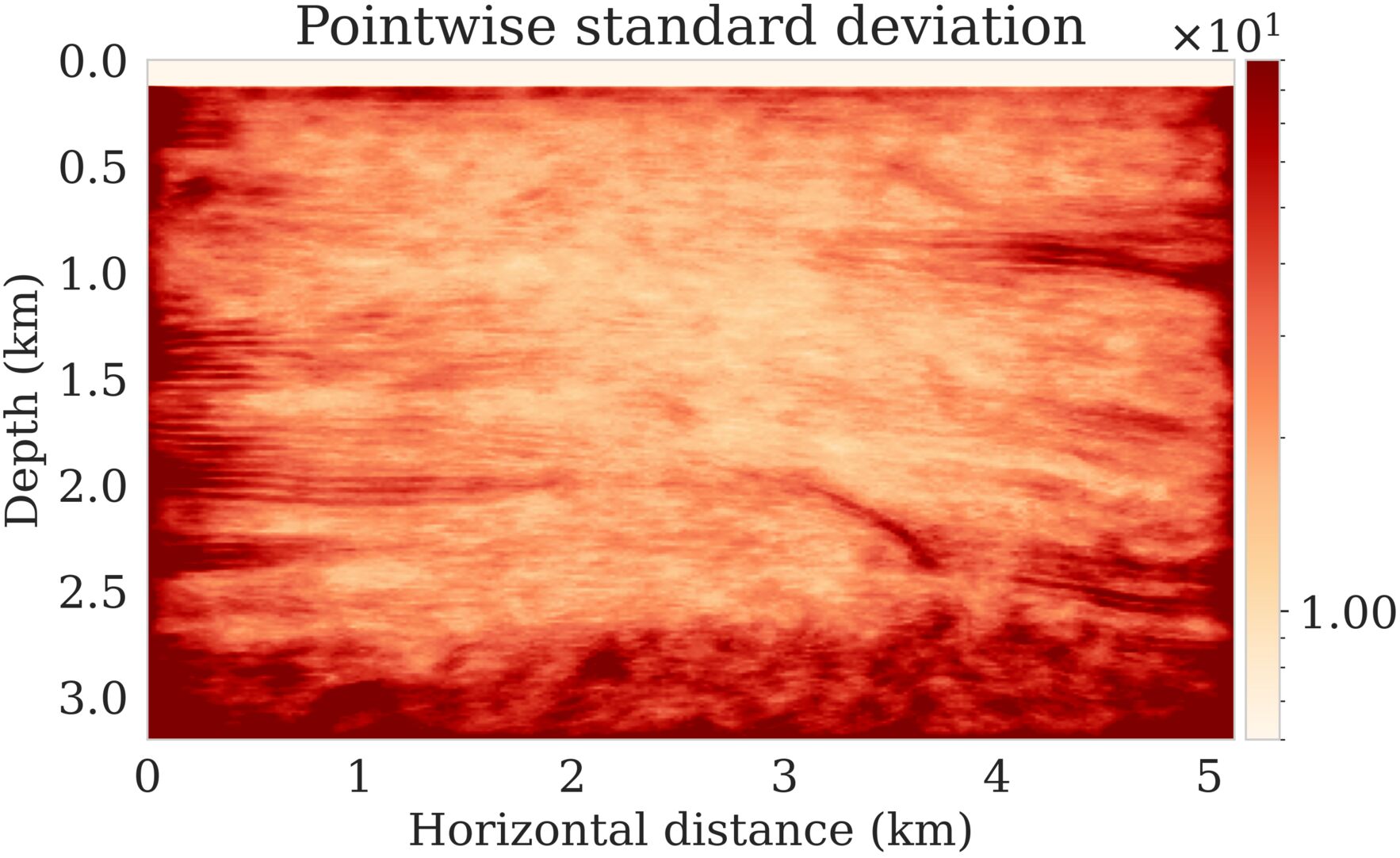}
        \vspace{0ex}\caption{}
        \label{fig:pointwise_std}
        \end{subfigure}\hspace{0em}
        \begin{subfigure}[b]{0.325\textwidth}
            \includegraphics[width=\textwidth]{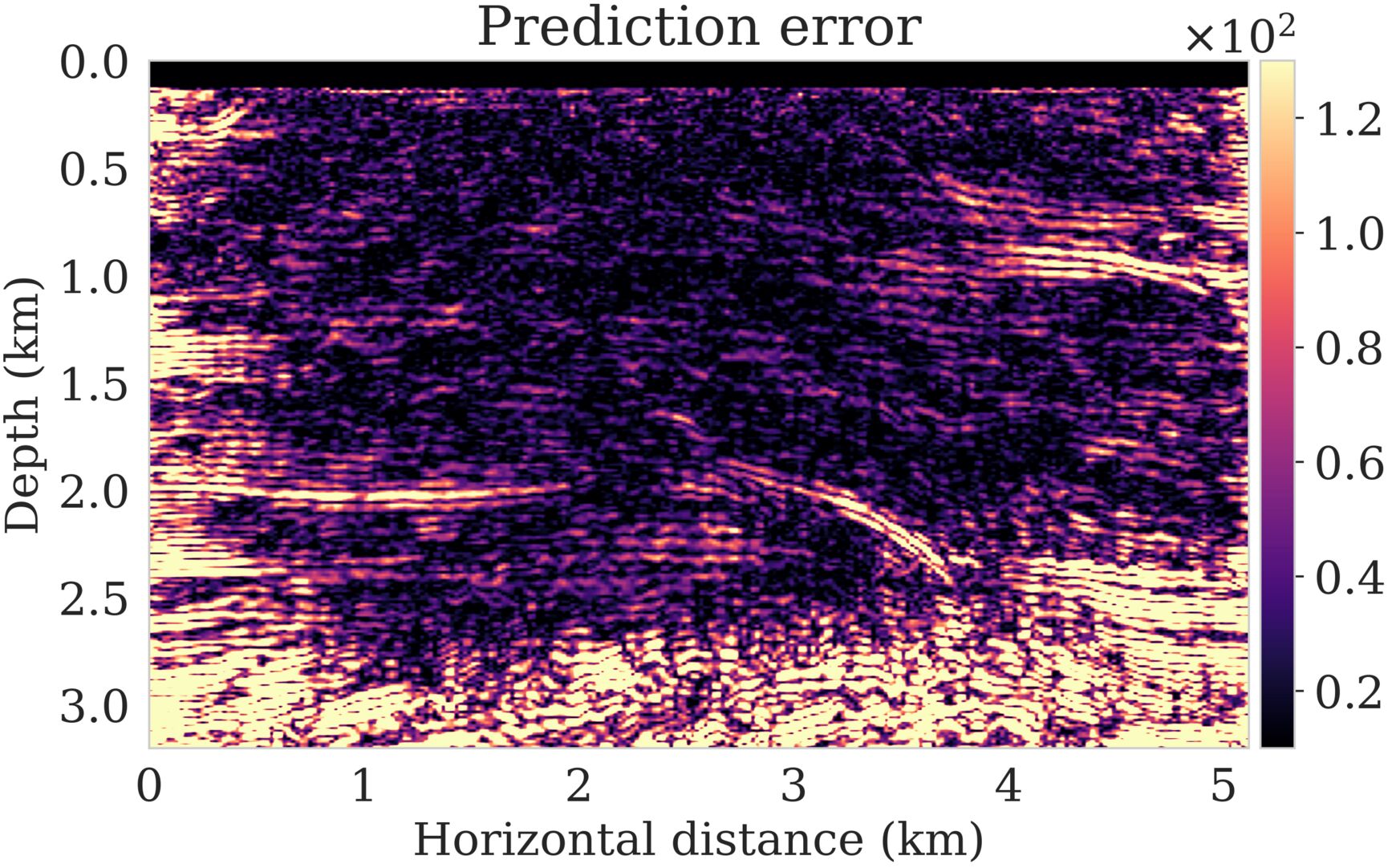}
        \vspace{0ex}\caption{}
        \label{fig:error}
        \end{subfigure}\hspace{0em}

\caption{Seismic imaging and uncertainty quantification. (a) Ground-truth seismic image. (b) Background squared-slowness. (c) Data after applying the adjoint Born operator. (d) Conditional (posterior) mean. (e) Pointwise standard deviation. (f) Absolute error between Figures~\ref{fig:true_model} and~\ref{fig:conditional_mean}.}
\end{figure}

\section{Conclusions}

We introduced a theoretically-grounded method that \emph{is able to perform conditional sampling in infinite-dimensional Hilbert (function) spaces using score-based diffusion models}. This is a foundational step in using diffusion models to perform Bayesian inference. To achieve this, we learned the infinite-dimensional score function, as defined by \citet{pidstrigach2023infinitedimensional}, conditioned on the observed data. Under mild assumptions on the prior, this newly defined score---used as the reverse drift of the diffusion process---yields a generative model that samples from the posterior of a linear inverse problem. In particular, the well-known singularity in the conditional score for small times can be avoided.
Building on these results, we presented stylized and large-scale examples that showcase the validity of our method and its discretization-invariance, a property that is a consequence of our theoretical and computational framework being built on infinite-dimensional spaces.


\section*{Acknowledgments}
JG was supported by Agence de l'Innovation de D\'efense – AID - via Centre Interdisciplinaire d’Etudes pour la D\'efense et la S\'ecurit\'e – CIEDS - (project 2021 - PRODIPO). LB, AS, and MVdH acknowledge support from the Simons Foundation under the MATH\,$+$\,X program, the Department of Energy under grant DE-SC0020345, and the corporate members of the Geo-Mathematical Imaging Group at Rice University.
KS was supported by  Air Force Office of Scientific Research under grant FA9550-22-1-0176 and the National Science Foundation under grant DMS-2308389.

\bibliography{bib_camera_ready}

\begin{thebibliography}{61}
\providecommand{\natexlab}[1]{#1}
\providecommand{\url}[1]{\texttt{#1}}
\expandafter\ifx\csname urlstyle\endcsname\relax
  \providecommand{\doi}[1]{doi: #1}\else
  \providecommand{\doi}{doi: \begingroup \urlstyle{rm}\Url}\fi

\bibitem[Hadamard(1923)]{hadamard1923lectures}
J. Hadamard.
\newblock \emph{Lectures on the {C}auchy's Problem in Linear Partial
  Differential Equations}.
\newblock Yale University Press, 1923.

\bibitem[Aster et~al.(2018)Aster, Borchers, and Thurber]{aster2018parameter}
R.~C. Aster, B. Borchers, and C.~H. Thurber.
\newblock \emph{Parameter estimation and inverse problems}.
\newblock Elsevier, 2018.

\bibitem[Ito and Jin(2014)]{ito2014inverse}
K. Ito and B. Jin.
\newblock \emph{Inverse Problems}.
\newblock World Scientific, 2014.

\bibitem[Lehtinen et~al.(1989)Lehtinen, Paivarinta, and
  Somersalo]{lehtinen1989linear}
M.~S. Lehtinen, L. Paivarinta, and E. Somersalo.
\newblock Linear inverse problems for generalised random variables.
\newblock \emph{Inverse Problems}, 5\penalty0 (4):\penalty0 599, 1989.

\bibitem[Stuart(2010)]{stuart2010inverse}
A.~M. Stuart.
\newblock Inverse problems: {A} {B}ayesian perspective.
\newblock \emph{Acta Numerica}, 19:\penalty0 451--559, 2010.

\bibitem[Tarantola(2005)]{tarantola2005inverse}
A. Tarantola.
\newblock \emph{Inverse problem theory and methods for model parameter
  estimation}.
\newblock SIAM, 2005.

\bibitem[Knapik et~al.(2011)Knapik, van~der Vaart, and van
  Zanten]{knapik2011bayesian}
B.~T. Knapik, A.~W. van~der Vaart, and J.~H. van Zanten.
\newblock Bayesian inverse problems with {G}aussian priors.
\newblock \emph{The Annals of Statistics}, 39\penalty0 (5):\penalty0 2626 --
  2657, 2011.

\bibitem[Bennett(2002)]{bennett2005inverse}
A.~F. Bennett.
\newblock \emph{Inverse Modeling of the Ocean and Atmosphere}.
\newblock Cambridge University Press, 2002.

\bibitem[Stuart(2014)]{stuart2014uncertainty}
A. Stuart.
\newblock Uncertainty quantification in {B}ayesian inversion.
\newblock In \emph{2014 SIAM Conference on Uncertainty Quantification}, 2014.

\bibitem[Song et~al.(2021)Song, Sohl-Dickstein, Kingma, Kumar, Ermon, and
  Poole]{song2021scorebased}
Y. Song, J. Sohl-Dickstein, D.~P. Kingma, A. Kumar, S. Ermon, and B. Poole.
\newblock Score-based generative modeling through stochastic differential
  equations.
\newblock In \emph{International Conference on Learning Representations}, 2021.

\bibitem[Kawar et~al.(2021)Kawar, Vaksman, and Elad]{kawar2021snips}
B. Kawar, G. Vaksman, and M. Elad.
\newblock {SNIPS}: {S}olving noisy inverse problems stochastically.
\newblock In \emph{Advances in Neural Information Processing Systems},
  volume~34, pages 21757--21769, 2021.

\bibitem[Song et~al.(2022)Song, Shen, Xing, and Ermon]{song2022solving}
Y. Song, L. Shen, L. Xing, and S. Ermon.
\newblock Solving inverse problems in medical imaging with score-based
  generative models.
\newblock In \emph{International Conference on Learning Representations}, 2022.

\bibitem[Anderson(1982)]{anderson1982reverse}
B.~D. Anderson.
\newblock Reverse-time diffusion equation models.
\newblock \emph{Stochastic Processes and their Applications}, 12\penalty0
  (3):\penalty0 313--326, 1982.

\bibitem[Dhariwal and Nichol(2021)]{dhariwal2021diffusion}
P. Dhariwal and A. Nichol.
\newblock Diffusion models beat {GAN}s on image synthesis.
\newblock In \emph{Advances in Neural Information Processing Systems},
  volume~34, pages 8780--8794, 2021.

\bibitem[Gnaneshwar et~al.(2022)Gnaneshwar, Ramsundar, Gandhi, Kurchin, and
  Viswanathan]{gnaneshwar2022score}
D. Gnaneshwar, B. Ramsundar, D. Gandhi, R. Kurchin, and V. Viswanathan.
\newblock Score-based generative models for molecule generation, 2022.
\newblock URL \url{https://arxiv.org/abs/2203.04698}.
\newblock Preprint.

\bibitem[Lee et~al.(2023)Lee, Lu, and Tan]{lee2023convergence}
H. Lee, J. Lu, and Y. Tan.
\newblock Convergence of score-based generative modeling for general data
  distributions.
\newblock In \emph{Proceedings of The 34th International Conference on
  Algorithmic Learning Theory}, volume 201, pages 946--985, 2023.

\bibitem[Song and Ermon(2019)]{song2019generative}
Y. Song and S. Ermon.
\newblock Generative modeling by estimating gradients of the data distribution.
\newblock In \emph{Advances in Neural Information Processing Systems},
  volume~32, 2019.

\bibitem[Jalal et~al.(2021)Jalal, Arvinte, Daras, Price, Dimakis, and
  Tamir]{jalal2021robust}
A. Jalal, M. Arvinte, G. Daras, E. Price, A.~G. Dimakis, and J. Tamir.
\newblock Robust compressed sensing {MRI} with deep generative priors.
\newblock \emph{Advances in Neural Information Processing Systems},
  34:\penalty0 14938--14954, 2021.

\bibitem[Dashti and Stuart(2017)]{dashti2013bayesian}
M. Dashti and A.~M. Stuart.
\newblock \emph{The Bayesian approach to inverse problems}, pages 311--428.
\newblock Springer International Publishing, 2017.

\bibitem[Uhlmann and Vasy(2016)]{uhlmann2016inverse}
G. Uhlmann and A. Vasy.
\newblock The inverse problem for the local geodesic ray transform.
\newblock \emph{Inventiones Mathematicae}, 205\penalty0 (1):\penalty0 83--120,
  2016.

\bibitem[Dynin(1978)]{dynin1978inversion}
A. Dynin.
\newblock Inversion problem for singular integral operators:
  {C}$^\ast$-approach.
\newblock \emph{Proceedings of the National Academy of Sciences}, 75\penalty0
  (10):\penalty0 4668--4670, 1978.

\bibitem[Chen et~al.(2023)Chen, Chewi, Li, Li, Salim, and
  Zhang]{chen2022sampling}
S. Chen, S. Chewi, J. Li, Y. Li, A. Salim, and A. Zhang.
\newblock Sampling is as easy as learning the score: {T}heory for diffusion
  models with minimal data assumptions.
\newblock In \emph{The Eleventh International Conference on Learning
  Representations}, 2023.

\bibitem[Bortoli(2022)]{de2022convergence}
V.~D. Bortoli.
\newblock Convergence of denoising diffusion models under the manifold
  hypothesis.
\newblock \emph{Transactions on Machine Learning Research}, 2022.

\bibitem[Pidstrigach et~al.(2023)Pidstrigach, Marzouk, Reich, and
  Wang]{pidstrigach2023infinitedimensional}
J. Pidstrigach, Y. Marzouk, S. Reich, and S. Wang.
\newblock Infinite-dimensional diffusion models for function spaces, 2023.
\newblock URL \url{https://arxiv.org/abs/2302.10130v1}.
\newblock Preprint.

\bibitem[Arora et~al.(2018)Arora, Risteski, and Zhang]{arora2018gans}
S. Arora, A. Risteski, and Y. Zhang.
\newblock Do {GAN}s learn the distribution? {S}ome theory and empirics.
\newblock In \emph{International Conference on Learning Representations}, 2018.

\bibitem[Lasanen(2007)]{lasanen2007measurements}
S. Lasanen.
\newblock Measurements and infinite-dimensional statistical inverse theory.
\newblock \emph{Proceedings in Applied Mathematics and Mechanics}, 7:\penalty0
  1080101--1080102, 2007.

\bibitem[Hyv{{\"a}}rinen(2005)]{hyvarinen2005estimation}
A. Hyv{{\"a}}rinen.
\newblock Estimation of non-normalized statistical models by score matching.
\newblock \emph{Journal of Machine Learning Research}, 6\penalty0
  (24):\penalty0 695--709, 2005.

\bibitem[Sohl-Dickstein et~al.(2015)Sohl-Dickstein, Weiss, Maheswaranathan, and
  Ganguli]{sohl2015deep}
J. Sohl-Dickstein, E. Weiss, N. Maheswaranathan, and S. Ganguli.
\newblock Deep unsupervised learning using nonequilibrium thermodynamics.
\newblock In \emph{Proceedings of the 32nd International Conference on Machine
  Learning}, volume~37, pages 2256--2265, 2015.

\bibitem[Ho et~al.(2020)Ho, Jain, and Abbeel]{ho2020denoising}
J. Ho, A. Jain, and P. Abbeel.
\newblock Denoising diffusion probabilistic models.
\newblock In \emph{Advances in Neural Information Processing Systems},
  volume~33, pages 6840--6851, 2020.

\bibitem[Dupont et~al.(2022)Dupont, Kim, Eslami, Rezende, and
  Rosenbaum]{dupont2022data}
E. Dupont, H. Kim, S.~M.~A. Eslami, D.~J. Rezende, and D. Rosenbaum.
\newblock From data to functa: {Y}our data point is a function and you can
  treat it like one.
\newblock In \emph{Proceedings of the 39th International Conference on Machine
  Learning}, volume 162, pages 5694--5725, 2022.

\bibitem[Phillips et~al.(2022)Phillips, Seror, Hutchinson, Bortoli, Doucet, and
  Mathieu]{phillips2022spectral}
A. Phillips, T. Seror, M.~J. Hutchinson, V.~D. Bortoli, A. Doucet, and E.
  Mathieu.
\newblock Spectral diffusion processes.
\newblock In \emph{NeurIPS 2022 Workshop on Score-Based Methods}, 2022.

\bibitem[Kerrigan et~al.(2023)Kerrigan, Ley, and Smyth]{kerrigan2022diffusion}
G. Kerrigan, J. Ley, and P. Smyth.
\newblock Diffusion generative models in infinite dimensions.
\newblock In \emph{Proceedings of The 26th International Conference on
  Artificial Intelligence and Statistics}, volume 206, pages 9538--9563, 2023.

\bibitem[Dutordoir et~al.(2023)Dutordoir, Saul, Ghahramani, and
  Simpson]{dutordoir2023neural}
V. Dutordoir, A. Saul, Z. Ghahramani, and F. Simpson.
\newblock Neural diffusion processes.
\newblock In \emph{Proceedings of the 40th International Conference on Machine
  Learning}, volume 202, pages 8990--9012, 2023.

\bibitem[Lim et~al.(2023)Lim, Kovachki, Baptista, Beckham, Azizzadenesheli,
  Kossaifi, Voleti, Song, Kreis, Kautz, Pal, Vahdat, and
  Anandkumar]{lim2023score}
J.~H. Lim, N.~B. Kovachki, R. Baptista, C. Beckham, K. Azizzadenesheli, J.
  Kossaifi, V. Voleti, J. Song, K. Kreis, J. Kautz, C. Pal, A. Vahdat, and A.
  Anandkumar.
\newblock Score-based diffusion models in function space, 2023.
\newblock URL \url{https://arxiv.org/abs/2302.07400}.
\newblock Preprint.

\bibitem[Franzese et~al.(2023)Franzese, Corallo, Rossi, Heinonen, Filippone,
  and Michiardi]{franzese2023continuous}
G. Franzese, G. Corallo, S. Rossi, M. Heinonen, M. Filippone, and P. Michiardi.
\newblock Continuous-time functional diffusion processes, 2023.
\newblock URL \url{https://arxiv.org/abs/2303.00800}.
\newblock Preprint.

\bibitem[Föllmer and Wakolbinger(1986)]{follmer1986time}
H. Föllmer and A. Wakolbinger.
\newblock Time reversal of infinite-dimensional diffusions.
\newblock \emph{Stochastic Processes and their Applications}, 22\penalty0
  (1):\penalty0 59--77, 1986.

\bibitem[Millet et~al.(1989)Millet, Nualart, and Sanz]{millet1989time}
A. Millet, D. Nualart, and M. Sanz.
\newblock Time reversal for infinite-dimensional diffusions.
\newblock \emph{Probability Theory and Related Fields}, 82\penalty0
  (3):\penalty0 315--347, 1989.

\bibitem[Batzolis et~al.(2021)Batzolis, Stanczuk, Schönlieb, and
  Etmann]{batzolis2021conditional}
G. Batzolis, J. Stanczuk, C.-B. Schönlieb, and C. Etmann.
\newblock Conditional image generation with score-based diffusion models, 2021.
\newblock URL \url{https://arxiv.org/abs/2111.13606}.
\newblock Preprint.

\bibitem[Vincent(2011)]{vincent2011connection}
P. Vincent.
\newblock A connection between score matching and denoising autoencoders.
\newblock \emph{Neural Computation}, 23\penalty0 (7):\penalty0 1661--1674,
  2011.

\bibitem[Kim et~al.(2022)Kim, Shin, Song, Kang, and Moon]{kim2021soft}
D. Kim, S. Shin, K. Song, W. Kang, and I.-C. Moon.
\newblock Soft truncation: {A} universal training technique of score-based
  diffusion model for high precision score estimation.
\newblock In \emph{Proceedings of the 39th International Conference on Machine
  Learning}, volume 162, pages 11201--11228, 2022.

\bibitem[Dockhorn et~al.(2022)Dockhorn, Vahdat, and Kreis]{dockhorn2021score}
T. Dockhorn, A. Vahdat, and K. Kreis.
\newblock Score-based generative modeling with critically-damped {L}angevin
  diffusion.
\newblock In \emph{International Conference on Learning Representations}, 2022.

\bibitem[Cranmer et~al.(2020)Cranmer, Brehmer, and Louppe]{cranmer2020frontier}
K. Cranmer, J. Brehmer, and G. Louppe.
\newblock The frontier of simulation-based inference.
\newblock \emph{Proceedings of the National Academy of Sciences}, 117\penalty0
  (48):\penalty0 30055--30062, 2020.

\bibitem[Lavin et~al.(2022)Lavin, Krakauer, Zenil, Gottschlich, Mattson,
  Brehmer, Anandkumar, Choudry, Rocki, Baydin, Prunkl, Paige, Isayev, Peterson,
  McMahon, Macke, Cranmer, Zhang, Wainwright, Hanuka, Veloso, Assefa, Zheng,
  and Pfeffer]{lavin2021simulation}
A. Lavin, D. Krakauer, H. Zenil, J. Gottschlich, T. Mattson, J. Brehmer, A.
  Anandkumar, S. Choudry, K. Rocki, A.~G. Baydin, C. Prunkl, B. Paige, O.
  Isayev, E. Peterson, P.~L. McMahon, J. Macke, K. Cranmer, J. Zhang, H.
  Wainwright, A. Hanuka, M. Veloso, S. Assefa, S. Zheng, and A. Pfeffer.
\newblock Simulation intelligence: {T}owards a new generation of scientific
  methods, 2022.
\newblock URL \url{https://arxiv.org/abs/2112.03235}.
\newblock Preprint.

\bibitem[Baptista et~al.(2023)Baptista, Hosseini, Kovachki, and
  Marzouk]{baptista2020adaptive}
R. Baptista, B. Hosseini, N.~B. Kovachki, and Y. Marzouk.
\newblock Conditional sampling with monotone {GAN}s: {F}rom generative models
  to likelihood-free inference, 2023.
\newblock URL \url{https://arxiv.org/abs/2006.06755}.
\newblock Preprint.

\bibitem[Kim et~al.(2018)Kim, Wiseman, Miller, Sontag, and Rush]{kim2018semi}
Y. Kim, S. Wiseman, A. Miller, D. Sontag, and A. Rush.
\newblock Semi-amortized variational autoencoders.
\newblock In \emph{Proceedings of the 35th International Conference on Machine
  Learning}, volume~80, pages 2678--2687, 2018.

\bibitem[Kruse et~al.(2021)Kruse, Detommaso, Köthe, and
  Scheichl]{kruse2021hint}
J. Kruse, G. Detommaso, U. Köthe, and R. Scheichl.
\newblock {HINT}: {H}ierarchical invertible neural transport for density
  estimation and {B}ayesian inference.
\newblock \emph{Proceedings of the AAAI Conference on Artificial Intelligence},
  35\penalty0 (9):\penalty0 8191--8199, 2021.

\bibitem[Radev et~al.(2022)Radev, Mertens, Voss, Ardizzone, and
  Köthe]{radev2020bayesflow}
S.~T. Radev, U.~K. Mertens, A. Voss, L. Ardizzone, and U. Köthe.
\newblock {BayesFlow}: {L}earning complex stochastic models with invertible
  neural networks.
\newblock \emph{IEEE Transactions on Neural Networks and Learning Systems},
  33\penalty0 (4):\penalty0 1452--1466, 2022.

\bibitem[Siahkoohi and Herrmann(2021)]{siahkoohi2022wave}
A. Siahkoohi and F.~J. Herrmann.
\newblock Learning by example: {F}ast reliability-aware seismic imaging with
  normalizing flows.
\newblock In \emph{First International Meeting for Applied Geoscience \& Energy
  Expanded Abstracts}, pages 1580--1585, 2021.

\bibitem[Siahkoohi et~al.(2023)Siahkoohi, Rizzuti, Orozco, and
  Herrmann]{siahkoohi2023reliable}
A. Siahkoohi, G. Rizzuti, R. Orozco, and F.~J. Herrmann.
\newblock Reliable amortized variational inference with physics-based latent
  distribution correction.
\newblock \emph{Geophysics}, 88\penalty0 (3):\penalty0 R297--R322, 2023.

\bibitem[Prato(2006)]{da2006introduction}
G. Prato.
\newblock \emph{An Introduction to Infinite-Dimensional Analysis}.
\newblock Springer, 2006.

\bibitem[Mercer and Forsyth(1909)]{mercer1909xvi}
J. Mercer and A.~R. Forsyth.
\newblock {XVI}. {F}unctions of positive and negative type, and their
  connection the theory of integral equations.
\newblock \emph{Philosophical Transactions of the Royal Society of London.
  Series A, Containing Papers of a Mathematical or Physical Character},
  209\penalty0 (441-458):\penalty0 415--446, 1909.

\bibitem[Li et~al.(2021)Li, Kovachki, Azizzadenesheli, liu, Bhattacharya,
  Stuart, and Anandkumar]{li2021fourier}
Z. Li, N.~B. Kovachki, K. Azizzadenesheli, B. liu, K. Bhattacharya, A. Stuart,
  and A. Anandkumar.
\newblock Fourier neural operator for parametric partial differential
  equations.
\newblock In \emph{International Conference on Learning Representations}, 2021.

\bibitem[Orozco et~al.(2023)Orozco, Siahkoohi, Rizzuti, van Leeuwen, and
  Herrmann]{OrozcoSiahkoohiEtAl_2023}
R. Orozco, A. Siahkoohi, G. Rizzuti, T. van Leeuwen, and F.~J. Herrmann.
\newblock {Adjoint operators enable fast and amortized machine learning based
  {B}ayesian uncertainty quantification}.
\newblock In \emph{Medical Imaging 2023: Image Processing}, volume 12464, page
  124641L, 2023.

\bibitem[Anderson and Moore(1979)]{AndersonMoore1979}
B.~D. Anderson and J.~B. Moore.
\newblock \emph{Optimal Filtering}.
\newblock Prentice-Hall, Englewood Cliffs, NJ, 1979.

\bibitem[Kingma and Ba(2017)]{kingma2014adam}
D.~P. Kingma and J. Ba.
\newblock Adam: {A} method for stochastic optimization, 2017.
\newblock URL \url{https://arxiv.org/abs/1412.6980}.
\newblock Preprint.

\bibitem[Lambar{\'e} et~al.(1992)Lambar{\'e}, Virieux, Madariaga, and
  Jin]{lambare1992iterative}
G. Lambar{\'e}, J. Virieux, R. Madariaga, and S. Jin.
\newblock Iterative asymptotic inversion in the acoustic approximation.
\newblock \emph{Geophysics}, 57\penalty0 (9):\penalty0 1138--1154, 1992.

\bibitem[Nemeth et~al.(1999)Nemeth, Wu, and Schuster]{nemeth1999least}
T. Nemeth, C. Wu, and G.~T. Schuster.
\newblock Least-squares migration of incomplete reflection data.
\newblock \emph{Geophysics}, 64\penalty0 (1):\penalty0 208--221, 1999.

\bibitem[Veritas(2005)]{Veritas2005}
Veritas.
\newblock Parihaka 3{D} marine seismic survey -- {A}cquisition and processing
  report.
\newblock Technical Report 3460, New Zealand Petroleum \& Minerals, Wellington,
  2005.

\bibitem[WesternGeco(2012)]{WesternGeco2012}
WesternGeco.
\newblock Parihaka 3{D} {PSTM} final processing report.
\newblock Technical Report 4582, New Zealand Petroleum \& Minerals, Wellington,
  2012.

\bibitem[Louboutin et~al.(2019)Louboutin, Lange, Luporini, Kukreja, Witte,
  Herrmann, Velesko, and Gorman]{louboutin2019devito}
M. Louboutin, M. Lange, F. Luporini, N. Kukreja, P.~A. Witte, F.~J. Herrmann,
  P. Velesko, and G.~J. Gorman.
\newblock Devito ({V}3.1.0): {A}n embedded domain-specific language for finite
  differences and geophysical exploration.
\newblock \emph{Geoscientific Model Development}, 12\penalty0 (3):\penalty0
  1165--1187, 2019.

\bibitem[Luporini et~al.(2020)Luporini, Louboutin, Lange, Kukreja, Witte,
  H\"{u}ckelheim, Yount, Kelly, Herrmann, and Gorman]{luporini2020architecture}
F. Luporini, M. Louboutin, M. Lange, N. Kukreja, P. Witte, J. H\"{u}ckelheim,
  C. Yount, P.~H.~J. Kelly, F.~J. Herrmann, and G.~J. Gorman.
\newblock Architecture and performance of {D}evito, a system for automated
  stencil computation.
\newblock \emph{ACM Transactions on Mathematical Software}, 46\penalty0 (1),
  2020.

\end{thebibliography}

\newpage

\appendix
\section{Probability measures on infinite-dimensional Hilbert spaces}
In this section, we briefly present some fundamental notions related to probability measures on infinite-dimensional spaces, specifically separable Hilbert spaces $(H, \langle \cdot, \cdot \rangle)$. There is abundant literature on the subject. For more details we refer to \citet{da2006introduction, kerrigan2022diffusion, pidstrigach2023infinitedimensional, stuart2010inverse} and references therein.

\subsection{Gaussian measures on Hilbert spaces}
\begin{definition}
Let $(\Omega, \mathcal{F}, \mathbb{P})$ be a probability space. A measurable function $X:\Omega \to H$ is called a Gaussian random element (GRE) if for any $h \in H$, the random variable $\langle h,X\rangle $ has a scalar Gaussian distribution.
\end{definition}

Every GRE $X$ has a mean element $m \in H$ defined by
\[
m = \int_\Omega X(\omega) d\mathbb{P}(\omega),
\]
and a linear covariance operator $C:H \to H$ defined by
\[
Ch = \int_\Omega \langle h, X(\omega)\rangle X(\omega) d\mathbb{P}(\omega) - \langle m,h\rangle m, \quad \forall h \in H.
\]
We denote $X\sim \mathcal{N}(m, C)$ for a GRE in $H$ with mean element $m$ and covariance operator $C$. It can be shown that the covariance operator of a GRE is trace class, positive-definite and symmetric. Conversely, for any trace class, positive-definite and symmetric linear operator $C:H\to H$ and every $m \in H$, there exists a GRE with $X \sim \mathcal{N}(m,C)$. This leads us to the following definition:

\begin{definition}
If $X$ is a GRE, the pushforward of $\mathbb{P}$ through $X$, denoted by $\mathbb{P}_X$, is called a Gaussian probability measure on $H$. We will write $\mathbb{P}_X = \mathcal{N}(m,C)$.
\end{definition}
Let $X\sim \mathcal{N}(m,C)$. We can make a few remarks:\\
1) For any $h \in H$, we have $\langle h,X\rangle \sim \mathcal{N}(\langle h,m\rangle, \langle Ch,h\rangle)$.\\
2) $C$ is compact. By Mercer theorem \cite{mercer1909xvi} there exists $(\lambda_j)$ and an orthonormal basis of eigenfunctions $(v_j)$ such that $\lambda_j \geq 0$ and $Cv_j=\lambda_j v_j \forall j$. We consider the infinite-dimensional case in which $\lambda_j>0$ $\forall j$.\\
3) Suppose $m=0$ (we call the Gaussian measure of $X$ centered). The expected square norm of $X$ is given by
\[
\mathbb{E}[\|X\|_H^2] = \mathbb{E}\left[\sum_{j=1}^\infty \langle v_j,X \rangle^2 \right] = \sum_{j=1}^\infty \langle C v_j,v_j \rangle = \sum_{j=1}^\infty \lambda_j = \text{Tr}(C),
\]
which is finite since $C$ is trace class.

\subsection{Absolutely continuous measures and the Feldman-Hajek theorem}

Here we introduce the notion of absolute continuity for measures.

\begin{definition}
Let $\mu$ and $\nu$ be two probability measures on $H$ equipped with its Borel $\sigma$-algebra $\mathcal{B}(H)$. Measure $\mu$ is absolutely continuous with respect to $\nu$ (we write $\mu \ll \nu$) if $\mu(\Sigma) =0$ for all $\Sigma \in \mathcal{B}(H)$ such that $\nu(\Sigma)=0$.
\end{definition}

\begin{definition}
If $\mu \ll \nu$ and $\nu \ll \mu$ then $\mu$ and $\nu$ are said to be equivalent and we write $\mu \sim \nu$. If $\mu$ and $\nu$ are concentrated on disjoint sets then they are called singular; in this case we write $\mu \perp \nu$.
\end{definition}

Another notion that will be used throughout the paper is the Radon-Nikodym derivative.

\begin{theorem}
Let $\mu$ and $\nu$ be two measures on $(H,\mathcal{B}(H))$ and $\nu$ be $\sigma$-finite. If $\mu \ll \nu $, then there exists a $\nu$-measurable function $f$ on $H$ such that
\[
\mu (A') = \int_{A'} f d\nu, \quad \forall A' \in \mathcal{B}(H).
\]
Furthermore, $f$ is unique $\nu$-a.e.
and is called the Radon-Nikodym derivative of $\mu$ with respect to $\nu$. It is denoted by $d\mu/d\nu$.
\end{theorem}

\begin{remark}
In the paper, we will sometimes refer to $f$ as the density of $\mu$ with respect to $\nu$.
\end{remark}

We are finally able to state the Feldman-Hajek theorem in its general form.

\begin{theorem}
The following statements hold.
\begin{enumerate}
    \item Gaussian measures $\mu = \mathcal{N}(m_1,C_1)$, $\nu = \mathcal{N}(m_2,C_2)$ are either singular or equivalent.
    \item They are equivalent if and only if the following conditions hold:
    \begin{itemize}
        \item[(i)] $\nu$ and $\mu$ have the same Cameron-Martin space $H_0=C_1^{1/2}(H)=C_2^{1/2} (H)$.
        \item[(ii)] $m_1-m_2 \in H_0$.
        \item[(iii)] The operator $(C_1^{-1/2} C_2^{1/2}) (C_1^{-1/2} C_2^{1/2})^* - I$ is a Hilbert-Schmidt operator on the closure $\overline{H_0}$.
    \end{itemize}
\item If $\mu$ and $\nu$ are equivalent and $C_1=C_2=C$, then $\nu$-a.s. the Radon-Nikodym derivative $d\mu/d\nu$ is given by
\[
\frac{d\mu}{d\nu}(h) =e^{\Psi(h)},
\]
where $\Psi(h) = \langle C^{-1/2}(m_1-m_2), C^{-1/2} (h-m_2)\rangle - \frac{1}{2}
\|C^{-1/2}(m_1 - m_2)\|_H^2 \forall h \in H$.
\end{enumerate}
\end{theorem}

\subsection{Bayes' theorem for inverse problems}
Let $H$ and $K$ be separable Hilbert spaces, equipped with the Borel $\sigma$-algebra, and $A: H \to K$ a measurable mapping. We want to solve the inverse problem of finding $X$ from $Y$, where
\[
Y = A (X) + B
\]
and $B \in K$ denotes the noise. We adopt a Bayesian approach to this problem. We let $(X,Y) \in H\times K$ be a random variable and compute $X|Y$. We first specify $(X,Y)$ as follows:\\
1) Prior: $X\sim \mu_0$ measure on $H$.\\
2) Noise: $B \sim \eta_0$ measure on $K$, with $B$ independent from $X$.

The random variable $Y|X$ is then distributed according to the measure $\eta_x$, the translate of $\eta_0$ by $A(X)$. We assume that $\eta_x \ll \eta_0$. Thus for some potential $\Psi: H\times K \to \mathbb{R}$,
\[
\frac{d \eta_x}{d\eta_0} (y) = e^{-\Psi(x,y)}.
\]
The potential $\Psi(\cdot, y)$ satisfying the above formula is often termed the negative log likelihood of the problem.
Now define $\nu_0$ to be the product measure $\nu_0 = \mu_0 \times \eta_0$.
We can finally state the following infinite-dimensional analogue of the Bayes' theorem.

\begin{theorem}
Assume that $\Psi:H \times K \to \mathbb{R}$ is $\nu_0$-measurable and define
\[
Z(y) = \int e^{-\Psi(x,y)} d\mu_0.
\]
Then
\[
\frac{d\mu_0 (\cdot|Y=y)}{d\mu_0}(x) = \frac{1}{Z(y)}e^{-\Psi(x,y)},
\]
where $\mu_0 (\cdot|Y=y)$ is the conditional distribution of $X$ given $Y=y$.
\end{theorem}
\section{Proofs of Section 4}

\subsection{Proofs of Lemma \ref{lem:gauss1} and Proposition \ref{prop:Sbound}}\label{l2p1}

We assume that   $C_\mu$ in (\ref{eq:mumu0})  and
$C$ in
(\ref{eq:forward-sde})
have the same basis of eigenfunctions $(v_j)$ and that
$C_\mu v_j=\mu_j v_j, C v_j=\lambda_j v_j$ $\forall j$.
We define $X^{(j)}_t =\langle X_t,v_j \rangle$,
$y^{(j)} =\langle y,v_j \rangle$
and
$S^{(j)}(t,x,y) = \langle S(t,x,y),v_j \rangle$ so that
in (\ref{prop:Sbound}) $S(t,x,y) = \sum_j S^{(j)}(t,x,y) v_j$.
We assume $j \in {\cal I}^{(n)}$ so that we consider a mode
corresponding to an observation.
We then have
\ban
dX^{(j)}_t = -\frac{1}{2}  X^{(j)}_t dt + \sqrt{\lambda^{(j)}} dW^{(j)}_t  ,
\ean
with $W^{(j)}$ standard Brownian motions which are  independent
for the different modes $j$.  Note also that with $C_\mu$ and $C$ having the
same basis of eigenfunctions the system of modes is diagonalized so that
the $X^{(j)}_t$ processes are independent with respect to mode $j$,
both for the observed and un-observed modes.
Thus we have
\ban
X_0^{(j)} = \sqrt{\mu_j} \eta^{(j)}_0,  \quad Y^{(j)}=X_0^{(j)}+ \sigma_B\eta^{(j)}_1,
\quad X^{(j)}_t = X_0^{(j)}e^{-t/2}+\sqrt{\lambda_j (1-e^{-t})} \eta_2^{(j)},
\ean
for $\eta_i^{(j)}$ independent standard Gaussian random variables.
We then seek
\[
x_0^{(j,y)}=\EE[X^{(j)}_0\mid X_t=x,Y=y]=\EE[X^{(j)}_0
\mid X^{(j)}_t=x^{(j)}, Y^{(j)}=y^{(j)}]  ,
\]
which in this Gaussian setting is the $L^2$ projection
of $X_0^{(j)}$ onto $X^{(j)}_t$ and $Y^{(j)}$. Thus we can write
$x_0^{(j,y)}=ax^{(j)}+by^{(j)}$
with  $(a,b)$ solving
\ban
\EE[(aX^{(j)}_t + bY^{(j)}-X^{(j)}_0)Y^{(j)}]=0, \quad \EE[(aX^{(j)}_t
+ bY^{(j)}-X^{(j)}_0)X^{(j)}_t]=0 ,
\ean
which gives
\ban
a=\frac{e^{t/2}}{1+(e^t-1)p^{(j)}(1+q^{(j)})}, \qquad
b=\frac{p^{(j)}q^{(j)} (e^{t}-1)}{1+(e^t-1)p^{(j)}(1+q^{(j)})} ,
\ean
for $p^{(j)}=\lambda_j/\mu_j, q^{(j)}=\mu_j/\sigma_B^2$.

We then get in view of (\ref{eq:conditional-score-infinite})
\ba\label{eq:Sder}
S^{(j)}(t,x,y) =
  - \left( \frac{e^{t} p^{(j)}  (1+q^{(j)} )}{1+ (e^{t}-1)p^{(j)} (1+q^{(j)} )}
  \right) x^{(j)}  +
  \left(\frac{e^{t/2} p^{(j)} q^{(j)} }{1+ (e^{t}-1)p^{(j)} (1+q^{(j)} )} \right)
  y^{(j)} .
\ea
Note that with some abuse of notation we then have
\ban
S^{(j)}(t,x,y) = S^{(j)}(t,x^{(j)},y^{(j)}),
\ean
which is important since then also the time reversed system diagonalizes.
We remark that for an unobserved mode we get by a similar, but easier,
calculation
\ban
S^{(j)}(t,x,y) =  - \left( \frac{e^{t} p^{(j)}  }{1+ (e^{t}-1)p^{(j)}  }
  \right) x^{(j)},
\ean
which simply corresponds to setting $\sigma_B=\infty$ in (\ref{eq:Sder}).

Consider next $\EE[S^{(j)}(t,x,y)^2]$.  Note first that
\ban
\EE[X^{(j)}_t\mid Y=y]            &=& \left(\frac{q^{(j)} e^{-t/2}}{1+q^{(j)}}\right)  y^{(j)},   \\
  {\rm Var}[X^{(j)}_t \mid Y=y]   &=&  e^{-t}  {\rm Var}[X^{(j)}_t \mid Y=y] +
\lambda_j(1-e^{-t}) \\ &=&  (1+ (e^{t}-1)p^{(j)} (1+q^{(j)} )) \left( \frac{\mu_j e^{-t}}{1+q^{(j)}} \right) .
\ean
We can then easily check that the score is conditionally centered
$\EE[S^{(j)}(t,X_t,Y)\mid Y=y]=0$ and we then get
\begin{align*}
\EE[(S^{(j)}(t,X_t,Y))^2\mid Y=y] &= \left( \frac{e^{t} p^{(j)}  (1+q^{(j)} )}{1+ (e^{t}-1)p^{(j)} (1+q^{(j)} )}
  \right)^2
  {\rm Var}[X^{(j)}_t \mid Y=y] \\ & =
   \frac{e^{t} \mu_j (p^{(j)})^2(1+q^{(j)} )}{1+ (e^{t}-1)p^{(j)} (1+q^{(j)} )}
    ,
\end{align*}
which gives  Proposition \ref{prop:Sbound}
upon
summing over the mode index $j$, where we define  $q^{(j)}=0$ for
the unobserved modes.

\section{Proofs of Section 5}

\subsection{Discussion about an alternative approach}
\label{app:mill}
The following lemma is a complementary result related to Remark 1. It shows that we can actually derive the expression of the score from the results contained in \citet{millet1989time}.
The result is powerful,  but requires the verification of technical conditions.

\begin{lemma}
     Under the conditions stated in Proposition  \ref{lemma3} the score is defined by
	\ba
	S(x,y,t) =
	-(1-e^{-t})^{-1} \left(  x  - e^{-t/2}\EE\left[ X_0 | X_t=x, Y=y\right]\right)
	\ea
 and the time reversed diffusion takes the form in
 (\ref{eq:conditional-reverse-SDE}).
\end{lemma}
\begin{proof}
	Define
	\ban
	X^{(j)} = \langle v_j,X \rangle  \hbox{~for~} C v_j =\lambda_j v_j .
	\ean
	Then
	\ba\label{eq:inf}
	d X^{(j)} = -\frac{1}{2} X^{(j)} dt +  \sqrt{\lambda_j} dW^{(j)} \hbox{~for~}
	W^{(j)}=\langle v_j, W \rangle,
	\ea
 and where we assume that $C$ is of trace class. This is then
 an infinite dimensional system of the type considered in
 \citet{millet1989time}.
 We proceed to verify some conditions stated in \citet{millet1989time}:
 (i)  the coefficients of the system (\ref{eq:inf}) satisfy
 standard  growth  and Lipschitz continuity conditions
   (assumption $(H1,H4)$ satisfied); (ii) the coefficients depend
 on finitely many coordinates (assumption $(H2)$ satisfied);
 the system is time independent and diagonal (assumption $(H5)$ satisfied).
Moreover define  $\check{x}^{(j)}=(x_1,\ldots,x_{j-1},x_{j+1},\ldots)$,
then the law of $X_t^{(j)}$ given $\check{X}_t^{(j)}$
has for $t>0$  density $p_t(x^{(j)} | \check{X}_t^{(j)}=\check{x}^{(j)}, Y=y)$
with respect to Lebesgue measure and so that for $t_0>0$ and  each $j$:
$\int_{t_0}^T\EE[|\partial_{x^{(j)}} \log(p_t(x^{(j)})|  Y=y) dt < \infty $.
 Then  it follows from Theorems 3.1 and 4.3 in \citet{millet1989time}
 that  the time reversed problem is associated
 with the well-posed martingale problem
 defined by the coefficients in (\ref{eq:conditional-reverse-SDE})  for
 the score being:
	\ban
	\langle v_j, S(x,y,t)\rangle  &=&
	\frac{ \lambda_j \frac{\partial}{\partial x^{j}}
		\left( p_t( x^{(j)} | \check{X}_t^{(j)}=\check{x}^{(j)},Y=y)
		\right) }{p_t(x^{(j)} | \check{X}_t^{(j)}=\check{x}^{(j)}, Y=y)} ,
	\ean
 with the convention that the right hand side  is null on the set
$\{ p_t(x^{(j)} | \check{X}_t^{(j)}=\check{x}^{(j)}, Y=y) =0 \}$.

 It then follows for $t> 0$
\begin{align*}
& \langle v_j, S(x,y,t)\rangle  \\ &=   \int_\mathbb{R}
	d\mu_0(x_0^{(j)}| \check{X}_t^{(j)}=\check{x}^{(j)}, Y=y)
	\frac{ \lambda_j \frac{\partial}{\partial x^{j}}
		\left( p_t( x^{(j)} | \check{X}_t^{(j)}=\check{x}^{(j)},X_0^{(j)}=x_0^{(j)}, Y=y)
		\right) }{p_t(x^{(j)} | \check{X}_t^{(j)}=\check{x}^{(j)}, Y=y)}  .
\end{align*}
We then get
\begin{align*}
\langle v_j, S(x,y,t)\rangle  =
	  -  \int_\mathbb{R} & d\mu_0(x_0^{(j)}| \check{X}_t^{(j)}=\check{x}^{(j)}, Y=y)
	\\ & \times \left(\frac{x^{j}-e^{-t/2}x^{(j)}_0}{1-e^{-t}} \right)
	\frac{
		\left( p_t( x^{(j)} | \check{X}_t^{(j)}=\check{x}^{(j)},X_0^{(j)}=x_0^{(j)}, Y=y)
		\right) }{p_t(x^{(j)} | \check{X}_t^{(j)}=\check{x}^{(j)}, Y=y)} \\
	& \hspace*{-0.45in} =  -  \int_\mathbb{R} d\mu_0(x_0^{(j)}| {X}_t= x, Y=y)
	\left(\frac{x^{j}-e^{-t/2}x^{(j)}_0}{1-e^{-t}} \right) .
\end{align*}

\end{proof}

\subsection{A preliminary lemma}
The following lemma is the equivalent of \cite[Lemma 3]{pidstrigach2023infinitedimensional}.
It is used in the forthcoming proof of Proposition 2.
\begin{lemma}\label{lemma:mart}
	In the finite-dimensional setting $x\in \mathbb{R}^D$, we have for any $0 \leq s \leq t \leq T$:
	$$
	\nabla \log p_{t,y}(x_t) =  e^{(t-s)/2} \EE_y\big[ \nabla \log p_{s,y}(X_s) |X_t=x_t \big]  ,
	$$
	where $\EE_y$ is the expectation with respect to the distribution of $X_0$ and $W$ given $Y=y$ and $p_{t,y}$ is the pdf of $X_t$ under this distribution.
\end{lemma}

\begin{proof}
	We can write
	$$
	p_{t,y} (x_t) =\int_{\mathbb{R}^D} p_{s,y}(x_s) p_{t|s,y}(x_t|x_s) dx_s  ,
	$$
	where $p_{t|s,y}(\cdot|x_s) $ is the pdf of $X_t$ given $Y=y$ and $X_s=x_s$.
	It is, in fact, equal to the pdf of $X_t$ given $X_s=x_s$, which is the pdf of the multivariate Gaussian distribution with mean $\exp(-(t-s)/2) x_s$ and covariance $(1-\exp(-(t-s)))C$.
	Therefore
	$$
	p_{t,y} (x_t) =\int_{\mathbb{R}^D} p_{s,y}(x_s) p_{t|s}(x_t|x_s) dx_s .
	$$
	We can then deduce that
 \begin{align*}
     \nabla p_{t,y} (x_t) & = \frac{1}{(2\pi)^{D/2} (1-e^{-(t-s)})^{1/2} \det(C)^{1/2}}
     \int_{\mathbb{R}^D} dx_s  p_{s,y}(x_s)  \\
     &\hspace*{-0.5in}  \times  \nabla_{x_t} \exp \Big( - \frac{1}{2(1-e^{-(t-s)})} ( x_t  -\exp(-(t-s)/2) x_s)^T C^{-1} (x_t  -\exp(-(t-s)/2) x_s)\Big)  \\
     &= - \frac{e^{(t-s)/2}}{(2\pi)^{D/2} (1-e^{-(t-s)})^{1/2} \det(C)^{1/2}}
     \int_{\mathbb{R}^D} dx_s  p_{s,y}(x_s)   \\
     &\hspace*{-0.5in}   \times \nabla_{x_s} \exp \Big( - \frac{1}{2(1-e^{-(t-s)})} ( x_t  -\exp(-(t-s)/2) x_s)^T C^{-1} (x_t  -\exp(-(t-s)/2) x_s)\Big)    \\
     &=  \frac{e^{(t-s)/2}}{(2\pi)^{D/2} (1-e^{-(t-s)})^{1/2} \det(C)^{1/2}}
     \int_{\mathbb{R}^D} dx_s \nabla_{x_s} \big(  p_{s,y}(x_s)  \big)\\
     &\hspace*{-0.5in}   \times \exp \Big( - \frac{1}{2(1-e^{-(t-s)})} ( x_t  -\exp(-(t-s)/2) x_s)^T C^{-1} (x_t  -\exp(-(t-s)/2) x_s)\Big)  \\
     &=   e^{(t-s)/2}
     \int_{\mathbb{R}^D} dx_s \nabla_{x_s} \big(  p_{s,y}(x_s)  \big)  p_{t|s}(x_t|x_s),
 \end{align*}
 which gives
	$$
	\nabla p_{t,y}(x_t) = e^{(t-s)/2} \int_{\mathbb{R}^D} p_{t|s}(x_t|x_s) p_{s,y} (x_s) \nabla \log p_{s,y}(x_s) dx_s  .
	$$
	Using again that $p_{t|s,y}(\cdot|x_s)=p_{t|s}(\cdot|x_s) $ and $p_{t|s,y}(x_t|x_s) = \frac{p_{(s,t),y}(x_s,x_t)}{p_{s,y}(x_s)}$, we get
	\begin{align*}
	\nabla  p_{t,y}(x_t) & = e^{(t-s)/2} \int_{\mathbb{R}^D} p_{t|s,y}(x_t|x_s) p_{s,y} (x_s) \nabla \log p_{s,y}(x_s) dx_s \\
	&=
	e^{(t-s)/2} \int_{\mathbb{R}^D} p_{(s,t),y}(x_s,x_t)  \nabla \log p_{s,y}(x_s) dx_s .
	\end{align*}
	Since $\nabla \log p_{t,y}(x_t) = \frac{\nabla p_{t,y}(x_t)}{p_{t,y}(x_t)}$ and $p_{s|t,y}(x_s|x_t) = \frac{p_{(s,t),y}(x_s,x_t)}{p_{t,y}(x_t)}$ we get that
	\begin{align*}
	\nabla \log p_{t,y}(x_t) &= e^{(t-s)/2} \int_{\mathbb{R}^D} p_{s|t,y}(x_s|x_t)   \nabla \log p_{s,y}(x_s) dx_s
	\\
 &=
	e^{(t-s)/2} \EE_y \big[  \nabla \log p_{s,y}(X_s)|X_t=x_t].
	\end{align*}
\end{proof}

	\subsection{Proof of Proposition 2}
The proof adapts the one of \cite{pidstrigach2023infinitedimensional} to the conditional setting.
The only difference is that the expectation is $\EE_y$, which affects the distribution of $X_0$ but not the one of $W$. Moreover, Lemma \ref{lemma:mart} shows that the key to the proof (the reverse-time martingale property of the finite-dimensional score) is still valid.
Here $\EE_y$ is the expectation with respect to the distribution of $X_0$ and $W$ given $Y=y$.

To prove Proposition 2, we are left to show that the solution of the reverse-time SDE
	\begin{equation}
	dZ_t = \frac{1}{2} Z_t dt + S(T-t,Z_t,y)dt + \sqrt{C}W_t, \quad Z_0 \sim X_T|Y=y
 \label{eq:appendix-reverse-sde}
	\end{equation}
	satisfies $Z_T\sim X_0|Y=y$. We recall that $X_t$ is the solution to the SDE
	\[
	dX_t = - \frac{1}{2} X_t dt + \sqrt{C}dW_t, \quad X_0 \sim \mu_0.
	\]
	We first notice that $X_t$ is given by the following stochastic convolution:
	\[
	X_t = e^{-t/2}X_0 + \int_0^t e^{-(t-s)/2} \sqrt{C}dW_s.
	\]
	For $P^{D}$ the orthogonal projection on the subspace of $H$ spanned by $v_1, \ldots, v_D$ (the eigenfunctions of $C$), $X_t^{D}=P^{D}(X_t)$ are solutions to
	\[
	dX_t^{D} = -\frac{1}{2}X_t^D dt + \sqrt{(C^{D})} dW_t^{D},
	\]
	where
	\[
	C^D = P^DCP^D, \quad W_t^D = P^D W_t.
	\]
	We define $X_t^{D:M}=X_t^M - X_t^D$.
	Then
	\[
	X_t^{D:M} = e^{-t/2}X_0^{D:M} + \int_0^t e^{-(t-s)/2} \sqrt{(C^{D:M})} dW_s^{D:M},
	\]
	where the superscript ${D:M}$ indicates the projection onto span$\{v_{D+1},\ldots, v_M\}$. It holds that
	\[
	\mathbb{E}_y \big[\sup_{t \leq T} \|X_t^{D:M}\|^2_H \big] \leq 2 e^{-t} \mathbb{E}_y [\|X_0^{D:\infty}\|_H^2 ] +  2(1-e^{-t})\sum_{i=D+1}^\infty \lambda_i \to 0
	\]
	as $D \to \infty$, where we used Doob's $L^2$ inequality to bound the stochastic integral. Therefore $(X_t^N)$ is a Cauchy sequence and converges to $X_t$ in $L^2(\mathbb{P}_y)$.
	Consequently, the distribution of $X_t^N$ given $Y=y$ converges to the distribution of $X_t$ given $Y=y$ as $N \to +\infty$.

Recall that
	\[
	S(t,X_t,y) = - (1-e^{-t})^{-1} \mathbb{E}_y[X_t - e^{-t/2}X_0 \mid X_t ],
	\]
and recall that
\[
C^D \nabla \log p_{t,y}^D (X_t^{1:D}) = - (1-e^{-t})^{-1} P^D \mathbb{E}_y[X_t - e^{-t/2}X_0\mid   X_t^{1:D} ].
\]
In particular, due to the tower property of the conditional expectations,
\[
C^D \nabla \log p_{t,y}^D (X_t^{1:D}) = \mathbb{E}_y[S(t,X_t,y) \mid  X_t^{1:D}].
\]
Since, by Assumption \ref{assumption1},
\[
\mathbb{E}_y[\|S(t,X_t,y)\|_H^2] < \infty,
\]
the quantities $\mathbb{E}_y [X_t - e^{-t/2} X_0 \mid  X_{t}^{1:D}]$ are bounded in $L^2(\mathbb{P}_y)$ and will converge to the limit, $\mathbb{E}_y[S(t,X_t,y) \mid X_t] = S(t,X_t,y)$, by the Martingale convergence theorem. We get rid of the projection $P^D$ by
\begin{align*}
& \quad \; (1-e^{-t}) \mathbb{E}_y [\| C^D \nabla \log p_{t,y}^D(X_t^{1:D}) - S(t,X_t,y)\|_H^2] \\
& = \mathbb{E}_y [\|P^D \mathbb{E}_y [X_t-e^{-t/2} X_0\mid X_t^{1:D}] - \mathbb{E}_y [X_t-e^{-t/2} X_0\mid X_t]\|_H^2 ] \\
& \leq \mathbb{E}_y [\| \mathbb{E}_y [X_t-e^{-t/2} X_0\mid X_t^{1:D}] - \mathbb{E}_y [X_t-e^{-t/2} X_0\mid X_t]\|_H^2 ] \\
& \quad + \mathbb{E}_y [\| (I-P^D)\mathbb{E}_y[X_t-e^{-t/2}X_0\mid X_t^{1:D}]\|_H^2 ] \\
& \leq \mathbb{E}_y [\| \mathbb{E} [X_t-e^{-t/2} X_0\mid X_t^{1:D}] - \mathbb{E}_y [X_t-e^{-t/2} X_0\mid X_t]\|_H^2 ] \\
& \quad + \mathbb{E}_y [\| (I-P^D)(X_t-e^{-t/2}X_0)\|_H^2] .
\end{align*}
The first term vanishes due to our previous discussion. The second term vanishes since
\begin{align*}
\mathbb{E}_y [\| (I-P^D)(X_t-e^{-t/2}X_0)\|_H^2] &=  \mathbb{E}_y [ \| (I-P^D) \int_0^t e^{-(t-s)/2} \sqrt{C}dW_s \|_H^2] \\
& = \mathbb{E} [ \| (I-P^D) \int_0^t e^{-(t-s)/2} \sqrt{C}dW_s \|_H^2] \\
& = (1-e^{-t}) \sum_{i=D+1}^\infty \lambda_j \to 0
\end{align*}
as $D \to \infty$.

We now make use of the fact that $\nabla \log p_{t,y}^D$ is a square-integrable Martingale in the reverse-time direction by Lemma \ref{lemma:mart}. We therefore get a sequence of continuous $L^2$-bounded Martingales converging to a stochastic process. Since the space of continuous $L^2$-bounded martingale is closed and pointwise convergence translates to uniform convergence, we get that $S$ is a $L^2$-bounded martingale, with the convergence of of $C^D \nabla \log p_{t,y}^D$ to $S$ being uniform in time.

We have that
\[
Z_t^D - Z_0^D - \frac{1}{2} \int_0^t Z_s ds - \int_0^t C^D \nabla \log p_{s,y}^D(Z_s) = \sqrt{C}W_t^D.
\]
Since all the terms on the left-hand side converge in $L^2$, uniformly in $t$, so does the right-hand side. Using again the closedness of the spaces of Martingales and Levy's characterization of Wiener process, we find that $\sqrt{C}W_t^D$ converges to $\sqrt{C}W_t$. Therefore
\[
Z_t= Z_0 + \frac{1}{2}\int_0^t Z_s ds + \int_0^t S(t,Z_t,y) + \sqrt{C}W_t^D.
\]
Therefore, $Z_t$ is indeed a solution to \eqref{eq:appendix-reverse-sde} and $Z_T \sim X_0|Y=y$. Using uniqueness of the solution we then conclude that this holds for any solution $Z_t$.

\subsection{Proof of (\ref{eq:bound2})}
\begin{align*}
	\mathbb{E} [ \| S(t,X_t,y) \|_H^2 |Y=y]
	&=(1-e^{-t})^{-2}   \mathbb{E} [ \|  \mathbb{E}[  X_t  - e^{-t/2} X_0 | Y=y,X_t ] \|_H^2 | Y=y] \\
	&=(1-e^{-t})^{-2}   \mathbb{E} [ \|  \mathbb{E}[  \int_0^t e^{-(t-s)/2 } \sqrt{C} dW_s | Y=y,X_t ] \|_H^2| Y=y ] \\
	&\leq (1-e^{-t})^{-2}   \mathbb{E} [ \mathbb{E}[ \| \int_0^t e^{-(t-s)/2 } \sqrt{C} dW_s  \|_H^2| Y=y,X_t ] | Y=y] \\
	&=(1-e^{-t})^{-2}   \mathbb{E} [  \| \int_0^t e^{-(t-s)/2 } \sqrt{C} dW_s  \|^2_H |Y=y] \\
	&=(1-e^{-t})^{-2}   \mathbb{E} [  \| \int_0^t e^{-(t-s)/2 } \sqrt{C} dW_s  \|^2_H ] \\
	& = (1-e^{-t})^{-1} {\rm Tr}(C) .
\end{align*}

\subsection{Proof of Proposition \ref{lemma3}}
 Note that with the assumptions in Proposition  \ref{lemma3} with $C_\mu$
 and $C$ having the same basis of eigenfunctions and the separability
 assumption on the Radon-Nikodym derivative for the  modes,
 the system for the modes again diagonalizes. However, in this case
 the (conditional) distribution for $X^{(j)}_0$ is non-Gaussian in general
 and the change of measure with respect to the Gaussian measure characterized by
 $\psi^{(j)}(x^{(j)},y)$.
 We  let the  superscript $g$ denote the Gaussian case with $\psi\equiv 1$, then we have:
	\begin{align*}
	S^{(j)}&(t,x,y)  \\&=  -(1-e^{-t})^{-1} \left( x^{(j)} -
 e^{-t/2} \EE[ X_0^{(j)}| X_t=x,Y=y] \right) \\
	& =    -(1-e^{-t})^{-1} \left( x^{(j)} - e^{-t/2} \int_{\mathbb R}  x_0
	   \left(  \frac{  \mu^{(j,g)}_{x_0,x^{(j)}_t\mid y}(x_0,x^{(j)}) \psi^{(j)}(x_0,y) }{\mu^{(j)}_{x^{(j)}_t\mid y}(x^{(j)}) }   \right)   dx_0
	 \right) \\
	 & =      (1-e^{-t})^{-1} \left(   e^{-t/2} \int_{\mathbb R}  x_0
	   \left(  \frac{  \mu_{x_0,x^{(j)}_t\mid y}^{(j,g)}(x_0+x^{(j)}e^{t/2},x^{(j)}) \psi^{(j)}(x_0+x^{(j)} e^{t/2},y) }{\mu^{(j)}_{x^{(j)}_t\mid y}(x^{(j)}) }   \right)   dx_0
	 \right) \\
	 & =    S^{(j,g)}(t,x,y) T^{(j)}(t,x^{(j)},y) + \tilde{R}^{(j)}(t,x^{(j)},y)  ,
 	\end{align*}
	for  $S^{(j,g)}$ the mode score in the Gaussian case given in
 (\ref{eq:Sder})
 and with
	\ban
	&& \hspace*{-0.7cm} T^{(j)}(t,x^{(j)},y)  = \left( \frac{\mu^{(j,g)}_{x^{(j)}_t\mid y}(x^{(j)}) \psi^{(j)}(x^{(j)} e^{t/2},y)}{\mu^{(j)}_{x^{(j)}_t\mid y}(x^{(j)})} \right)  , \\
&& 	\hspace*{-0.7cm}\tilde{R}^{(j)}(t,x^{(j)},y) = (1-e^{-t})^{-1} e^{-t/2} \\ &&
\hspace*{-0.7cm}\times  \int_{\mathbb R}  x_0
	   \left(  \frac{  \mu_{x_0,x^{(j)}_t\mid y}^{(j,g)}(x_0+x^{(j)}e^{t/2},x^{(j)})
	   \left( \psi^{(j)}(x_0+x^{(j)} e^{t/2},y) -
    \psi^{(j)}(x^{(j)} e^{t/2},y) \right)   }{\mu_{x_t^{(j)}
    \mid y}(x^{(j)}) }   \right)   dx_0
	 .
	\ean
	We have for $\phi_{\lambda}$ the centered Gaussian density at second moment $\lambda$
	\ban
	T^{(j)}(t,x^{(j)},y)  &=& \frac{ \int_{\mathbb R}    \phi_{\lambda^{(j)}_t}(x^{(j)}-ve^{-t/2})   \phi_{\mu^{(j)}_y}(v-x^{(j)}_y)   dv }
	{\int_{\mathbb R}  \phi_{\lambda^{(j)}_t}(x^{(j)}-ve^{-t/2})   \phi_{\mu^{(j)}_y}(v-x^{(j)}_y) \psi^{(j)}(v,y) / \psi^{(j)}(x^{(j)} e^{t/2},y) dv } ,
	\ean
	for $\lambda^{(j)}_t=\lambda_j (1-e^{-t}),  \mu^{(j)}_y = \mu_j/(1+q^{(j)})$
 and $x_y^{(j)}=y^{(j)} q^{(j)}/(1+q^{(j)})$ for $y^{(j)}= \langle y,v_j \rangle$.
 Here $x_y^{(j)},\mu^{(j)}_y$ are respectively the mean and variance of $X_0^{(j)}$ given $y$ and
 where we used the parameterization set forth  in Section \ref{l2p1}.
We then have
\ban
|T^{(j)}(t,x^{(j)},y)| \leq K^2,   \quad \lim_{t \downarrow 0} T^{(j)}(t,x^{(j)}y) = 1 .
\ean
	We moreover have
	\begin{align*}
	| \tilde{R}^{(j)}(t,x^{(j)},y) |  &\leq  e^{-t/2}  (1-e^{-t})^{-1}  L
	\left( \frac{   \int_{\mathbb R} x_0^2  \phi_{\lambda^{(j)}_t }(x^{(j)}-x_0 e^{-t/2})   \phi_{\mu^{(j)}_y}(x_0-x_y^{(j)}) dx_0 }
	{\int_{\mathbb R}  \phi_{\lambda_t^{(j)}}(x^{(j)}-v e^{-t/2})   \phi_{\mu^{(j)}_y}(v-x^{(j)}_y) \psi^{(j)}(v,y)     dv} \right)
	     \\ &\leq
	    e^{-t/2}  (1-e^{-t})^{-1}  L K
	\left( \frac{   \int_{\mathbb R} x_0^2  \phi_{\lambda_t^{(j)} }(x^{(j)}-x_0 e^{-t/2})
 \phi_{\mu^{(j)}_y}(x_0-x^{(j)}_y) dx_0 }
	{\int_{\mathbb R}  \phi_{\lambda_t^{(j)}}(x^{(j)}-v e^{-t/2})  \phi_{\mu^{(j)}_y}(v-x_y^{(j)})      dv} \right)
	     .
	\end{align*}
	We then find
\ban
\limsup_{t \downarrow 0} | \tilde{R}^{(j)}(t,x^{(j)},y) | \leq \lambda^{(j)} LK ,
 \ean
 and, moreover
 \ban
&& \hspace*{-.7cm}| \tilde{R}^{(j)}(t,x^{(j)},y) | \leq \\ &&  \hspace*{-.7cm}  \lambda_j LK e^{t/2} \left(
\frac{1}{1+(e^t-1)p^{(j)}(1+q^{(j)})} +
\lambda_j \frac{(x_y^{(j)}-x^{(j)} e^{t/2})^2}{(\mu^{(j)}_y)^2}\frac{(e^t-1)}{(1+(e^t-1)p^{(j)}(1+q^{(j)}))^2} \right) .
 \ean
Consider in the Gaussian case  in (\ref{eq:bound1g}) a mode so that
$p^{(j)}(1+q^{(j)})\uparrow \infty$ and $\lambda_j$ fixed,
then the contribution of this mode to the score norm
blows up in the small time limit.
The situation with $p^{(j)}(1+q^{(j)})\uparrow \infty$ would happen for instance in a limit of perfect mode observation
so that $\sigma_B \downarrow 0$ and thus $q^{(j)} \uparrow \infty$.
Indeed in the limit of small (conditional) target mode variabilty
relative to the diffusion noise parameter the score drift becomes
large for small time to drive the mode to the conditional target
distribution.
We here thus assume $p^{(j)}(1+q^{(j)})$  is uniformly bounded
with respect to mode ($j$ index), moreover, that $C$ is of trace class.
We then find that Assumption \ref{assumption1} is satisfied
with the following bound
\ban
  &&  \hspace*{-.5cm}  \sup_{t \in [0,T]} \EE \big[ \| S(t,X_t,y) \|_H^2 | Y=y \big]
  \\ && \hspace*{-.5cm}
    \leq
    2  \left(\sum_j \lambda_j    e^T \left( K^4 p^{(j)} (1+q^{(j)})
    + 2 \lambda_j e^T (LK)^2 \left( 1 +
     3 \left(p^{(j)} (1+q^{(j)})(e^{T}-1)\right)^4
    \right)
   \right) \right) .
\ean
 We remark that in the case that we do not have a uniform bound on the $p^{(j)}(1+q^{(j)})$'s
 it follows from (\ref{eq:bound2}) that the rate of divergence of the expected square norm of the score is at most
 $t^{-1}$ as $t \downarrow  0$ with $C$ of trace class.

\subsection{Proof of Proposition 4}
We start from (\ref{eq:carola2}):
\begin{align*}
&\mathbb{E}_{x_t,y \sim {\cal L}(X_t,Y)} \big[ \|
 S(t,x_t,y) -s_\theta(t,x_t,y)\|_H^2\big]
 =
\mathbb{E}_{x_t,y \sim {\cal L}(X_t,Y)} \big[ \|
 S(t,x_t,y)  \|_H^2\big] \\
&\quad +\mathbb{E}_{x_t,y \sim {\cal L}(X_t,Y)} \big[ \|s_\theta(t,x_t,y)\|_H^2 \big]
-2 \mathbb{E}_{x_t,y \sim {\cal L}(X_t,Y)} \big[ \left<
 S(t,x_t,y) ,s_\theta(t,x_t,y)\right>\big] .
\end{align*}
From Definition~\ref{def:score} we have
\begin{align*}
  &  \mathbb{E}_{x_t,y \sim {\cal L}(X_t,Y)} \big[ \left<
 S(t,x_t,y) ,s_\theta(t,x_t,y)\right>\big] \\
 &= -(1-e^{-t})^{-1}
 \mathbb{E}_{x_t,y \sim {\cal L}(X_t,Y)} \Big[ \left<
 x_t-e^{-t/2} \EE_{x_0 \sim {\cal L}(X_0|X_t=x_t,Y=y)} [x_0],s_\theta(t,x_t,y)\right>\Big]  \\
  &= -(1-e^{-t})^{-1}
 \mathbb{E}_{x_t,y \sim {\cal L}(X_t,Y)} \Big[  \EE_{x_0 \sim {\cal L}(X_0|X_t=x_t,Y=y)} \Big[
 \left<
 x_t-e^{-t/2} x_0 ,s_\theta(t,x_t,y)\right>\Big]\Big] \\
   &= -(1-e^{-t})^{-1}
 \mathbb{E}_{(x_0,x_t,y) \sim {\cal L}(X_0,X_t,Y)} \Big[
 \left<  x_t-e^{-t/2} x_0 ,s_\theta(t,x_t,y)\right>\Big]  .
\end{align*}
We obtain that
\begin{align*}
&\mathbb{E}_{x_t,y \sim {\cal L}(X_t,Y)} \big[ \|
 S(t,x_t,y) -s_\theta(t,x_t,y)\|_H^2\big] \\
& = B
 +
 \mathbb{E}_{(x_0,x_t,y) \sim {\cal L}(X_0,X_t,Y)} \big[ \|
- (1-e^{-t})^{-1} ( x_t -e^{-t/2} x_0)
 - s_\theta(t,x_t,y)\|_H^2\big] ,
\end{align*}
with
$$
B=\mathbb{E}_{x_t,y \sim {\cal L}(X_t,Y)} \big[ \|S(t,x_t,y) \|_H^2\big]
-
\mathbb{E}_{(x_0,x_t) \sim {\cal L}(X_0,X_t)} \big[\| (1-e^{-t})^{-1} (x_t-e^{-t/2} x_0)  \|_H^2\big]
$$
that does not depend on $\theta$.
Since ${\cal L}(X_t|X_0=x_0,Y=y)={\cal L}(X_t|X_0=x_0)$
we finally get that
\begin{align*}
&\mathbb{E}_{x_t,y \sim {\cal L}(X_t,Y)} \big[ \|
 S(t,x_t,y) -s_\theta(t,x_t,y)\|_H^2\big] \\
& = B
 +
 \mathbb{E}_{x_0,y \sim {\cal L}(X_0,Y), x_t \sim {\cal L}(X_t|X_0=x_0)} \big[ \|
- (1-e^{-t})^{-1} ( x_t -e^{-t/2} x_0)
 - s_\theta(t,x_t,y)\|_H^2\big] .
\end{align*}

\section{Numerical experiments: details and additional results}

In this section, we provide additional details regarding our numerical experiments. In both experiments, we parameterize the conditional score $s_\theta(t, x_t,y)$ using discretization-invariant Fourier neural operators \citep[FNOs;][]{li2021fourier}. This parameterization enables mapping input triplets $(t,x_t,y)$ to the score conditioned on $y$ at time $t$. Once trained---by minimizing the objective function in equation~\eqref{eq:objective} with respect to $\theta$---we use the FNO as an approximation to the conditional score to sample new realizations of the conditional distribution by simulating the reverse-time SDE in equation~\eqref{eq:conditional-reverse-SDE}.

\subsection{Stylized example}

In this example, the conditional distribution that we approximate is defined using the relation

\begin{equation}
    x_0 = a y^2 + \eps,
    \label{eq:stylized_example}
\end{equation}

where $\eps \sim \Gamma (1, 2)$ and $a \sim \mathcal{U}\{-1, 1\}$. Here, $\Gamma$ refers to the Gamma distribution, and $\mathcal{U}\{-1, 1\}$ denotes the uniform distribution over the set $\{-1, 1\}$. Having an explicit expression characterizing the conditional distribution allows us to easily evaluate the obtained result through our method---as opposed to needing to use a baseline method, e.g., Markov chain Monte Carlo.

\paragraph{Training data} The discretization invariance of our model enables us to use training data that live on varying discretization grids. We exploit this property and simulate training joint samples $(x_0, y)$ by evaluating the expression in equation \eqref{eq:stylized_example} over values of $y$ that are selected as nonuniform grids over the domain $[-3, 3]$ that contain $15$--$50$ grid points.

\paragraph{Architecture} In this example, the FNO comprises (i) a fully connected lifting layer that maps the three-dimensional vector---including the timestep $t$, the $y$ value, and the corresponding $x$ value---for each grid point to a $128$-dimensional lifted space; (ii) five Fourier neural layers, as introduced in \citep{li2021fourier}, which contain pointwise linear filters applied to the five lower Fourier modes; and (iii) two fully-connected layers, separated by a ReLU activation function, that map the $128$-dimensional lifted space back to the conditional score (a scalar) for each grid point.

\paragraph{Optimization details} To train the FNO, we minimized the objective function in equation~\eqref{eq:objective} using $2 \times 10^4$ training steps. At each step, we simulated a batch of $512$ training pairs selected from a grid with varying numbers of discretization points. We utilized the Adam stochastic optimization method \citep{kingma2014adam} with an initial learning rate of $10^{-3},$ which decayed to $5 \times 10^{-4}$ during optimization, following a power-law rate of $-1/3$. Regarding the diffusion process, we followed the approach outlined in \citet{ho2020denoising} and employed standard Gaussian noise with linearly increasing variance for the forward dynamics, which was discretized into $500$ timesteps. The initial Gaussian noise variance was set to $10^{-4}$ and linearly increased over the timesteps until it reached $2 \times 10^{-2}$. The training hyperparameters were chosen by monitoring the validation loss over $1024$ samples. The training process took approximately 7 minutes on a Tesla V100 GPU device. For further details, please refer to our open-source implementation on \href{https://github.com/alisiahkoohi/csgm}{GitHub}.

\paragraph{Additional results} Figure~\ref{fig:sup_toy_example} illustrates the predicted samples associated with the remaining testing grid sizes, whose densities were shown in Figures~\ref{fig:marginal_left}--\ref{fig:marginal_right}.

\begin{figure}[t]
    \centering
    \captionsetup[subfigure]{skip=-11pt}
    \begin{subfigure}[b]{0.329\textwidth}
        \includegraphics[width=\textwidth]{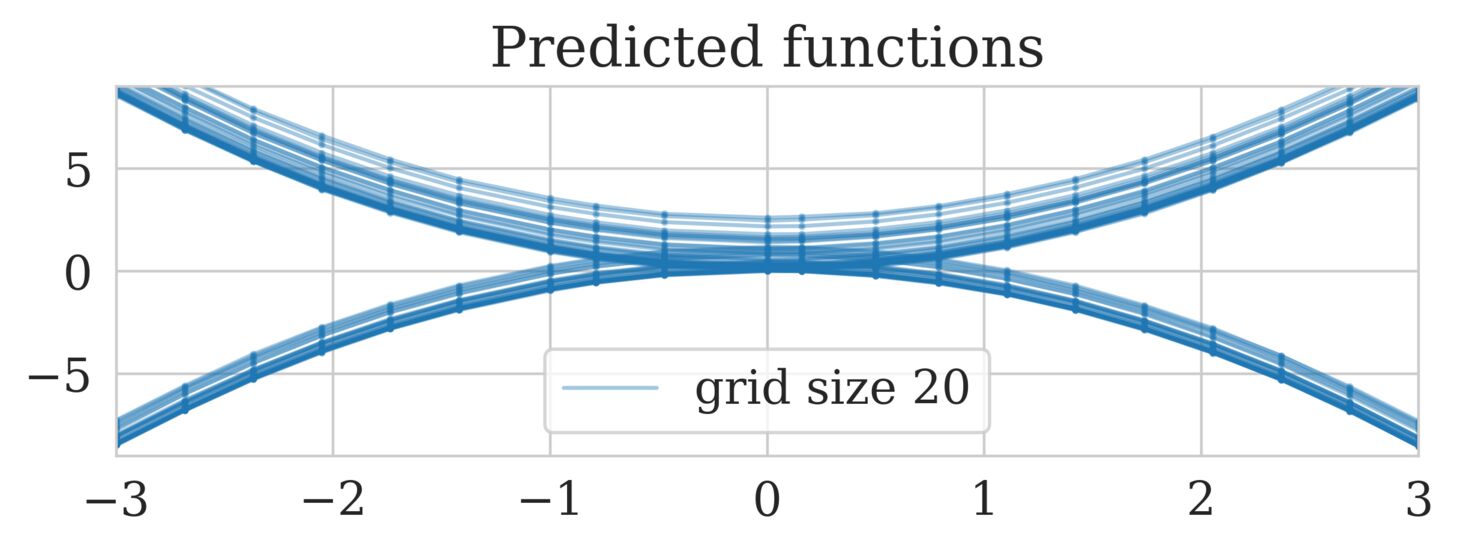}
    \vspace{0ex}\caption{}
    \label{fig:sup_grid_20}
    \end{subfigure}\hspace{0em}
    \begin{subfigure}[b]{0.329\textwidth}
        \includegraphics[width=\textwidth]{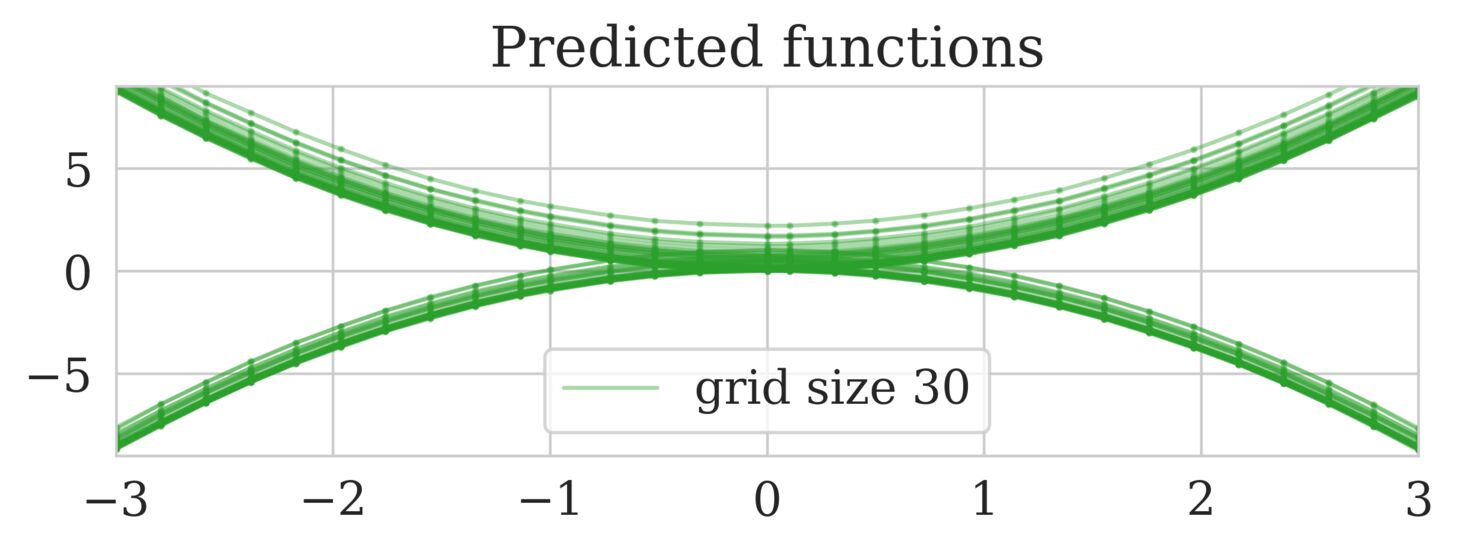}
    \vspace{0ex}\caption{}
    \label{fig:sup_grid_30}
    \end{subfigure}\hspace{0em}
    \begin{subfigure}[b]{0.329\textwidth}
        \includegraphics[width=\textwidth]{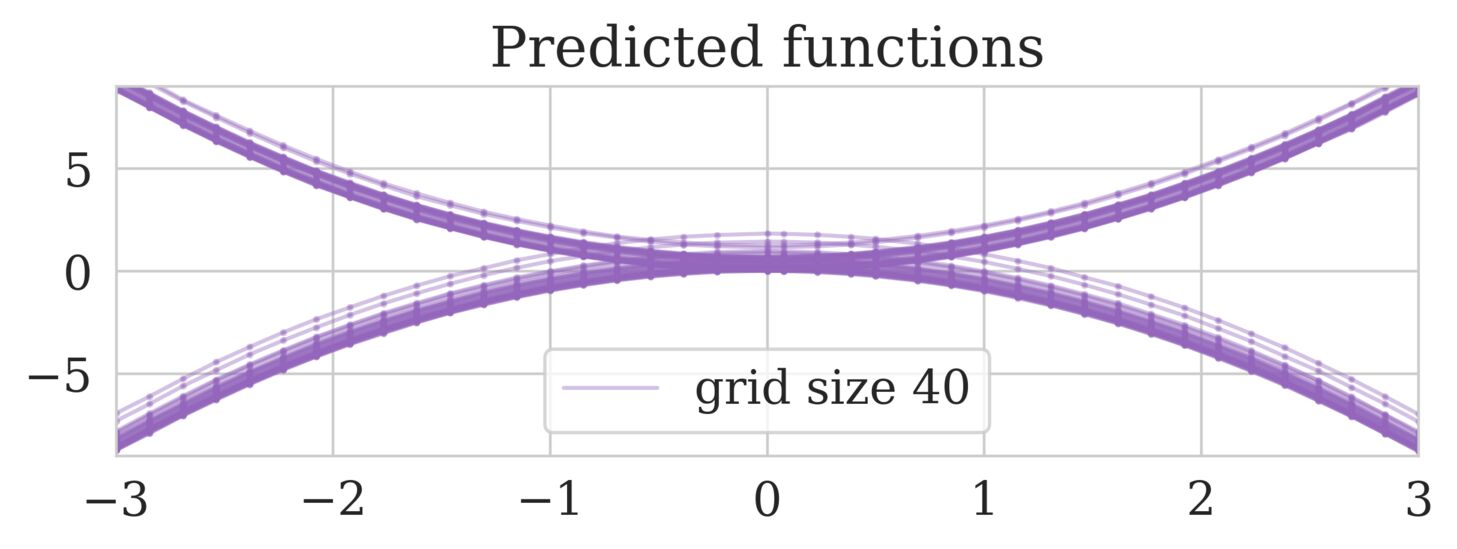}
    \vspace{0ex}\caption{}
    \label{fig:sup_grid_40}
    \end{subfigure}\hspace{0em}

    \caption{The proposed method's discretization invariance. Predicted samples (functions) on a uniformly sampled grid with (a) 20, (b) 30, and (c) 40 grid points.}
    \label{fig:sup_toy_example}
\end{figure}

\subsection{Linearized seismic imaging example}

The inverse problem we are addressing in this example involves the task of estimating the short-wavelength component of the Earth's unknown subsurface squared-slowness model using measurements taken at the surface. This particular problem, often referred to as seismic imaging, can be recast as a linear inverse problem when we linearize the nonlinear relationship between surface measurements and the squared-slowness model, as governed by the wave equation. In its simplest acoustic form, the linearization with respect to the slowness model---around a background smooth squared slowness model $m_0$---results in a linear inverse problem for the estimation of the true seismic image $\delta m^{\ast}$ using the following forward model,
\begin{equation}
    d_i = J(m_0, q_i)
        \delta m^{\ast} + \epsilon_i, \quad
        \epsilon_i \sim p(\epsilon),
    \label{linear-fwd-op}
    \end{equation}
 where $d=\left\{d_{i}\right\}_{i=1}^{n_s}$ represents a collection of $n_s$ linearized surface measurements, i.e., data wherein the zero-order term of Taylor's series has been subtracted, and $J(m_0, q_i)$ denotes the linearized Born scattering operator, which is defined in terms of the source signature $q_{i}$ and the background squared-slowness model $m_0$. Due to noise and linearization errors, the above expression contains the term $\epsilon_i$. In addition to this noise term, the forward operator has a non-trivial nullspace due to the presence of shadow zones and finite-aperture data \citep{lambare1992iterative, nemeth1999least}. To tackle this challenges, we set seismic imaging into a Bayesian framework and learn the associated posterior distribution via our proposed method.

\paragraph{Training data} We generated synthetic data by applying the Born scattering operator to $4750$ 2D seismic images, each with dimensions of $3075\, \mathrm{m} \times 5120\, \mathrm{m}$. These images were extracted from the Kirchhoff migrated \href{https://wiki.seg.org/wiki/Parihaka-3D}{Parihaka-3D} dataset, which contains seismic images obtained by imaging the data collected in New Zealand~\citep{Veritas2005, WesternGeco2012}. We parameterize the Born scattering operator using a background squared-slowness model (recall Figure~\ref{fig:background}). The data acquisition geometry involves $102$ sources with a spacing of $50\, \mathrm{m}$, each recorded for two seconds via $204$ receivers spaced at $25\, \mathrm{m}$ located on top the image. The source wavelet used was a Ricker wavelet with a central frequency of $30\, \mathrm{Hz}$. To replicate a more realistic imaging scenario, we add band-limited noise to the the data, obtained by filtering white noise with the source function. To create training pairs, we first simulated noisy seismic data for all the 2D seismic images based on the aforementioned acquisition design. Subsequently, we reduce the dimensionality of seismic data by applying the adjoint of the Born scattering operator to the data. We use \href{https://www.devitoproject.org/}{Devito} \citep{louboutin2019devito, luporini2020architecture} for the wave-equation based simulations.

\paragraph{Architecture} In this example, the FNO is composed of (i) a fully connected lifting layer that maps the five-dimensional vector---including the timestep $t$, the two spatial coordinates, the data value ($y$, obtained after dimensionality reduction by applying the adjoint of the forward operator), and the corresponding $x$ value---for each grid point to a $32$-dimensional lifted space; (ii) four Fourier neural layers, as introduced in \citep{li2021fourier}, which contain pointwise linear filters applied to the $24$ lower Fourier modes; and (iii) two fully-connected layers, separated by a ReLU activation function, that map the $128$-dimensional lifted space back to the conditional score for each grid point.

\paragraph{Optimization details} We train the FNO according to the objective function in equation~\eqref{eq:objective} with the Adam \citep{kingma2014adam} stochastic optimization method with batch size $128$ for $300$ epochs. We use an initial stepsize of $2 \times 10^{-3}$, decaying to $5 \times 10^{-4}$ during optimization with a power-law rate of $-1/3$. We use a similar diffusion process as the previous example. The training hyperparameters are chosen by monitoring the validation loss over $530$ samples. The training process takes approximately two hours and $15$ minutes on a Tesla V100 GPU device. For further details refer to our open-source implementation on \href{https://github.com/alisiahkoohi/csgm}{GitHub}.

\paragraph{Additional results} Figures~\ref{figs:sup_idx-15}--\ref{figs:sup_idx-19} illustrate more results regarding using the SDM to sample from the posterior distribution of multiple seismic imaging problem instances where the ground truth images (see Figures~\ref{fig:sup_idx-15_true_model}--\ref{fig:sup_idx-19_true_model}) are obtained from the test dataset. In each plot, we use the dimensionality-reduced data as the conditioning input to the FNO (see Figures~\ref{fig:sup_idx-15_observed_data}--\ref{fig:sup_idx-19_observed_data}). Through the SDM, we obtain $10^3$ posterior distribution samples and use them to estimate the conditional mean (see  Figures~\ref{fig:sup_idx-15_conditional_mean}--\ref{fig:sup_idx-19_conditional_mean}) and the pointwise standard deviation among samples (see Figures~\ref{fig:sup_idx-15_pointwise_std}--\ref{fig:sup_idx-19_pointwise_std}), with the former serving as a measure of uncertainty. In all cases, the regions of significant uncertainty correspond well with challenging-to-image sections of the model, which qualitatively confirms the accuracy of our Bayesian inference method. This observation becomes more apparent in Figures~\ref{fig:sup_idx-15_vertical_profile_at_10}--\ref{fig:sup_idx-19_vertical_profile_at_10}, displaying two vertical profiles with $99\%$ confidence intervals (depicted as orange-colored shading) for each experiment, which demonstrate the expected trend of increased uncertainty with depth. Furthermore, we notice that the ground truth (indicated by dashed black lines) mostly falls within the confidence intervals for most areas. We also consistently observe a strong correlation between the pointwise standard deviation and the error in the conditional mean estimate (see Figures~\ref{fig:sup_idx-15_error}--\ref{fig:sup_idx-19_error}), which further asserts the accuracy of our method in characterizing the posterior distribution in this large-scale Bayesian inference problem. To prevent bias from strong amplitudes in the estimated image, we present the normalized pointwise standard deviation divided by the envelope of the conditional mean in Figures~\ref{fig:sup_idx-15_normalized_pointwise_std}--\ref{fig:sup_idx-19_normalized_pointwise_std}. These visualizations provides an amplitude-independent assessment of uncertainty, highlighting regions of high uncertainty at the onset and offset of reflectors (both shallow and deeper sections). Additionally, the normalized pointwise standard deviation underscores uncertainty in areas of the image where there are discontinuities in the reflectors (indicated by black arrows), potentially indicating the presence of faults.

\begin{figure}[!htb]
    \centering
    \captionsetup[subfigure]{skip=-11pt}
    \begin{tabular}[t]{cc}
        \begin{tabular}{c}
        \begin{subfigure}[b]{0.31\textwidth}
            \includegraphics[width=\textwidth]{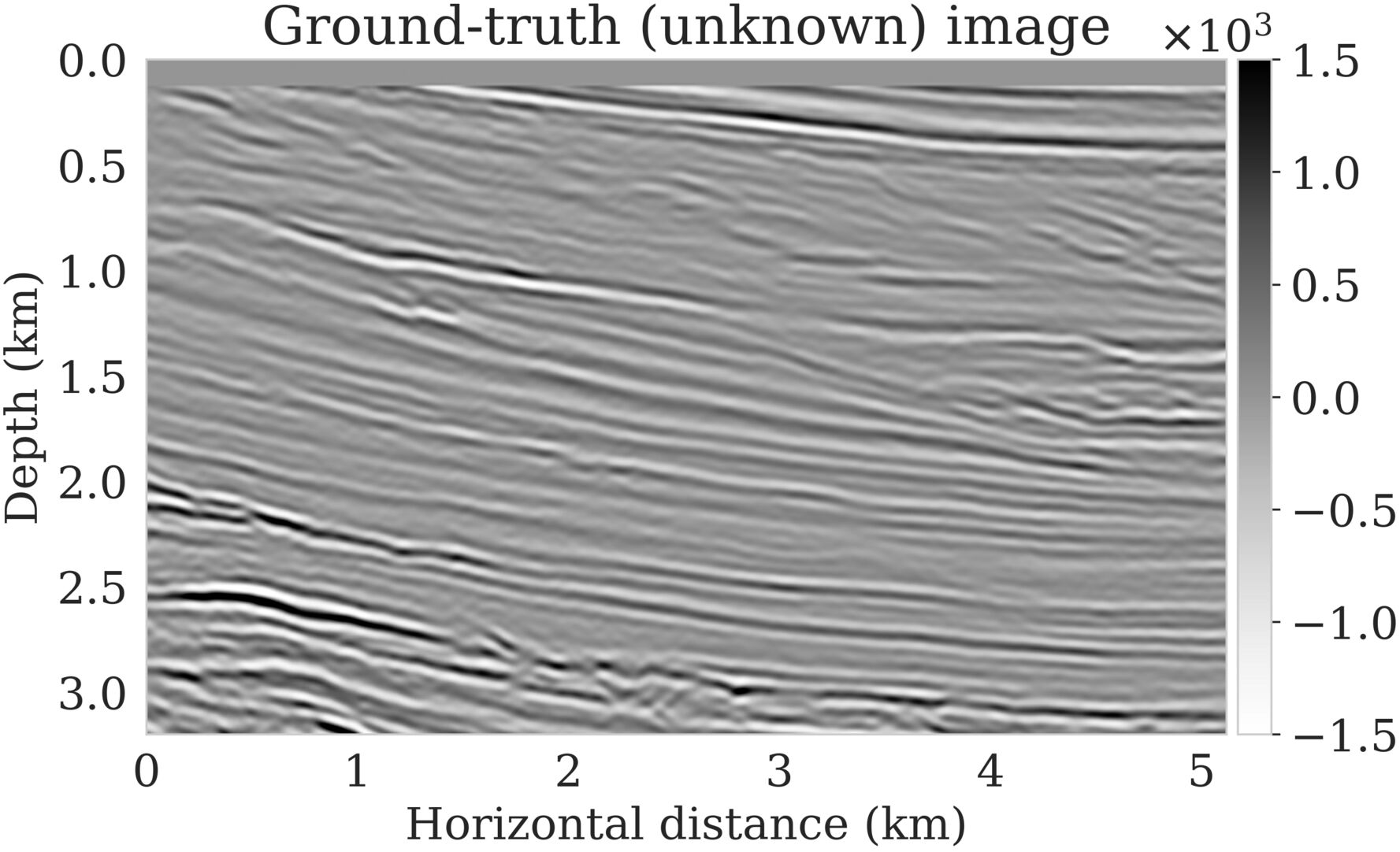}
        \vspace{0ex}\caption{}
        \label{fig:sup_idx-15_true_model}
        \end{subfigure}\hspace{0em}
        \begin{subfigure}[b]{0.31\textwidth}
            \includegraphics[width=\textwidth]{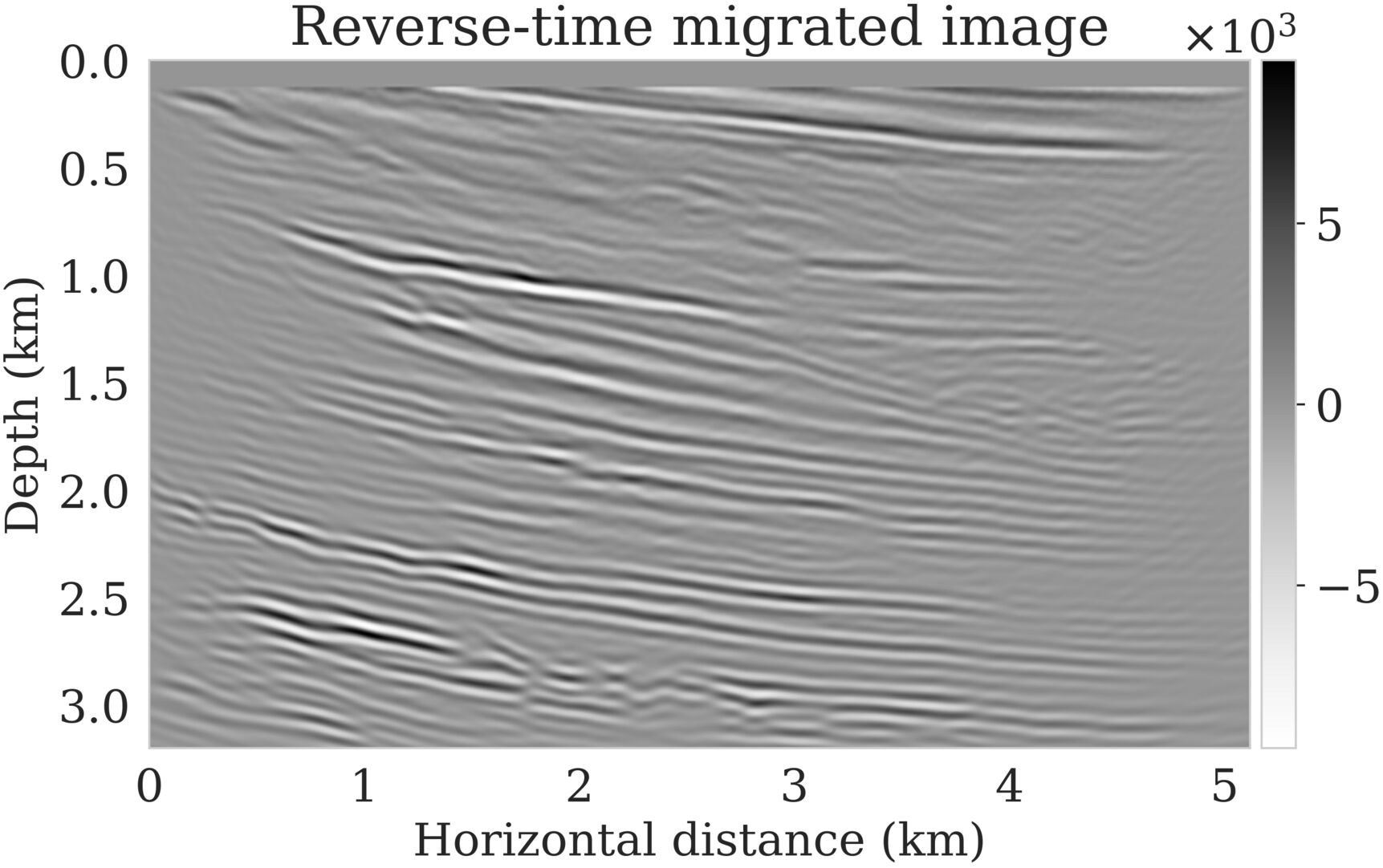}
        \vspace{0ex}\caption{}
        \label{fig:sup_idx-15_observed_data}
        \end{subfigure}\hspace{0em}
        \\
        \begin{subfigure}[b]{0.31\textwidth}
            \includegraphics[width=\textwidth]{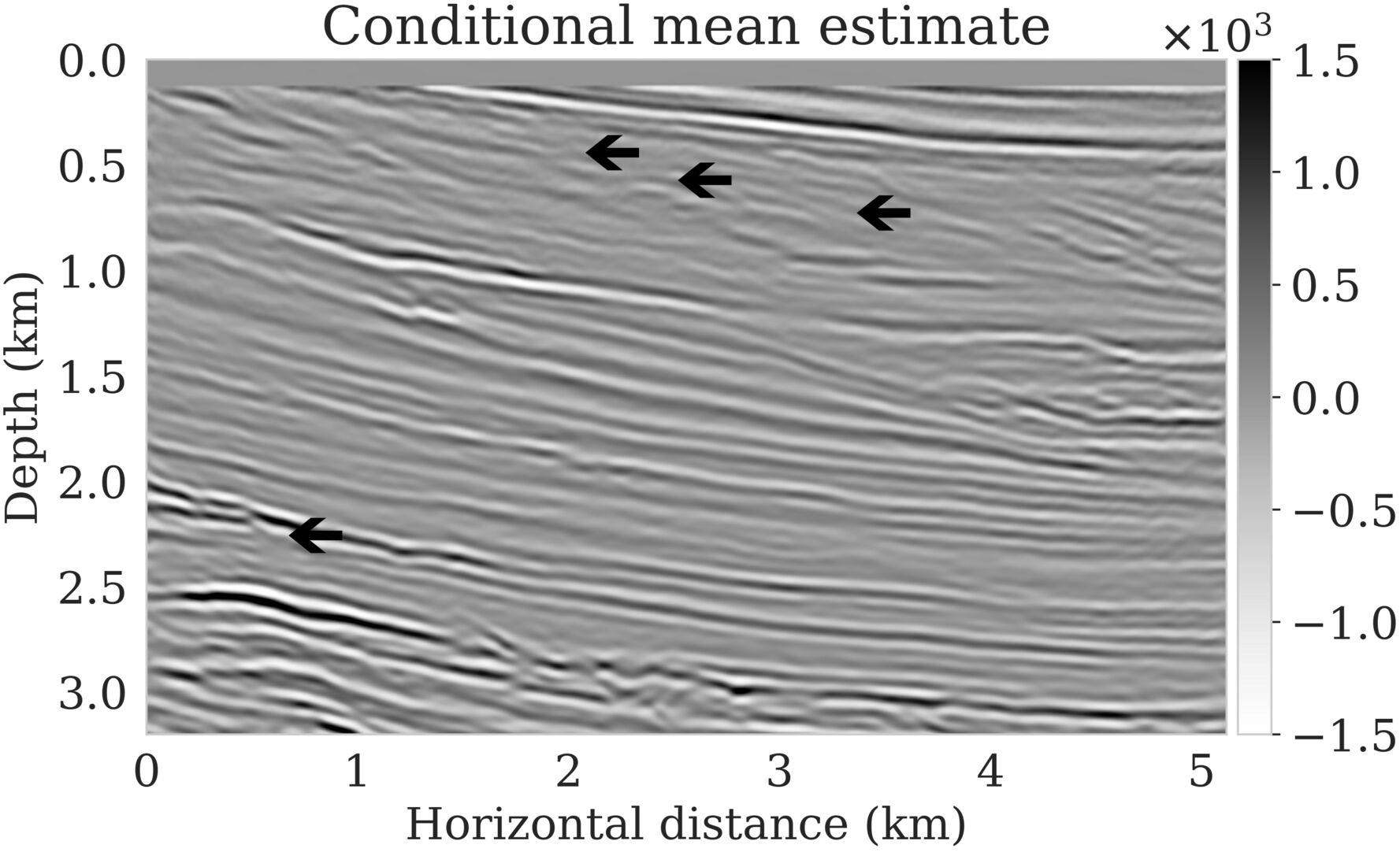}
        \vspace{0ex}\caption{}
        \label{fig:sup_idx-15_conditional_mean}
        \end{subfigure}\hspace{0em}
        \begin{subfigure}[b]{0.31\textwidth}
            \includegraphics[width=\textwidth]{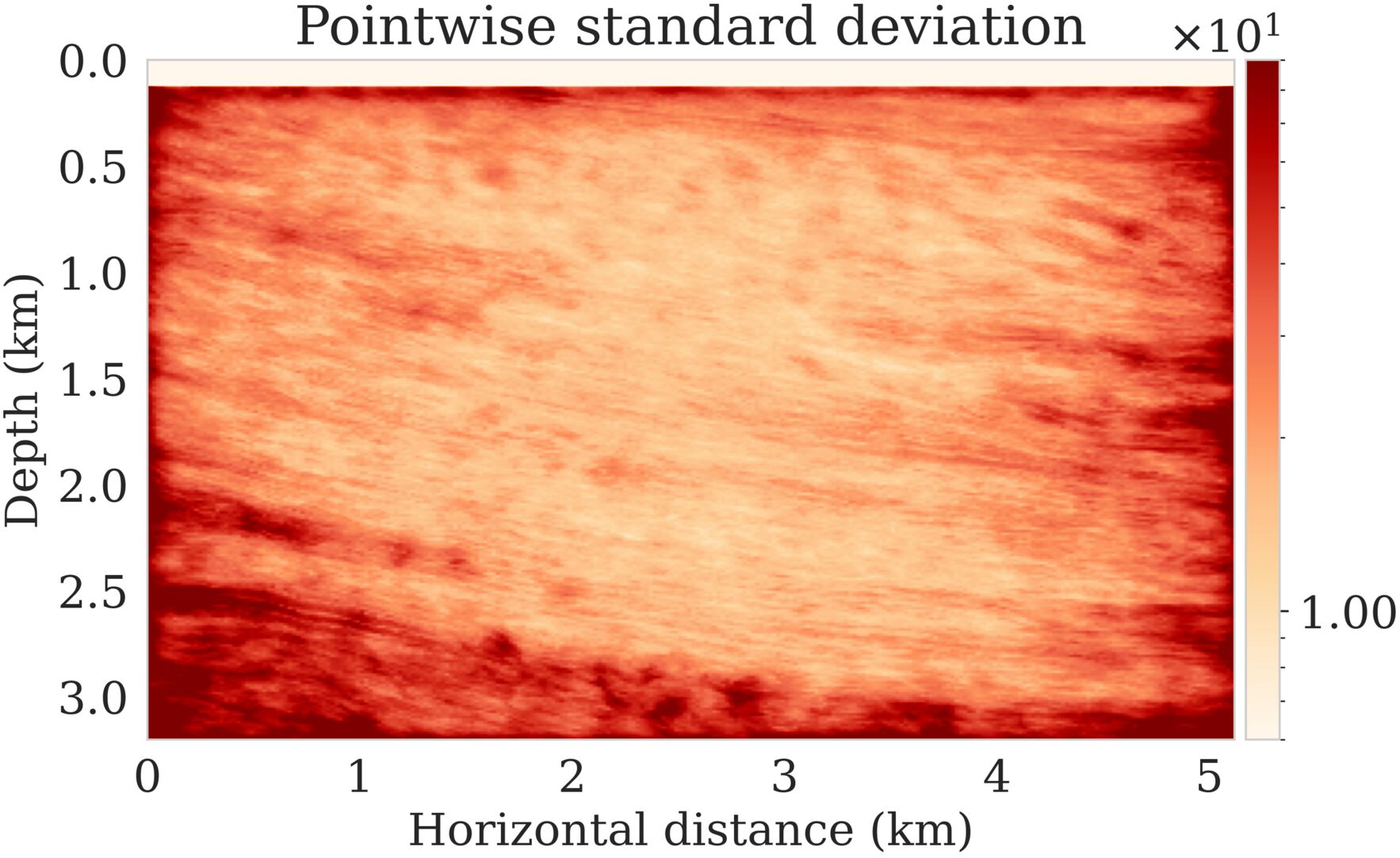}
        \vspace{0ex}\caption{}
        \label{fig:sup_idx-15_pointwise_std}
        \end{subfigure}\hspace{0em}
        \\
        \begin{subfigure}[b]{0.31\textwidth}
            \includegraphics[width=\textwidth]{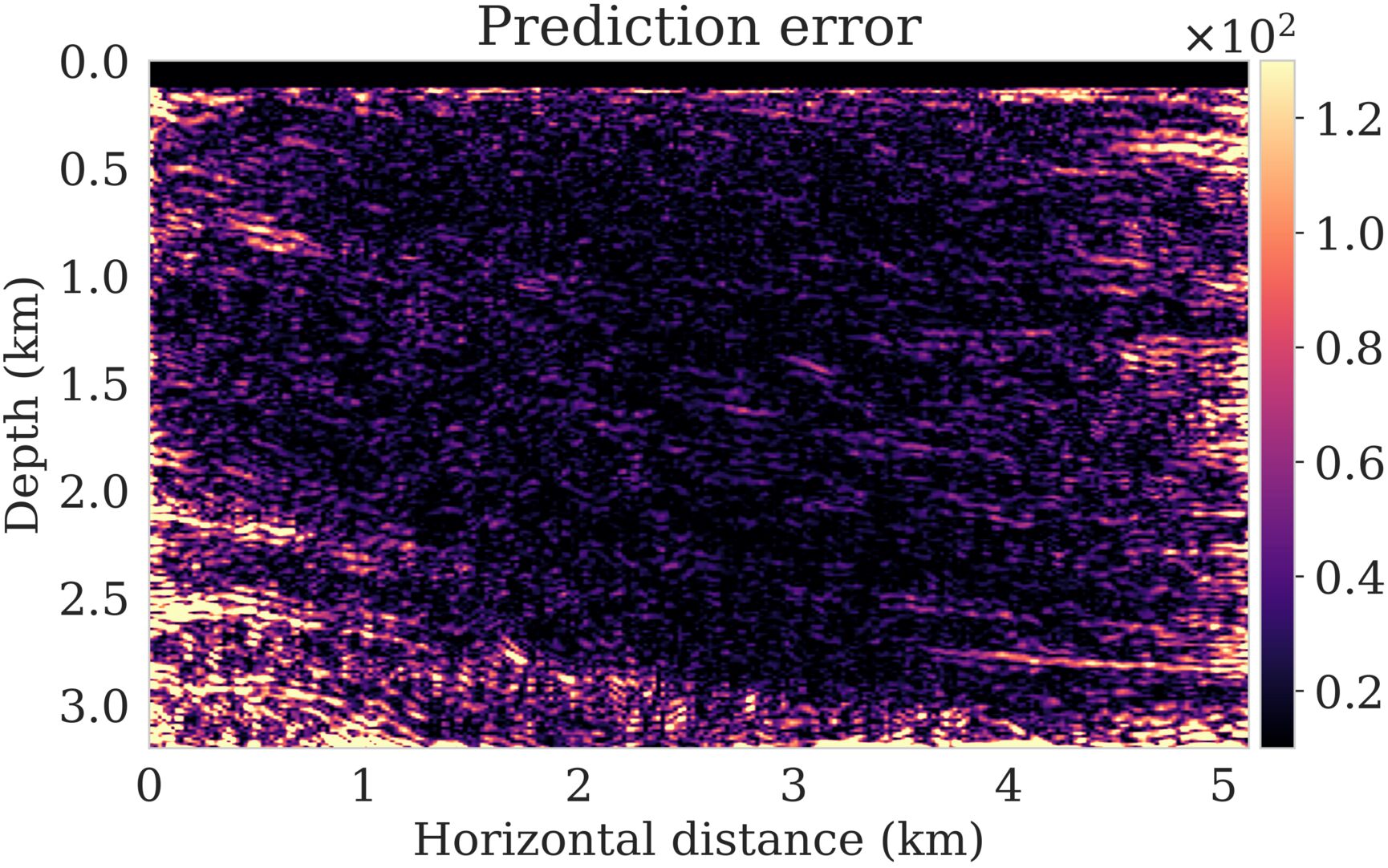}
        \vspace{0ex}\caption{}
        \label{fig:sup_idx-15_error}
        \end{subfigure}\hspace{0em}
        \begin{subfigure}[b]{0.31\textwidth}
            \includegraphics[width=\textwidth]{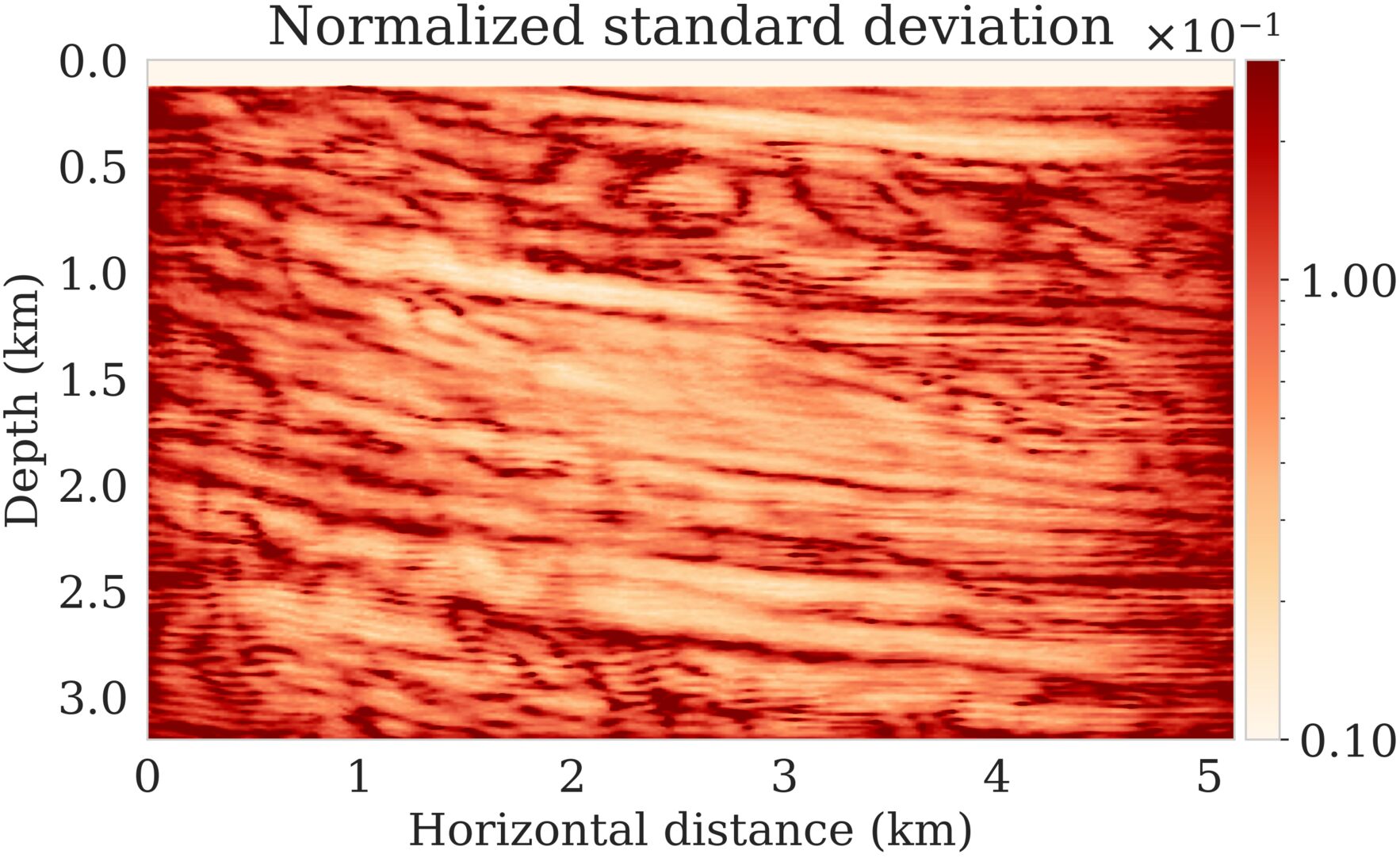}
        \vspace{0ex}\caption{}
        \label{fig:sup_idx-15_normalized_pointwise_std}
        \end{subfigure}\hspace{0em}
        \end{tabular}
    &
    \hspace{-2.5em}
    \begin{tabular}{c}
        \begin{subfigure}[b]{0.343\textwidth}
            \includegraphics[width=\textwidth]{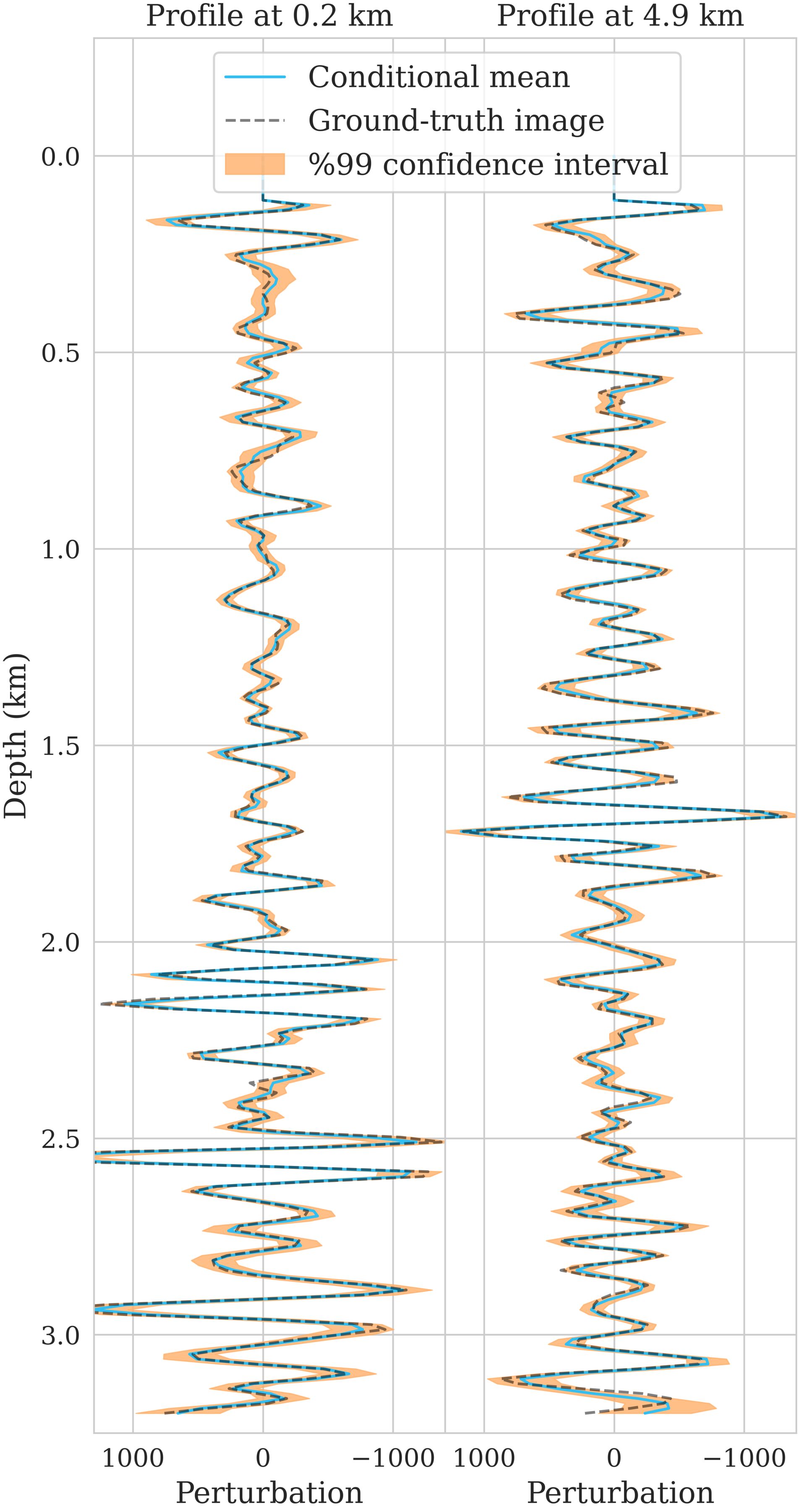}
        \vspace{0ex}\caption{}
        \label{fig:sup_idx-15_vertical_profile_at_10}
        \end{subfigure}\hspace{0em}
    \end{tabular}
    \end{tabular}
    \caption{Seismic imaging and uncertainty quantification. (a) Ground-truth seismic image. (b) Data after applying the adjoint Born operator (known as the reverse-time migrated image). (c) Conditional (posterior) mean. (d) Pointwise standard deviation. (e) Absolute error between Figures~\ref{fig:sup_idx-15_true_model} and~\ref{fig:sup_idx-15_conditional_mean}. (f) Normalized pointwise standard deviation by the envelope of the conditional mean. (g) Vertical profiles of the ground-truth image, conditional mean estimate, and the $99\%$ confidence interval at two lateral positions in the image.}
    \label{figs:sup_idx-15}
    \end{figure}

\begin{figure}[!htb]
    \centering
    \captionsetup[subfigure]{skip=-11pt}
    \begin{tabular}[t]{cc}
        \begin{tabular}{c}
        \begin{subfigure}[b]{0.31\textwidth}
            \includegraphics[width=\textwidth]{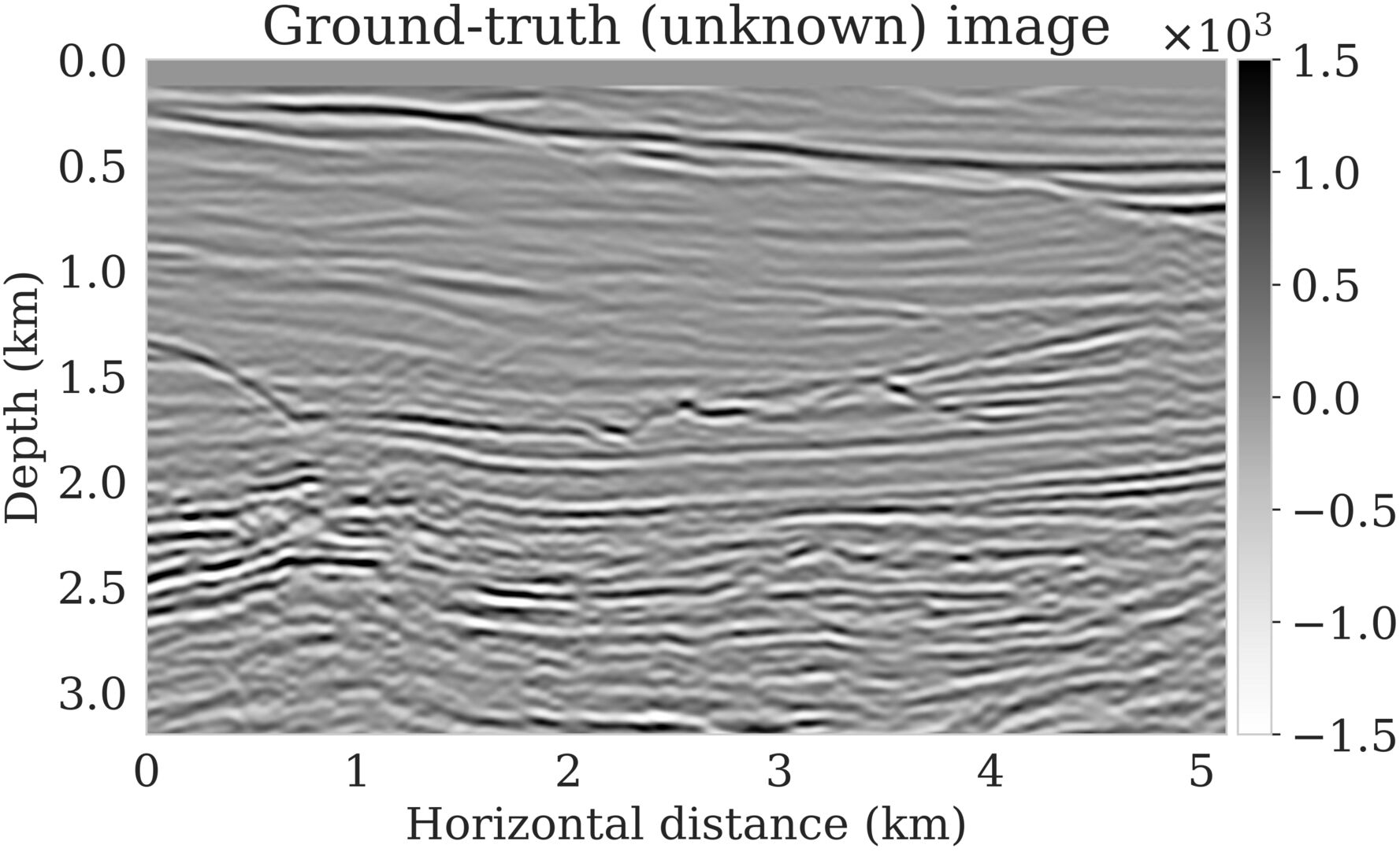}
        \vspace{0ex}\caption{}
        \label{fig:sup_idx-1_true_model}
        \end{subfigure}\hspace{0em}
        \begin{subfigure}[b]{0.31\textwidth}
            \includegraphics[width=\textwidth]{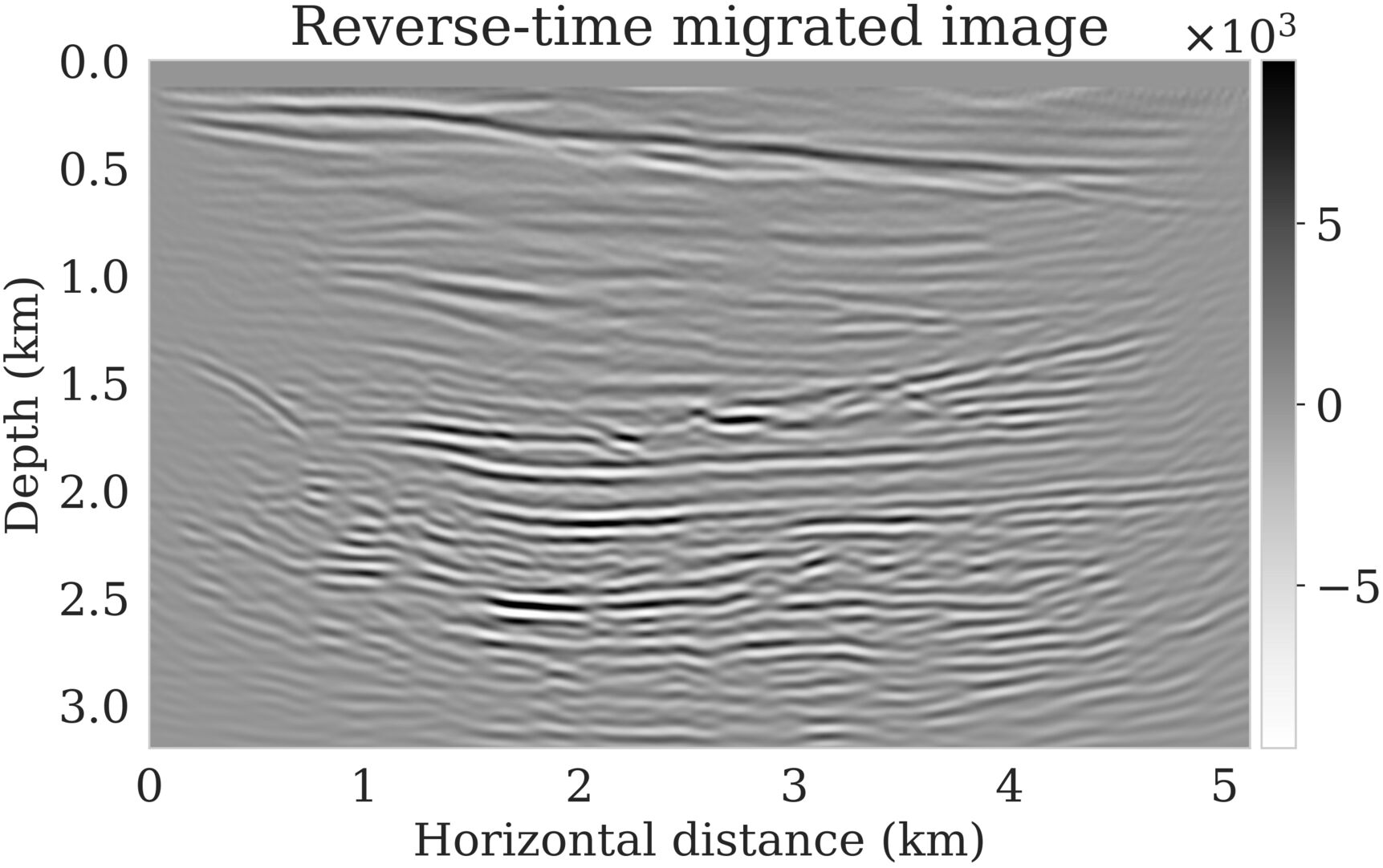}
        \vspace{0ex}\caption{}
        \label{fig:sup_idx-1_observed_data}
        \end{subfigure}\hspace{0em}
        \\
        \begin{subfigure}[b]{0.31\textwidth}
            \includegraphics[width=\textwidth]{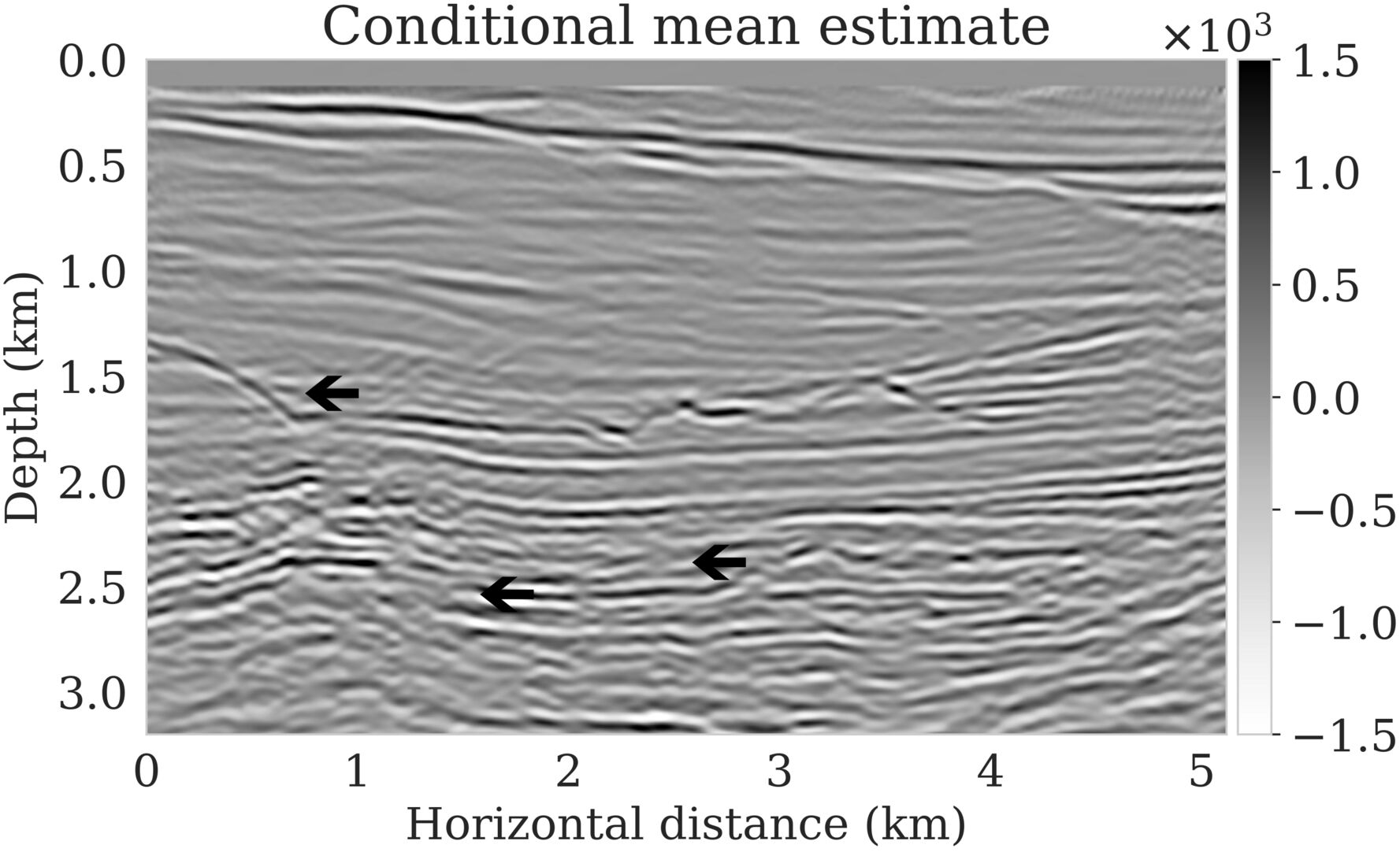}
        \vspace{0ex}\caption{}
        \label{fig:sup_idx-1_conditional_mean}
        \end{subfigure}\hspace{0em}
        \begin{subfigure}[b]{0.31\textwidth}
            \includegraphics[width=\textwidth]{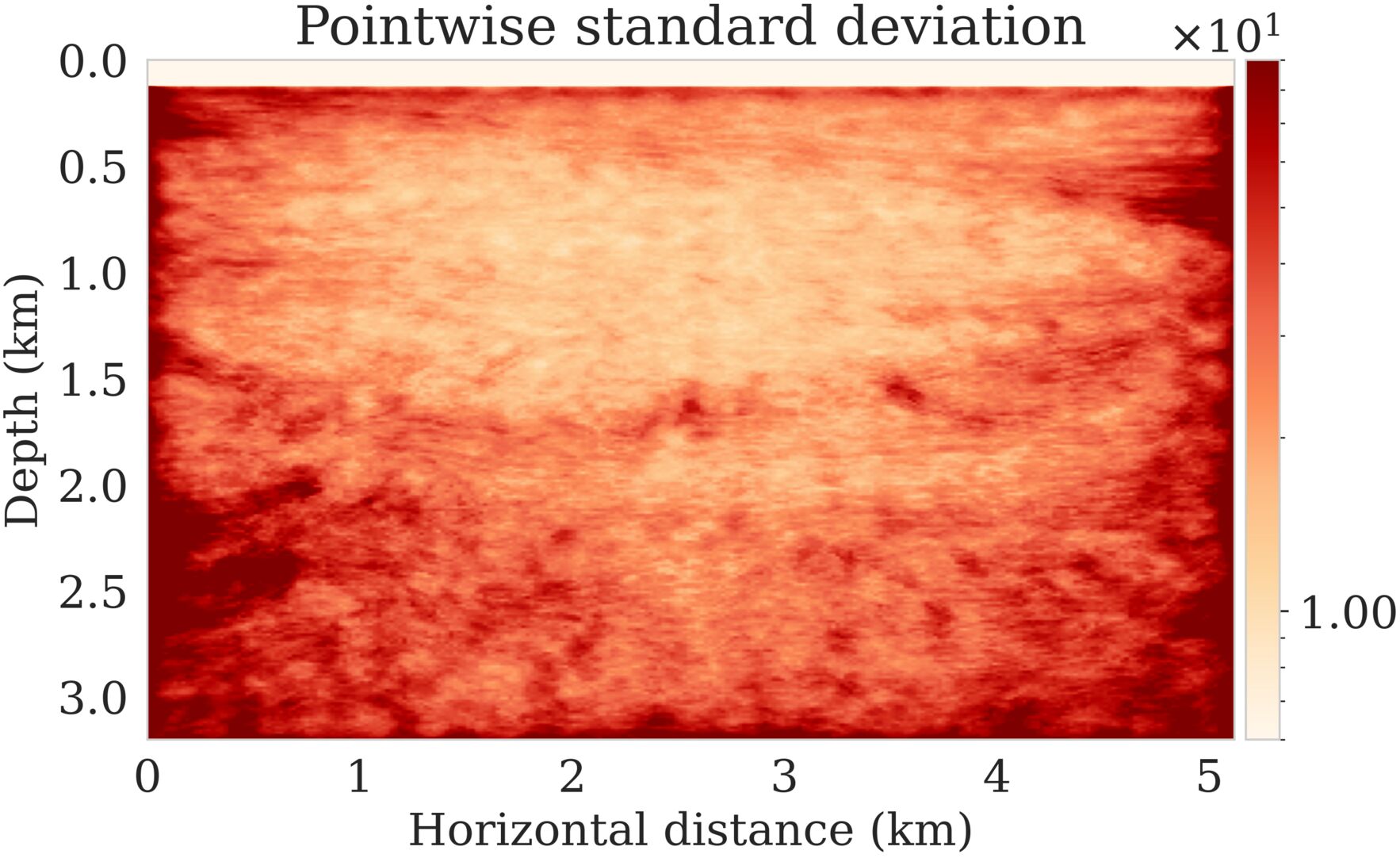}
        \vspace{0ex}\caption{}
        \label{fig:sup_idx-1_pointwise_std}
        \end{subfigure}\hspace{0em}
        \\
        \begin{subfigure}[b]{0.31\textwidth}
            \includegraphics[width=\textwidth]{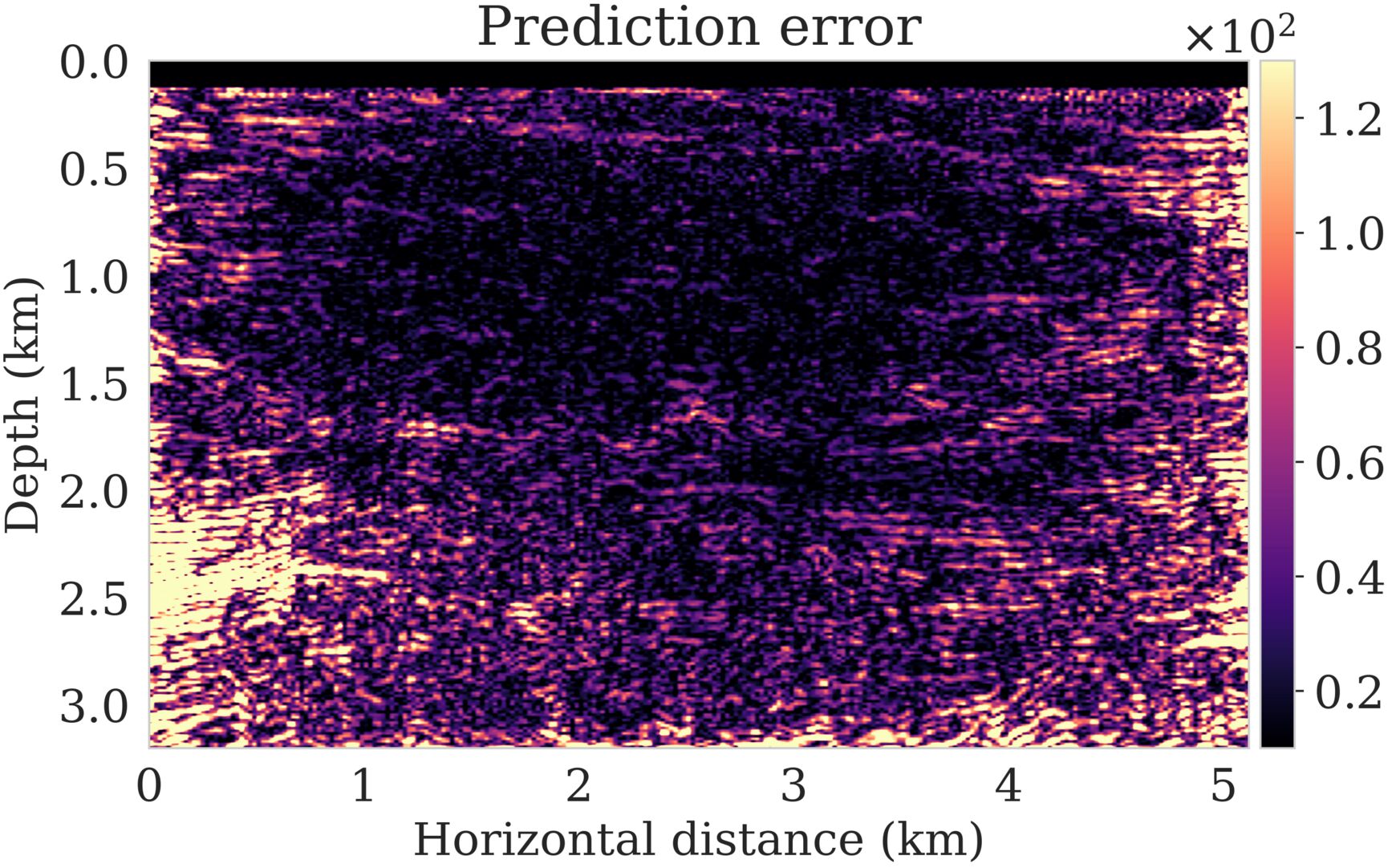}
        \vspace{0ex}\caption{}
        \label{fig:sup_idx-1_error}
        \end{subfigure}\hspace{0em}
        \begin{subfigure}[b]{0.31\textwidth}
            \includegraphics[width=\textwidth]{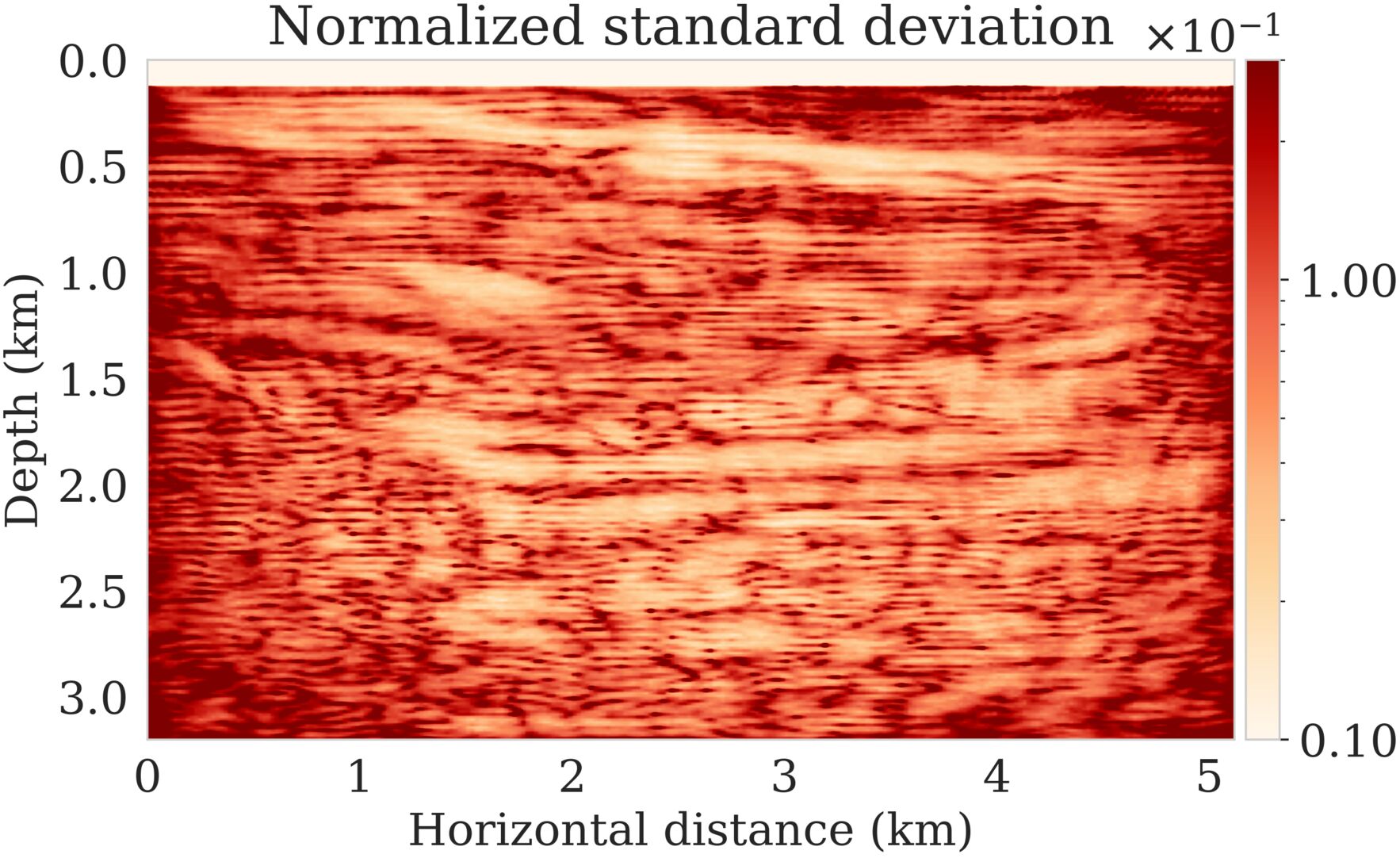}
        \vspace{0ex}\caption{}
        \label{fig:sup_idx-1_normalized_pointwise_std}
        \end{subfigure}\hspace{0em}
        \end{tabular}
    &
    \hspace{-2.5em}
    \begin{tabular}{c}
        \begin{subfigure}[b]{0.343\textwidth}
            \includegraphics[width=\textwidth]{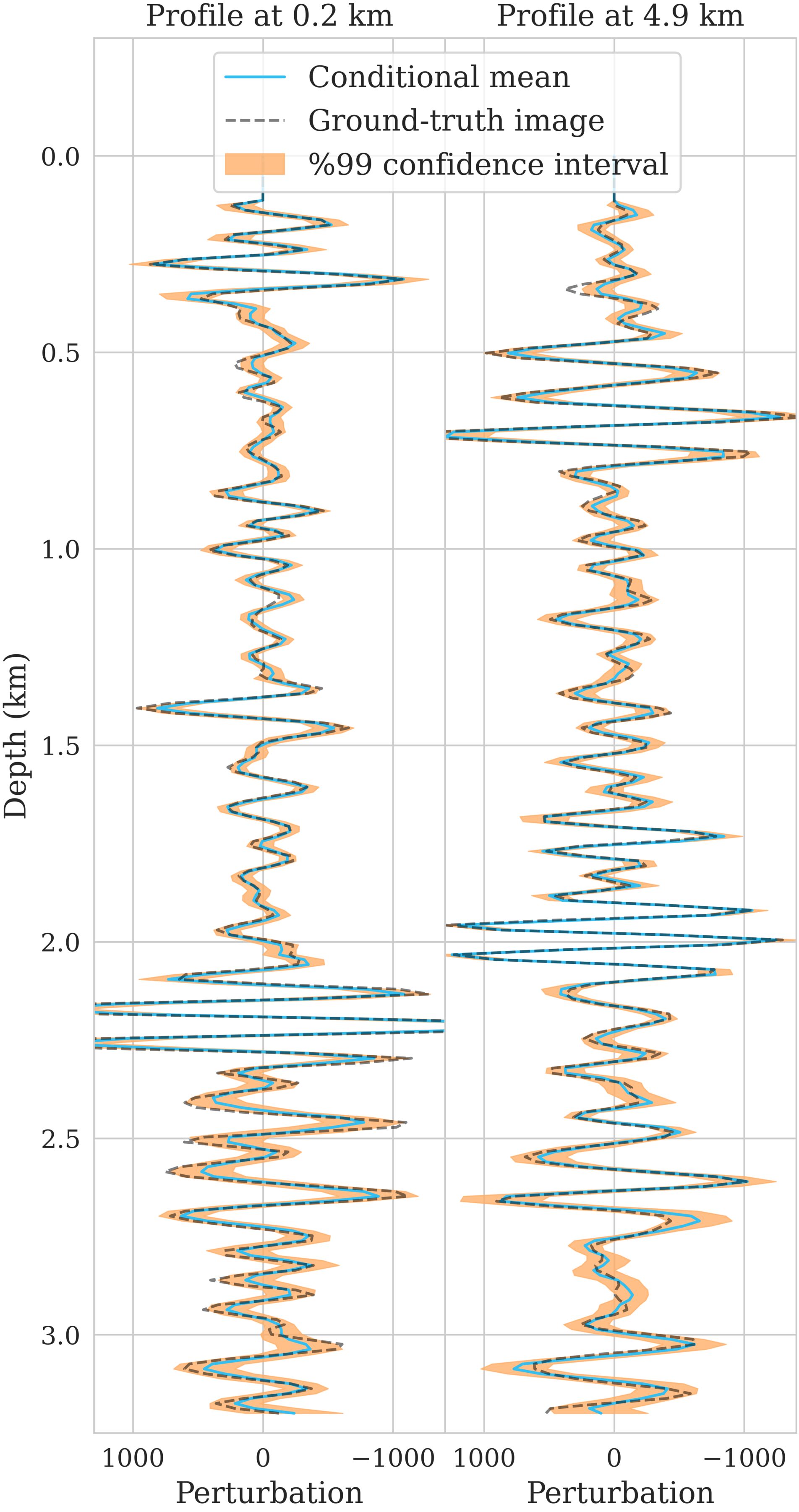}
        \vspace{0ex}\caption{}
        \label{fig:sup_idx-1_vertical_profile_at_10}
        \end{subfigure}\hspace{0em}
            \end{tabular}
    \end{tabular}
    \caption{Seismic imaging and uncertainty quantification. (a) Ground-truth seismic image. (b) Data after applying the adjoint Born operator (known as the reverse-time migrated image). (c) Conditional (posterior) mean. (d) Pointwise standard deviation. (e) Absolute error between Figures~\ref{fig:sup_idx-1_true_model} and~\ref{fig:sup_idx-1_conditional_mean}. (f) Normalized pointwise standard deviation by the envelope of the conditional mean. (g) Vertical profiles of the ground-truth image, conditional mean estimate, and the $99\%$ confidence interval at two lateral positions in the image.}
        \label{figs:sup_idx-1}
    \end{figure}

\begin{figure}[!htb]
    \centering
    \captionsetup[subfigure]{skip=-11pt}
    \begin{tabular}[t]{cc}
        \begin{tabular}{c}
        \begin{subfigure}[b]{0.31\textwidth}
            \includegraphics[width=\textwidth]{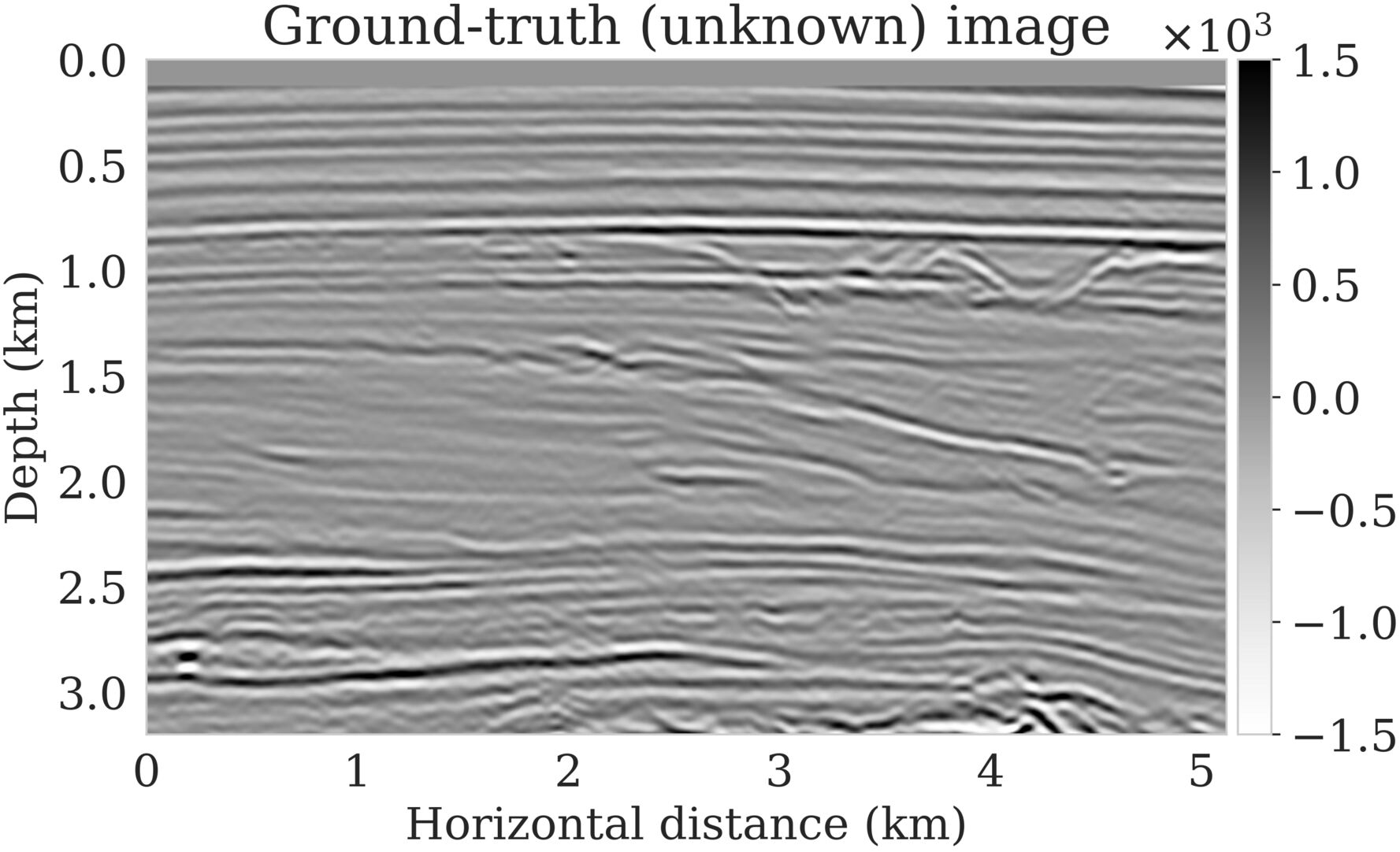}
        \vspace{0ex}\caption{}
        \label{fig:sup_idx-16_true_model}
        \end{subfigure}\hspace{0em}
        \begin{subfigure}[b]{0.31\textwidth}
            \includegraphics[width=\textwidth]{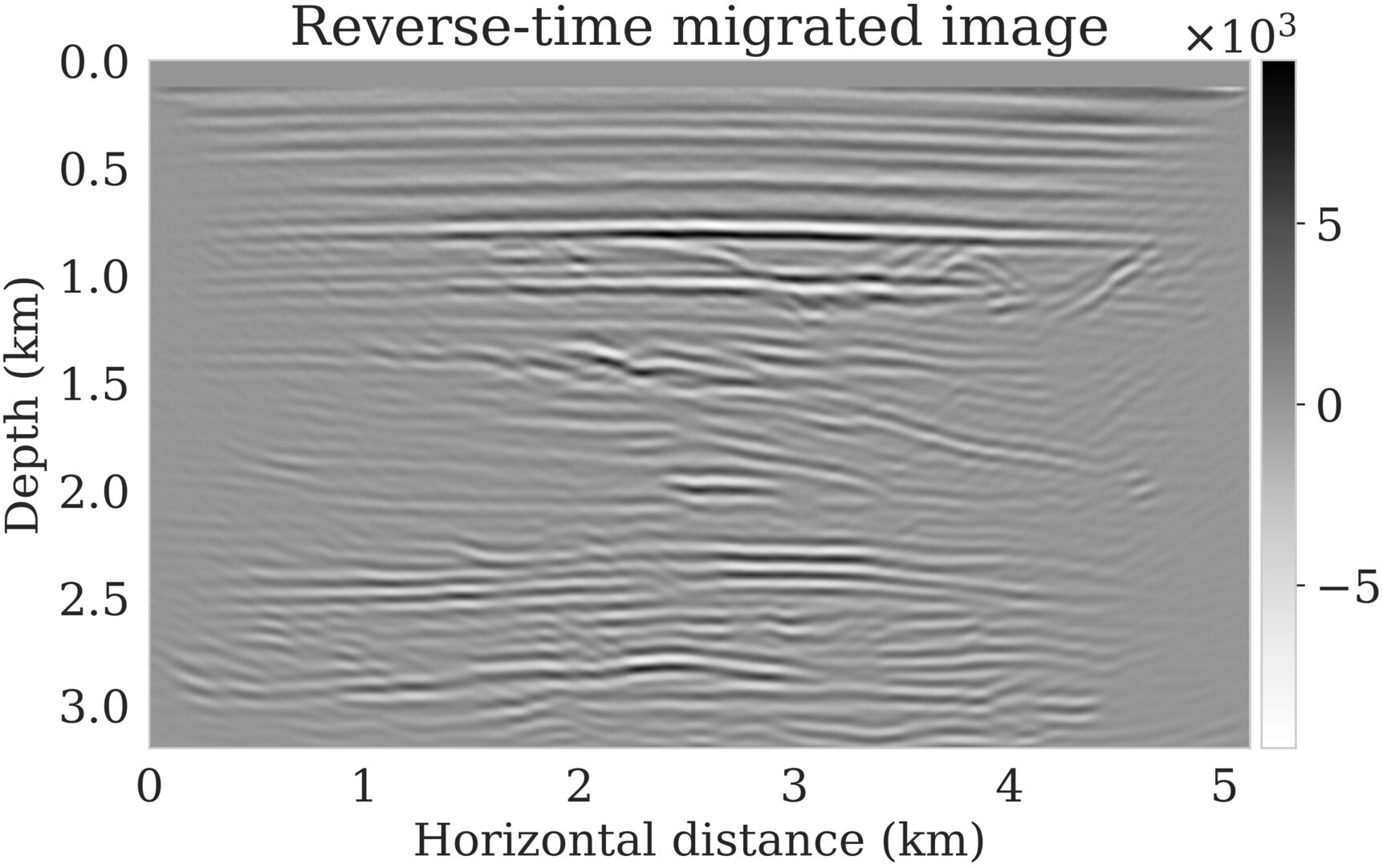}
        \vspace{0ex}\caption{}
        \label{fig:sup_idx-16_observed_data}
        \end{subfigure}\hspace{0em}
        \\
        \begin{subfigure}[b]{0.31\textwidth}
            \includegraphics[width=\textwidth]{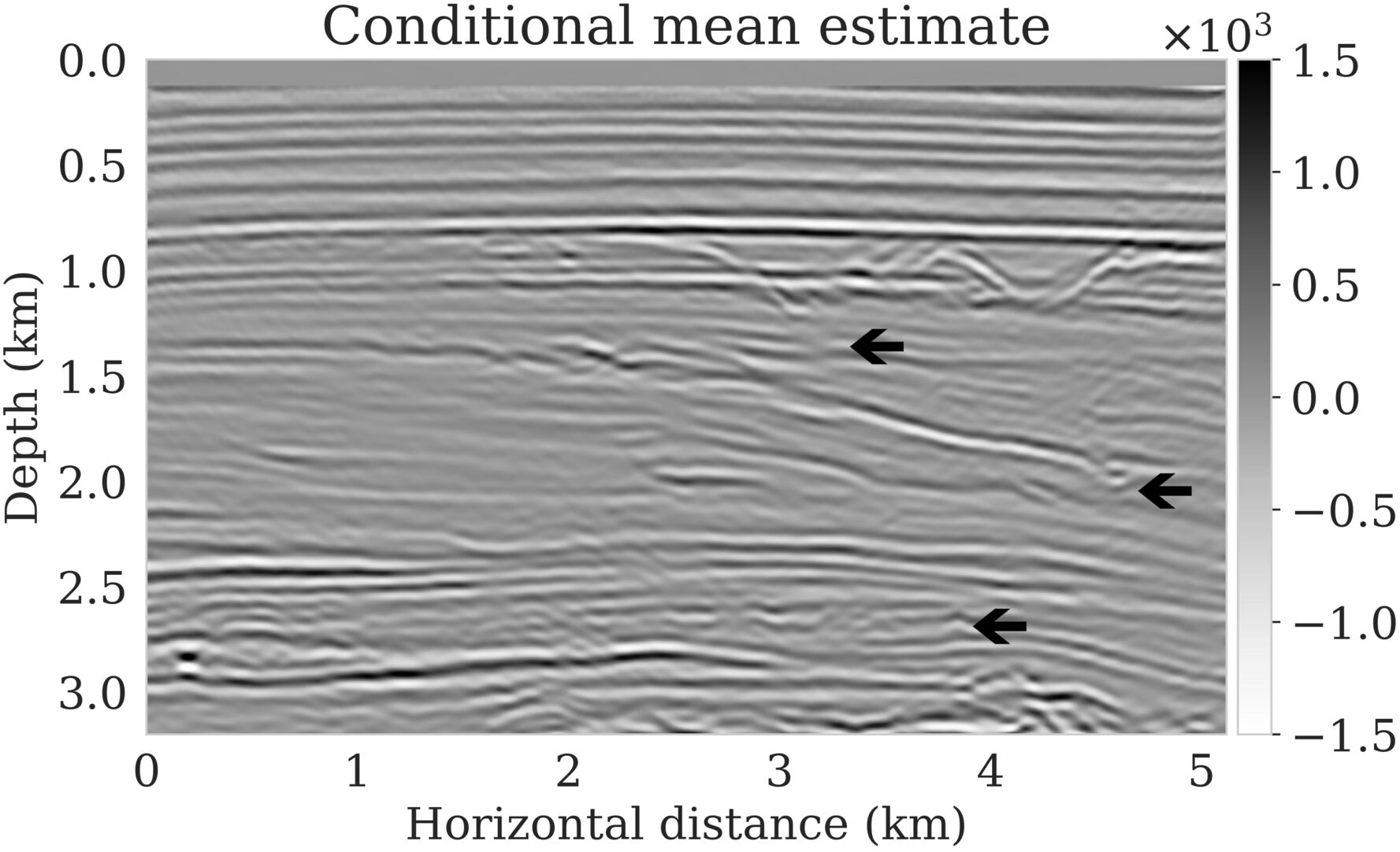}
        \vspace{0ex}\caption{}
        \label{fig:sup_idx-16_conditional_mean}
        \end{subfigure}\hspace{0em}
        \begin{subfigure}[b]{0.31\textwidth}
            \includegraphics[width=\textwidth]{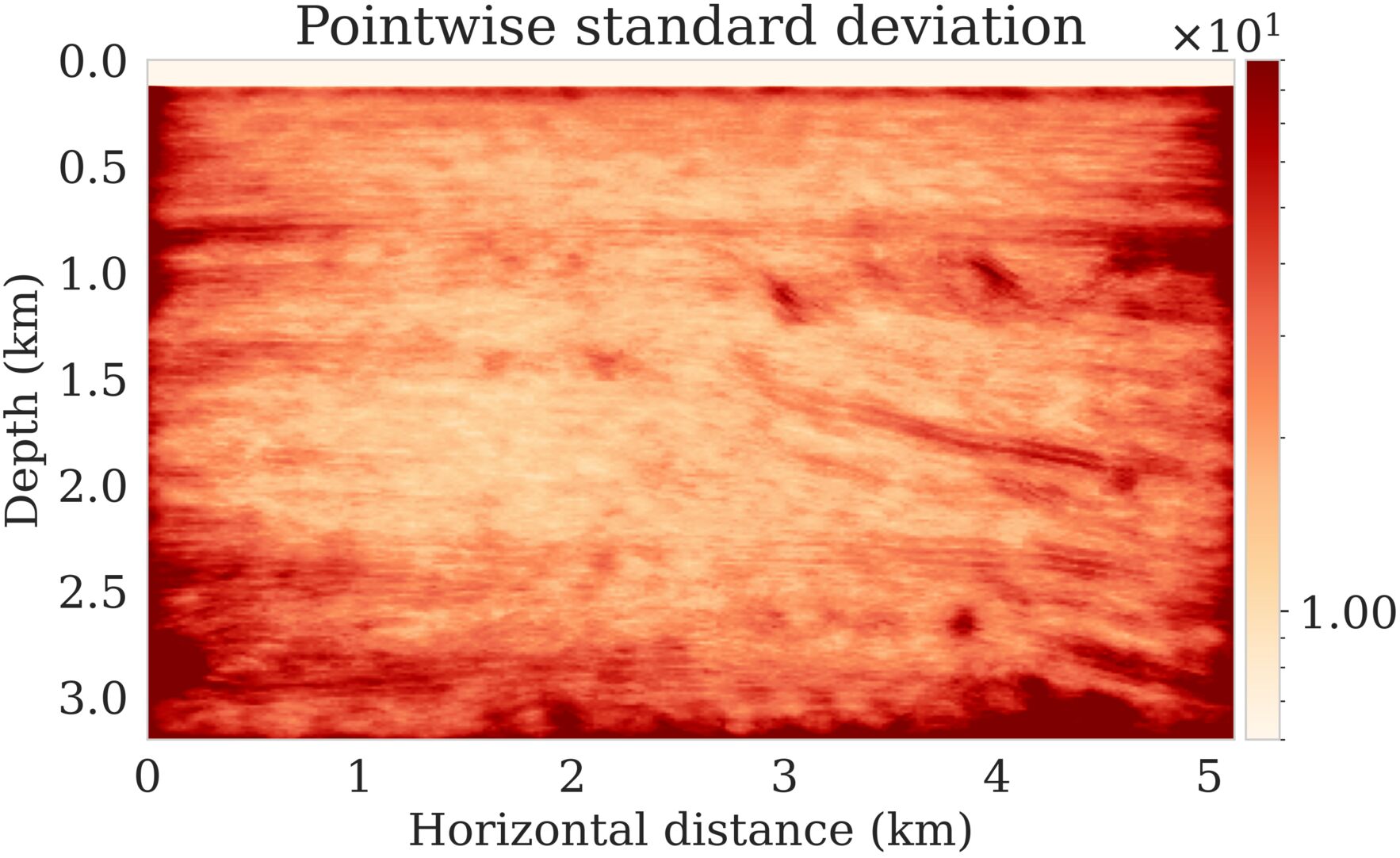}
        \vspace{0ex}\caption{}
        \label{fig:sup_idx-16_pointwise_std}
        \end{subfigure}\hspace{0em}
        \\
        \begin{subfigure}[b]{0.31\textwidth}
            \includegraphics[width=\textwidth]{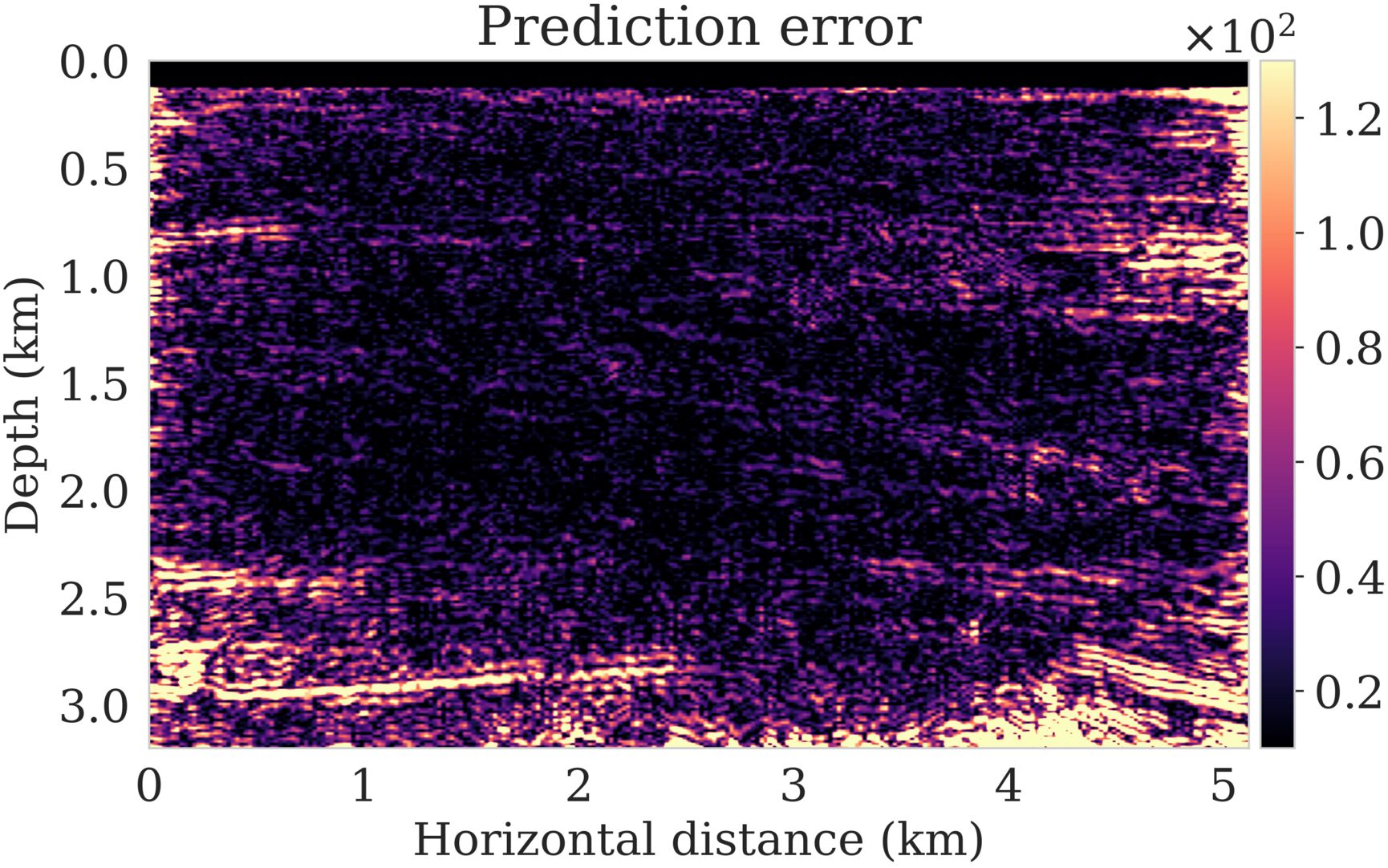}
        \vspace{0ex}\caption{}
        \label{fig:sup_idx-16_error}
        \end{subfigure}\hspace{0em}
        \begin{subfigure}[b]{0.31\textwidth}
            \includegraphics[width=\textwidth]{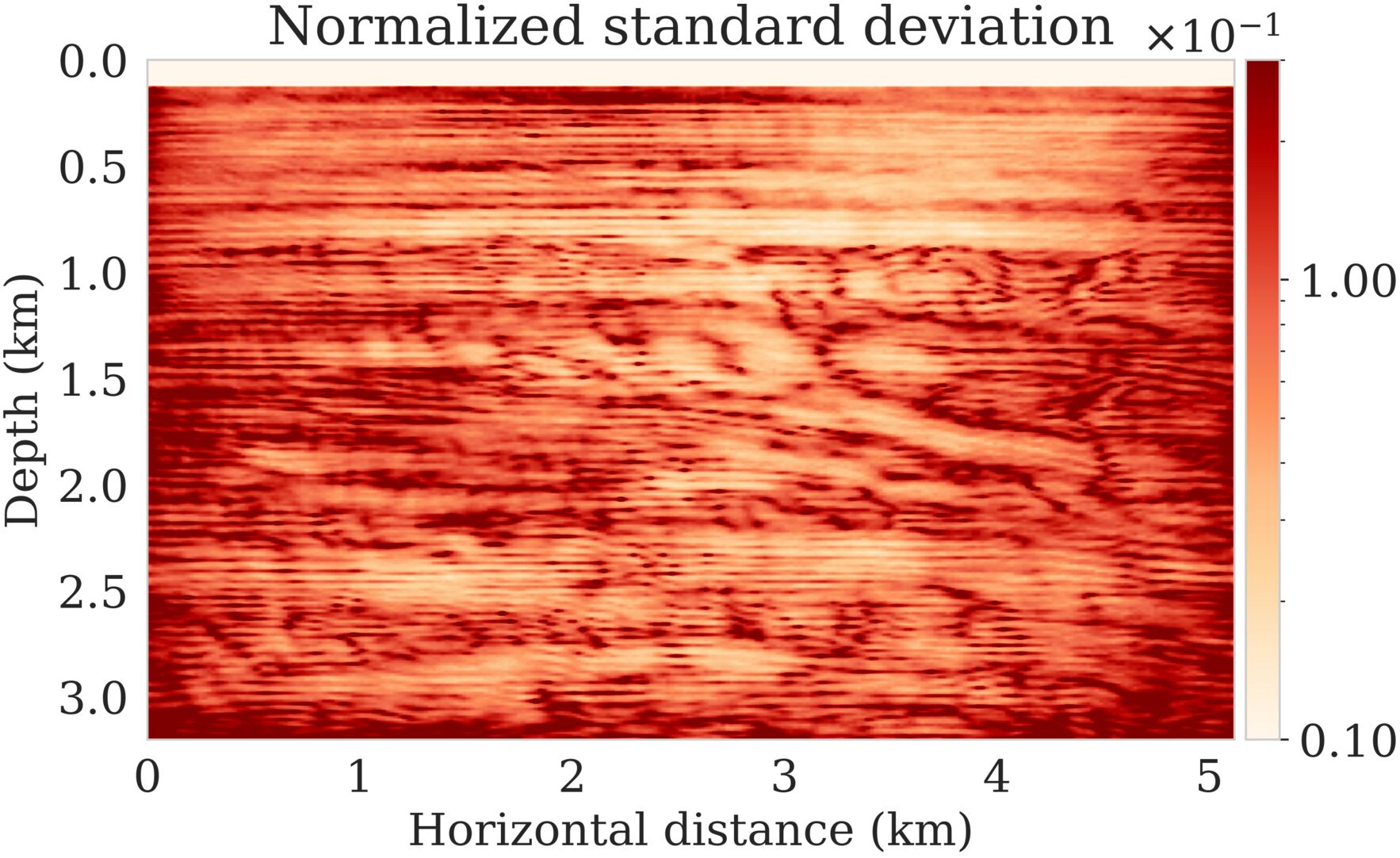}
        \vspace{0ex}\caption{}
        \label{fig:sup_idx-16_normalized_pointwise_std}
        \end{subfigure}\hspace{0em}
        \end{tabular}
    &
    \hspace{-2.5em}
    \begin{tabular}{c}
        \begin{subfigure}[b]{0.343\textwidth}
            \includegraphics[width=\textwidth]{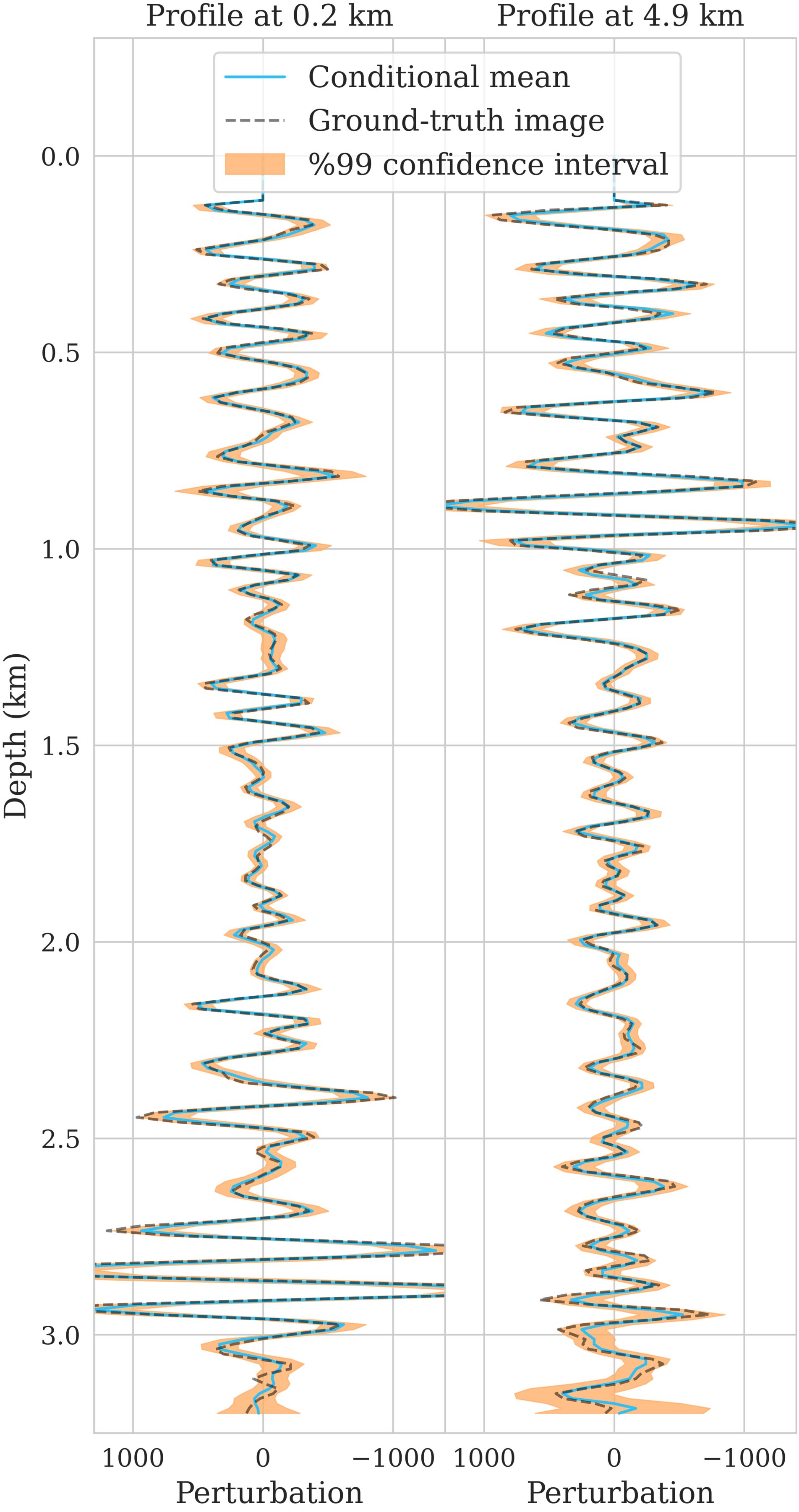}
        \vspace{0ex}\caption{}
        \label{fig:sup_idx-16_vertical_profile_at_10}
        \end{subfigure}\hspace{0em}
            \end{tabular}
    \end{tabular}
    \caption{Seismic imaging and uncertainty quantification. (a) Ground-truth seismic image. (b) Data after applying the adjoint Born operator (known as the reverse-time migrated image). (c) Conditional (posterior) mean. (d) Pointwise standard deviation. (e) Absolute error between Figures~\ref{fig:sup_idx-16_true_model} and~\ref{fig:sup_idx-16_conditional_mean}. (f) Normalized pointwise standard deviation by the envelope of the conditional mean. (g) Vertical profiles of the ground-truth image, conditional mean estimate, and the $99\%$ confidence interval at two lateral positions in the image.}
        \label{figs:sup_idx-16}
    \end{figure}

\begin{figure}[!htb]
    \centering
    \captionsetup[subfigure]{skip=-11pt}
    \begin{tabular}[t]{cc}
        \begin{tabular}{c}
        \begin{subfigure}[b]{0.31\textwidth}
            \includegraphics[width=\textwidth]{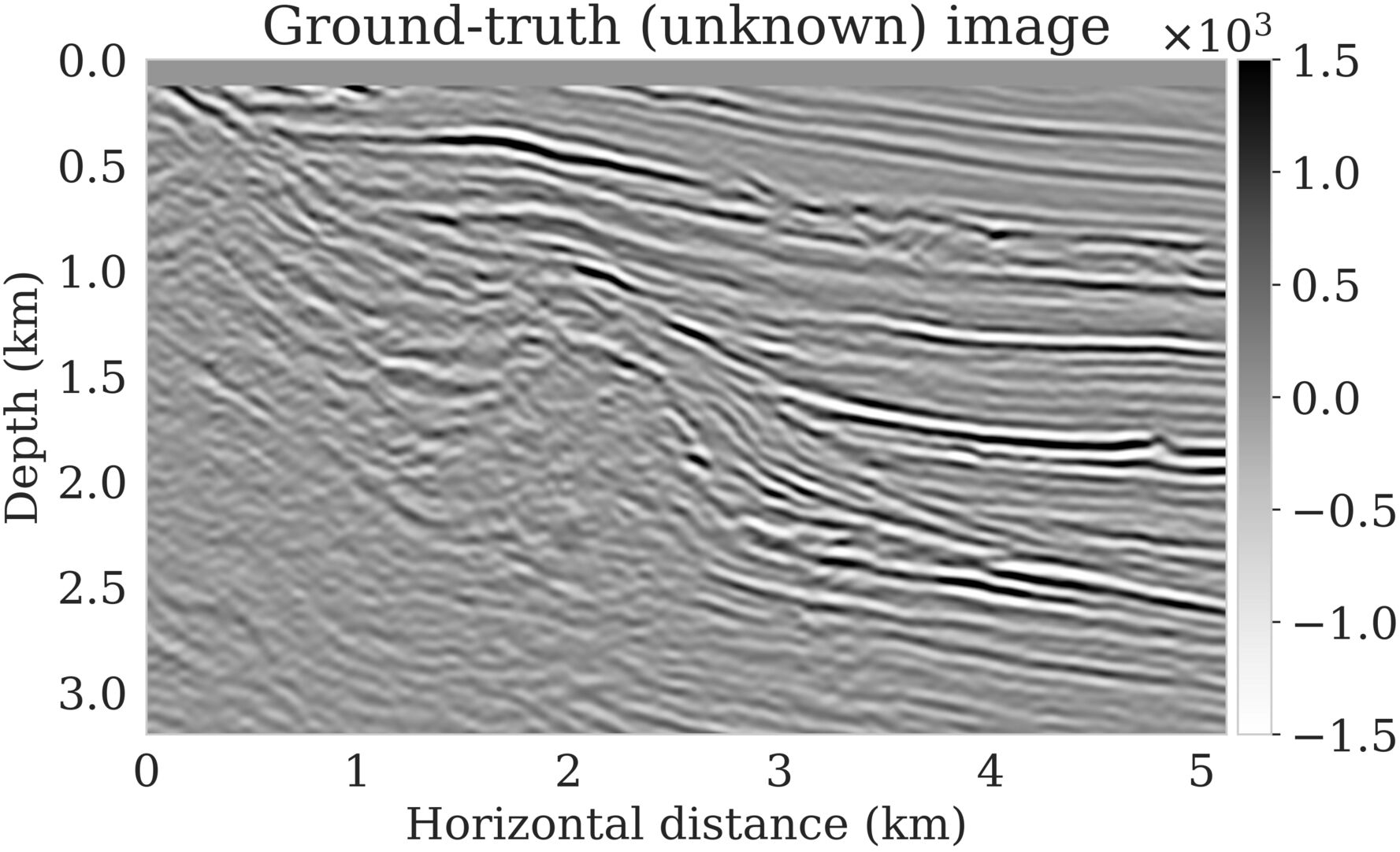}
        \vspace{0ex}\caption{}
        \label{fig:sup_idx-19_true_model}
        \end{subfigure}\hspace{0em}
        \begin{subfigure}[b]{0.31\textwidth}
            \includegraphics[width=\textwidth]{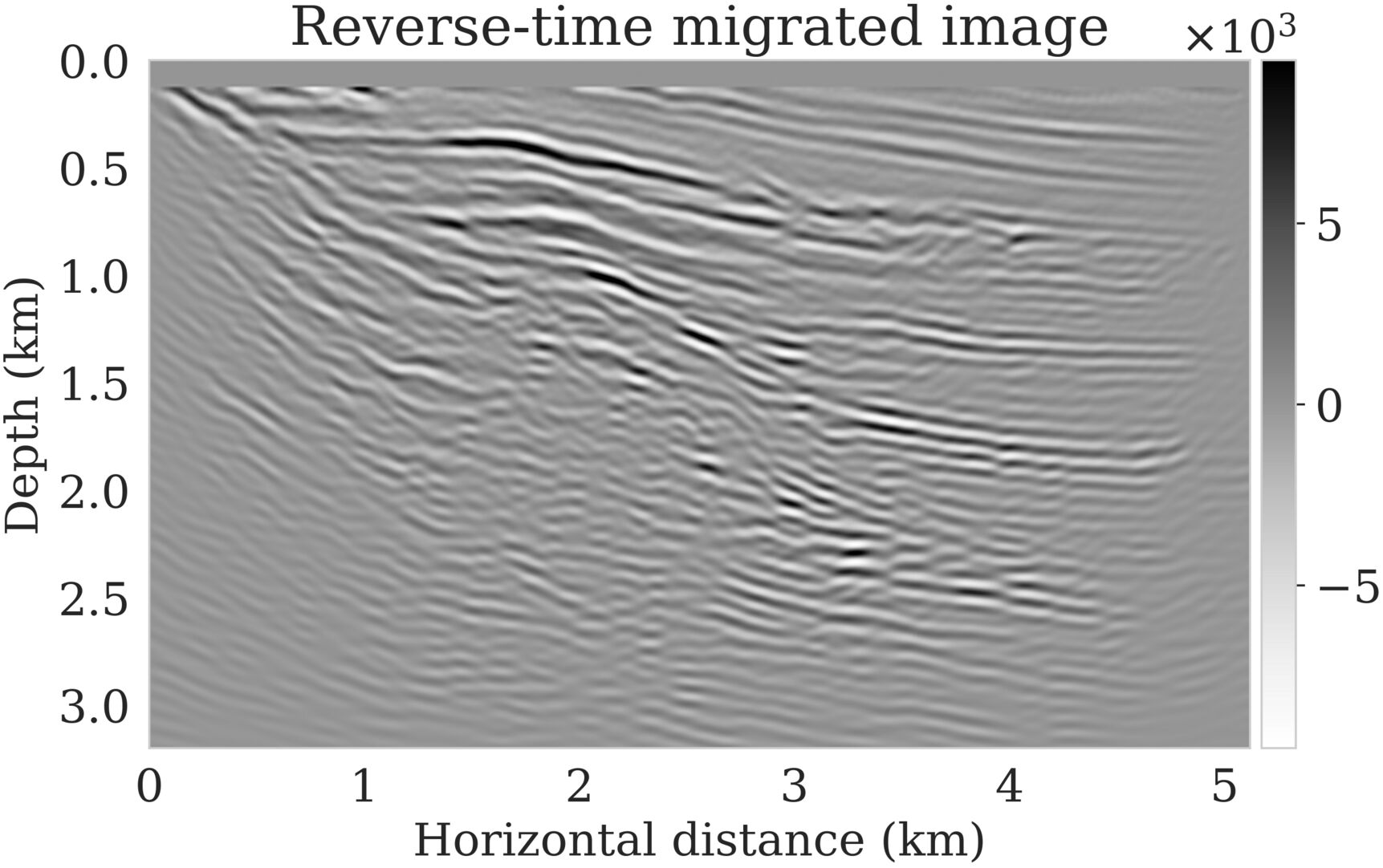}
        \vspace{0ex}\caption{}
        \label{fig:sup_idx-19_observed_data}
        \end{subfigure}\hspace{0em}
        \\
        \begin{subfigure}[b]{0.31\textwidth}
            \includegraphics[width=\textwidth]{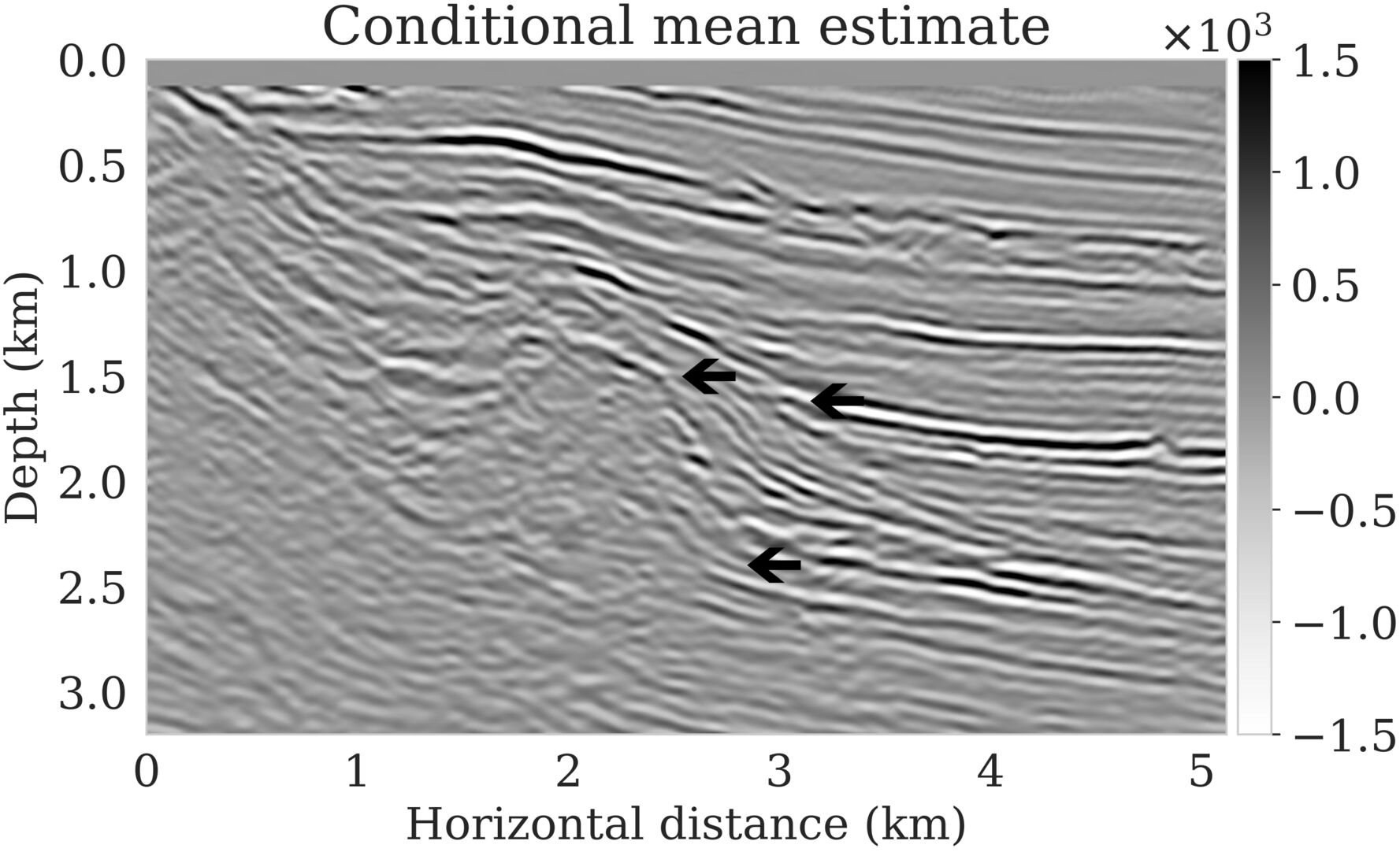}
        \vspace{0ex}\caption{}
        \label{fig:sup_idx-19_conditional_mean}
        \end{subfigure}\hspace{0em}
        \begin{subfigure}[b]{0.31\textwidth}
            \includegraphics[width=\textwidth]{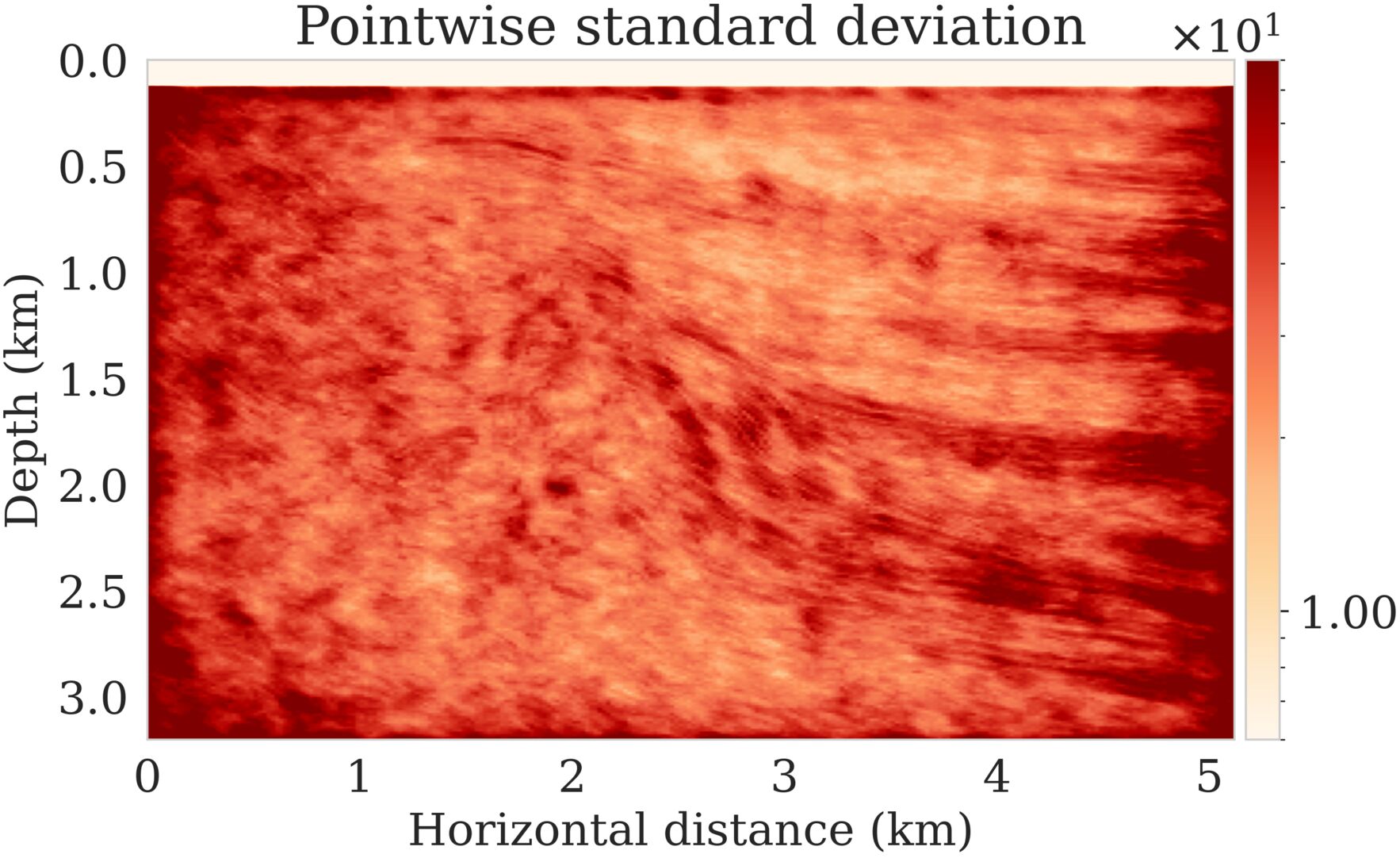}
        \vspace{0ex}\caption{}
        \label{fig:sup_idx-19_pointwise_std}
        \end{subfigure}\hspace{0em}
        \\
        \begin{subfigure}[b]{0.31\textwidth}
            \includegraphics[width=\textwidth]{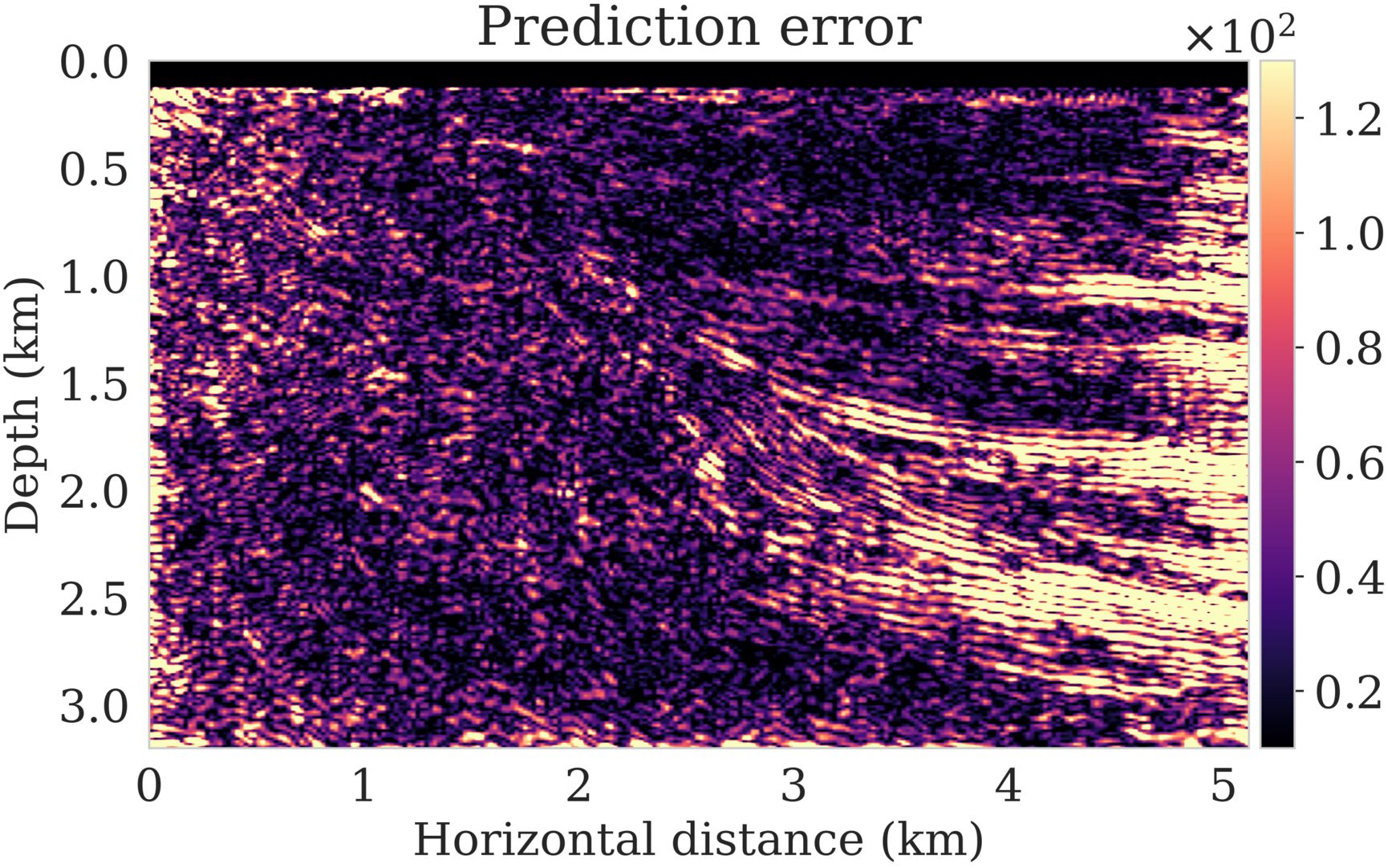}
        \vspace{0ex}\caption{}
        \label{fig:sup_idx-19_error}
        \end{subfigure}\hspace{0em}
        \begin{subfigure}[b]{0.31\textwidth}
            \includegraphics[width=\textwidth]{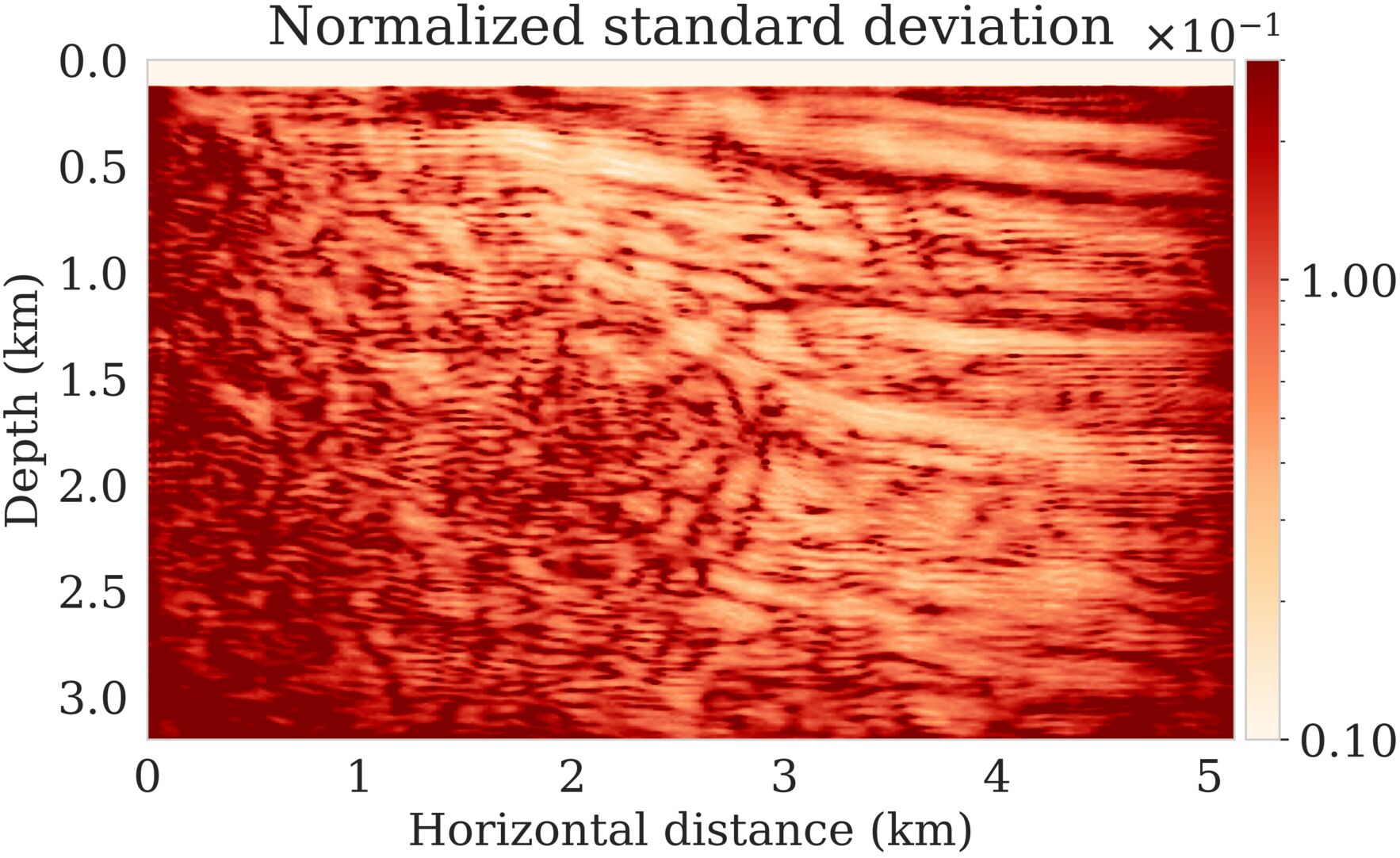}
        \vspace{0ex}\caption{}
        \label{fig:sup_idx-19_normalized_pointwise_std}
        \end{subfigure}\hspace{0em}
        \end{tabular}
    &
    \hspace{-2.5em}
    \begin{tabular}{c}
        \begin{subfigure}[b]{0.343\textwidth}
            \includegraphics[width=\textwidth]{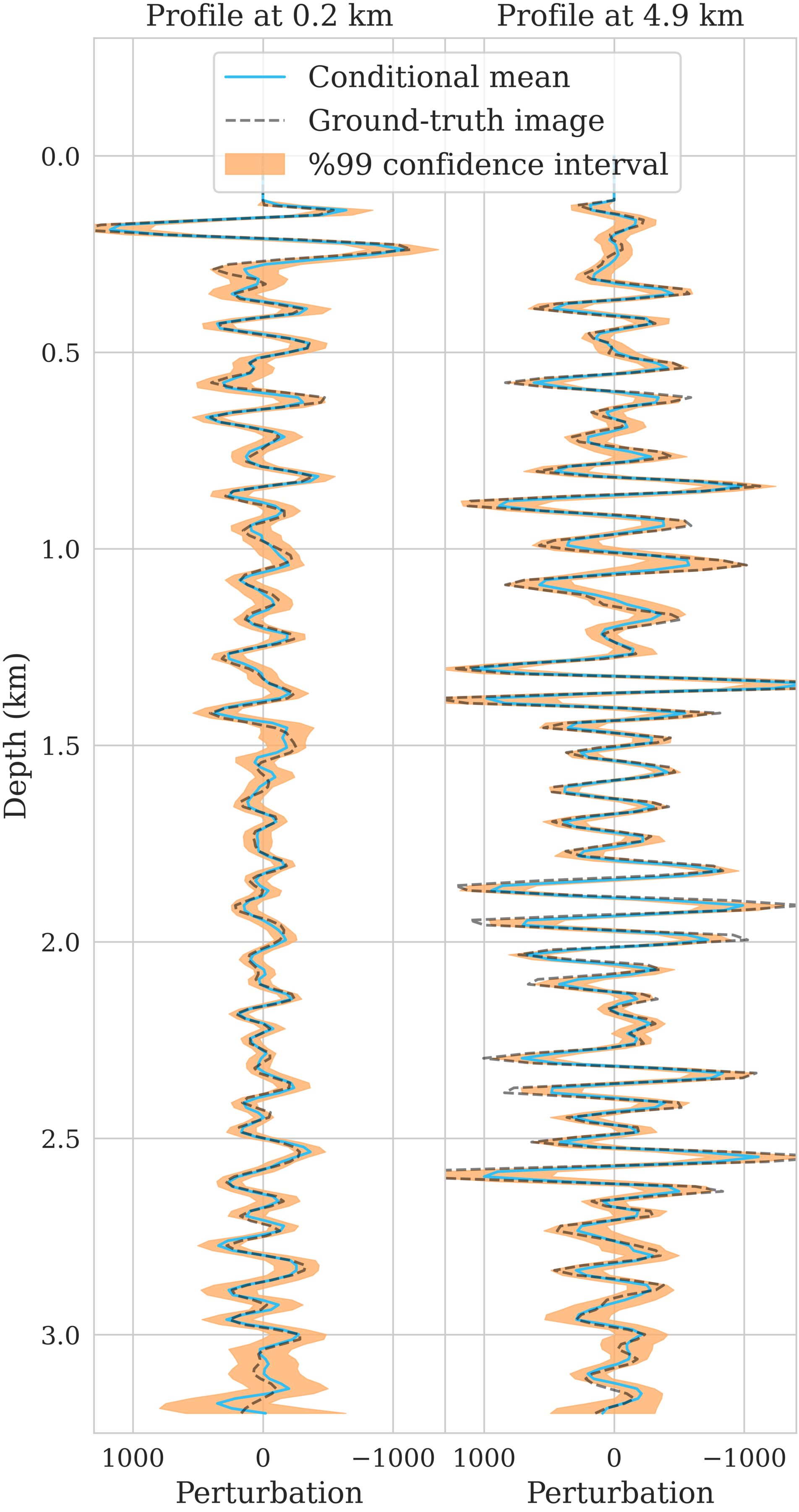}
        \vspace{0ex}\caption{}
        \label{fig:sup_idx-19_vertical_profile_at_10}
        \end{subfigure}\hspace{0em}
            \end{tabular}
    \end{tabular}
    \caption{Seismic imaging and uncertainty quantification. (a) Ground-truth seismic image. (b) Data after applying the adjoint Born operator (known as the reverse-time migrated image). (c) Conditional (posterior) mean. (d) Pointwise standard deviation. (e) Absolute error between Figures~\ref{fig:sup_idx-19_true_model} and~\ref{fig:sup_idx-19_conditional_mean}. (f) Normalized pointwise standard deviation by the envelope of the conditional mean. (g) Vertical profiles of the ground-truth image, conditional mean estimate, and the $99\%$ confidence interval at two lateral positions in the image.}
        \label{figs:sup_idx-19}
    \end{figure}

\end{document}